\documentclass{article}

\usepackage{ifthen}
\newboolean{blind}   
\setboolean{blind}{false}

\usepackage{comment}

\ifthenelse{\boolean{blind}}
{
\usepackage{neurips_2019}
}
{
\usepackage[nonatbib,final]{neurips_2019}
}

\ifthenelse{\boolean{blind}}
{

}
{
\usepackage[numbers]{natbib}
}
\usepackage[utf8]{inputenc} 
\usepackage[T1]{fontenc}    
\usepackage{hyperref}       
\hypersetup{colorlinks,linkcolor={blue},citecolor={blue},urlcolor={red}}  
\usepackage{url}            
\usepackage{booktabs}       
\usepackage{amsfonts}       
\usepackage{nicefrac}       
\usepackage{microtype}      
\usepackage{graphicx}
\usepackage{subcaption}

\usepackage{mathtools}
\usepackage{amsthm}
\usepackage{bm}
\usepackage{lipsum}
\usepackage{wrapfig}

\newtheorem{lemma}{Lemma}

\usepackage{bbm}
\usepackage{algpseudocode}
\usepackage{multirow,tabularx}

\usepackage{xcolor}

\newcommand{\pflem}[1]{\begin{proof}[Proof of Lemma \ref{#1}] \label{#1p}}
\newcommand{\pfthm}[1]{\begin{proof}[Proof of Theorem \ref{#1}] \label{#1p}}
\newcommand{\epf}{\end{proof}}
\newcommand{\mean}{0}
\newcommand{\M}{\widehat{\mathcal{M}}} 
\newcommand{\g}{\tilde{g}_}
\newcommand{\Q}{\widehat{Q}}
\newcommand{\C}{\widehat{C}}
\newcommand{\qu}{Q}
\newcommand{\cu}{C}

\ifthenelse{\boolean{blind}}
{
\newcommand{\git}{\url{https://anonymous.4open.science/r/bfd8ce59-4cb4-4554-8852-2e329d699f44/}}
}
{
\newcommand{\git}{\url{https://github.com/yanivbl6/quantized_meanfield}}
}

\makeatletter
\g@addto@macro\normalsize{%
  \setlength\abovedisplayskip{4pt}
  \setlength\belowdisplayskip{4pt}
  \setlength\abovedisplayshortskip{4pt}
  \setlength\belowdisplayshortskip{4pt}
}
\makeatother

\title{A Mean Field Theory of Quantized Deep Networks:\\ The Quantization-Depth Trade-Off  }

\author{%
  Yaniv Blumenfeld \\
  Technion, Israel \\ 
  \texttt{yanivblm6@gmail.com} \\
  \And
  Dar Gilboa \\
  Columbia University \\
  \texttt{dargilboa@gmail.com} \\
  \And
  Daniel Soudry \\
  Technion, Israel \\ 
  \texttt{daniel.soudry@gmail.com} \\
}

\begin{document}

\maketitle
\begin{abstract}
Reducing the precision of weights and activation functions in neural network training, with minimal impact on performance, is essential for the deployment of these models in resource-constrained environments. We apply mean field techniques to networks with quantized activations in order to evaluate the degree to which quantization degrades signal propagation at initialization. We derive initialization schemes which maximize signal propagation in such networks, and suggest why this is helpful for generalization. Building on these results, we obtain a closed form implicit equation for $L_{\max}$, the maximal trainable depth (and hence model capacity), given $N$, the number of quantization levels in the activation function. Solving this equation numerically, we obtain asymptotically: $L_{\max}\propto N^{1.82}$.


\end{abstract}

\section{Introduction}

As neural networks are increasingly trained and deployed on-device in settings with memory and space constraints \cite{howard2017mobilenets, chen2015compressing}, a better understanding of the trade-offs involved in the choice of architecture and training procedure are gaining in importance. One widely used method to conserve resources is the quantization (discretization) of the weights and/or activation functions during training 
\cite{courbariaux2016binarized,rastegari2016xnor,hubara2017quantized,banner2018scalable}. When choosing a quantized architecture, it is natural to expect depth to increase the flexibility of the model class, yet choosing a deeper architecture can make the training process more difficult. Additionally, due to resource constraints, when using a quantized activation function whose image is a finite set of size $N$, one would like to choose the smallest possible $N$ such that the model is trainable and performance is minimally affected. There is a trade-off here between the capacity of the network which depends on its depth and the ability to train it efficiently on the one hand --- and the parsimony of the activation function used on the other. 
 
We quantify this trade-off between capacity/trainability and the degree of quantization by an analysis of wide neural networks at initialization. This is achieved by studying \textit{signal propagation} in deep quantized networks, using techniques introduced in \cite{poole2016exponential, schoenholz2016deep} that have been applied to numerous architectures. Signal propagation will refer to the propagation of correlations between inputs into the hidden states of a deep network. Additionally, we consider the dynamics of training in this regime and the effect of signal propagation on the change in generalization error during training.

In this paper, 
\begin{itemize}
    \item We suggest (section \ref{sec:NTK generalization}) that if the signal propagation conditions do not hold, generalization error in early stages of training should not decrease at a typical test point, potentially explaining the empirically observed benefit of signal propagation to generalization. This is done using an analysis of learning dynamics in wide neural networks, and corroborated numerically.
    \item We obtain (section \ref{sec:General Quatnization}) initialization schemes that maximize signal propagation in certain classes of feed-forward networks with quantized activations. 
    \item Combining these results, we obtain an expression for the trade-off between the quantization level and the maximal trainable depth of the network (eq. \ref{eq:xi_N}), in terms of the depth scale of signal propagation. We experimentally corroborate these predictions (Figure \ref{fig:MNIST_experiment}).
\end{itemize}

\section{Related work}

Several works have shown that training a 16 bit numerical precision is sufficient for most machine learning applications \cite{gupta2015deep, das2018mixed}, with little to no cost to model accuracy. Since, many more aggressive quantization schemes were suggested \cite{hubara2017quantized,lin2017towards,miyashita2016convolutional, mishra2017wrpn}, ranging from the extreme usage of 1-bit at  representations and math operations \cite{rastegari2016xnor,courbariaux2016binarized}, to a more conservative usage of 8-bits \cite{banner2018scalable,wang2018training}, all in effort to minimize the computational cost with minimal loss to model accuracy. Theoretically, it is well known that a small amount of imprecision can significantly degrade the representational capacity of a model. For example, an infinite precision recurrent neural network can simulate a universal Turing machine \cite{Siegelmann1991}. However, any numerical imprecision reduces the representational power of these models to that of finite automata \cite{Maass1998}. In this paper, we focus on the effects of quantization on training. So far, these effects are typically quantified empirically, though some theoretical work has been done in this direction (e.g. \cite{Li2017b,Anderson2017,zhou2018adaptive,Yin2019}).
   
Signal propagation in wide neural networks has been the subject of recent work for fully-connected \cite{poole2016exponential, schoenholz2016deep, pennington2017resurrecting, yang2017mean}, convolutional \cite{xiao2018dynamical} and recurrent architectures \cite{chen2018dynamical, gilboa2019dynamical}. These works study the evolution of covariances between the hidden states of the network and the stability of the gradients. These depend only on the leading moments of the weight distributions and the nonlinearities at the infinite width limit, greatly simplifying analysis. They identify critical initialization schemes that allow training of very deep networks (or recurrent networks on long time sequence tasks) without performing costly hyperparameter searches. While the analytical results in these works assume that the layer widths are taken to infinity sequentially (which we will refer to this as the \textit{sequential} limit), the predictions prove predictive when applied to networks with layers of equal width once the width is typically of the order of hundreds of neurons. For fully connected networks it was also shown using an application of the Central Limit Theorem for exchangeable random variables that the asymptotic behavior at infinite width is independent of the order of limits \cite{matthews2018gaussian}.
\section{Preliminaries: the mean field approach}
\subsection{Signal propagation in feed-forward networks}
We now review the analysis of signal propagation in feed-forward networks performed in \cite{poole2016exponential, schoenholz2016deep}. The network function $f:\mathbb{R}^{n_{0}}\rightarrow\mathbb{R}^{n_{L+1}}$ is given by \begin{equation} \label{eq:net}
\begin{array}{c}
\phi(\alpha^{(0)}(x))=x\\
\alpha^{(l)}(x)=W^{(l)}\phi(\alpha^{(l-1)}(x))+b^{(l)}~~\text{   }l=1,...,L\\
f(x)=\alpha^{(L+1)}(x)
\end{array}
\end{equation}
for input $x \in \mathbb{R}^{n_0}$, weight matrices $W^{(l)} \in
\mathbb{R}^{n_{l} \times n_{l-1}}$ and nonlinearity $\phi:\mathbb{R} \rightarrow \mathbb{R}$. The weights are initialized using $W_{ij}^{(l)}\sim\mathcal{N}(0,\frac{\sigma_{w}^{2}}{n^{(l-1)}}),b_{i}^{(l)}\sim\mathcal{N}(0,\sigma_{b}^{2})$ so that the variance of the neurons at every layer is independent of the layer widths \footnote{In principle the following results should hold under more generally mild moment conditions alone.}.

According to Theorem 4 in \cite{matthews2018gaussian}, under a mild condition on the activation function that is satisfied by any saturating nonlinearity, the \textit{pre-activations} $\alpha^{(l)}(x)$ converge in distribution to a multivariate Gaussian as the layer widths $n_1,...,n_L$ are taken to infinity in any order (with $n_0, n_{L+1}$ finite) \footnote{When taking the sequential limit, asymptotic normality is a consequence of repeated application of the Central Limit Theorem \cite{poole2016exponential}}. In the physics literature the approximation obtained by taking this limit is known as the \textit{mean field approximation}.

The covariance of this Gaussian at a given layer is then obtained by the recursive formula
\begin{equation} \label{eq:cov_udpate}
\begin{array}{c}
\mathbb{E}\alpha_{i}^{(l)}(x)\alpha_{j}^{(l)}(x')=\mathbb{E}\underset{k,k'=1}{\overset{n_{l-1}}{\sum}}W_{ik}^{(l)}W_{jk'}^{(l)}\phi(\alpha_{k}^{(l-1)}(x))\phi(\alpha_{k'}^{(l-1)}(x'))+b_{i}^{(l)}b_{j}^{(l)}\\
=\left[\sigma_{w}^{2}\mathbb{E}\phi(\alpha_{1}^{(l-1)}(x))\phi(\alpha_{1}^{(l-1)}(x'))+\sigma_{b}^{2}\right]\delta_{ij}.
\end{array}
\end{equation}
Omitting the dependence on the inputs $x,x'$ in the RHS below, we define
\begin{equation} \label{eq:QC_def}
\left(\begin{array}{cc}
\mathbb{E}\alpha_{i}^{(l)}(x)\alpha_{i}^{(l)}(x) & \mathbb{E}\alpha_{i}^{(l)}(x)\alpha_{i}^{(l)}(x')\\
\mathbb{E}\alpha_{i}^{(l)}(x)\alpha_{i}^{(l)}(x') & \mathbb{E}\alpha_{i}^{(l)}(x')\alpha_{i}^{(l)}(x')
\end{array}\right)=Q^{(l)}\left(\begin{array}{cc}
1 & C^{(l)}\\
C^{(l)} & 1
\end{array}\right)=\Sigma(Q^{(l)},C^{(l)}).
\end{equation}
Combining eqs. \ref{eq:cov_udpate} and \ref{eq:QC_def} we obtain the following two-dimensional dynamical system:
\begin{equation} \label{eq:QCsys}
\left(\begin{array}{c}
Q^{(l)}\\
C^{(l)}
\end{array}\right)=\left(\begin{array}{c}
\sigma_{w}^{2}\underset{u\sim\mathcal{N}(0,Q^{(l-1)})}{\mathbb{E}}\phi^{2}(u)+\sigma_{b}^{2}\\
\frac{1}{Q^{(l-1)}}\left[\sigma_{w}^{2}\underset{(u_{1},u_{2})\sim\mathcal{N}(\mean,\Sigma(Q^{(l-1)},C^{(l-1)}))}{\mathbb{E}}\phi(u_{1})\phi(u_{2})+\sigma_{b}^{2}\right]
\end{array}\right)\equiv\mathcal{M}\left[\left(\begin{array}{c}
Q^{(l-1)}\\
C^{(l-1)}
\end{array}\right)\right]\,,
\end{equation}
where $\mathcal{M}$ depends on the nonlinearity and the initialization hyperparameters $\sigma^2_w, \sigma^2_b$ and the initial conditions $(Q^{(0)},C^{(0)})^{T}$ depend also on $x,x'$. See Figure \ref{fig:manifold_16_states} for a visualization of the covariance propagation.

Once the above dynamical system converges to a fixed point $(Q^\ast,C^\ast)$ or at least approaches it to within numerical precision, information about the initial conditions is lost. As argued in \cite{schoenholz2016deep}, this is detrimental to learning as inputs in different classes can no longer be distinguished in terms of the network output (assuming the fixed point $C^\ast$ is independent of $C^{(0)}$, see Lemma \ref{lem:fixedpoints}). The convergence rate to the fixed point can be obtained by linearizing the dynamics around it. This can be done for the two dimensional system as a whole, yet in \cite{schoenholz2016deep} it was also shown that, for any monotonically increasing nonlinearity, convergence of this linearized dynamical system in the direction $C^{(l)}$ cannot be faster than convergence in the $Q^{(l)}$ direction, and thus studying convergence can be reduced to the simpler one dimensional system $C^{(l)}=\mathcal{M}_{Q^{\ast}}(C^{(l-1)})$ that is obtained by assuming $Q^{(l)}$ has already converged, as assumption we review in appendix \ref{sup:CQcomparison}. The convergence rate is given by the following known results of \cite{schoenholz2016deep, gilboa2019dynamical} which we recapitulate for completeness:

\begin{lemma} \cite{schoenholz2016deep, gilboa2019dynamical} \label{lem:fixedpoints}
Defining $\Sigma(Q,C)=Q\left(\begin{array}{cc}
1 & C\\
C & 1
\end{array}\right)$ for $Q \geq 0, C \in [-1,1]$ the dynamical system 
\begin{equation} \label{eq:M}
\mathcal{M}_{Q^{\ast}}(C)=\frac{1}{Q^{\ast}}\left[\sigma_{w}^{2}\underset{(u_{1},u_{2})\sim\mathcal{N}(\mean,\Sigma(Q^{\ast},C))}{\mathbb{E}}\phi(u_{1})\phi(u_{2})+\sigma_{b}^{2}\right]
\end{equation}
when linearized around a fixed point $C^\ast$, converges at a rate 
\begin{equation}\label{eq:chi_basic}
\chi=\left.\frac{\partial\mathcal{M}_{Q^{\ast}}(C)}{\partial C}\right|_{C^{\ast}}=\sigma_{w}^{2}\underset{(u_{a},u_{b})\sim\mathcal{N}(\mean,\Sigma(Q^{\ast},C^{\ast}))}{\mathbb{E}}\phi'(u_{1})\phi'(u_{2}).
\end{equation}
Additionally, $\mathcal{M}_{Q^{\ast}}(C)$ has at most one stable fixed point in the range $[0,1]$ for any choice of $\phi$ such that $\phi$ is odd or $\phi''$ is non-negative.
\end{lemma}
Proof: See Appendix \ref{pf:fixedpoints}.

We subsequently drop the subscript in $\mathcal{M}_{Q^{\ast}}(C)$ to lighten notation. The corresponding time scale of convergence in the linearized regime is 
\begin{equation} \label{eq:xi}
\xi=-\frac{1}{\log\chi}.
\end{equation}
$\chi$ depends on the initialization hyperparameters and choice of nonlinearity, and it follows from the considerations above that signal propagation from the inputs to the outputs of a deep network would be facilitated by a choice of $\chi$ such that $\xi$ diverges, which occurs as $\chi$ approaches $1$ from below. Indeed, as observed empirically across multiple architectures and tasks \cite{xiao2018dynamical, chen2018dynamical, gilboa2019dynamical, yang2017mean}, up to a constant factor $\xi$ typically gives the maximal depth up to which a network is trainable. These calculations motivate initialization schemes that satisfy:
\[
\chi=1
\]
in order to train very deep networks. We will show shortly that this condition is unattainable for a large class of quantized activation functions. \footnote{It will at times be convenient to consider the dynamics of the correlations of the \textit{post-activations} $\widehat{\alpha}^{(l)}=\phi(\alpha^{(l)})$ which we denote by $\widehat{\mathcal{M}}(\widehat{C})$. The rates of convergence are identical in both cases, as shown in Appendix \ref{app:hidden_state}.} 

The analysis of forward signal propagation in the sense described above in networks with continuous activations is related to the stability of the gradients as well \cite{schoenholz2016deep}. The connection is obtained by relating the rate of convergence $\chi$ to the first moment of the state-to-state Jacobian
\begin{equation}\label{eq:JJ}
J=\underset{l\rightarrow\infty}{\lim}\frac{\partial\hat{\alpha}^{(l)}}{\partial\hat{\alpha}^{(l-1)}}.
\end{equation}
Taking all the layer widths to be equal to $n$, the first moment is given by 
\begin{equation}\label{eq:Mjj}
m_{JJ^{T}}=\frac{1}{n}\mathbb{E}\text{tr}\left(JJ^{T}\right).
\end{equation}
Since high powers of this matrix will appear in the gradient, controlling its spectrum can prevent the gradient from exploding or vanishing. In the case of quantized activations, however, the relationship between the Jacobian and the convergence rate $\chi$ no longer holds since the gradients vanish almost surely and modified weight update schemes such as the Straight-Through Estimator (STE) \cite{hinton2012neural,hubara2017quantized} are used instead. However, one can define a modified Jacobian $J_{\text{STE}}$ that takes the modified update scheme into account and control its moments instead. 

\begin{figure}[h]
    \centering
    \begin{subfigure}[]{1.0\textwidth}
        \centering
        \includegraphics[height=1.8in]{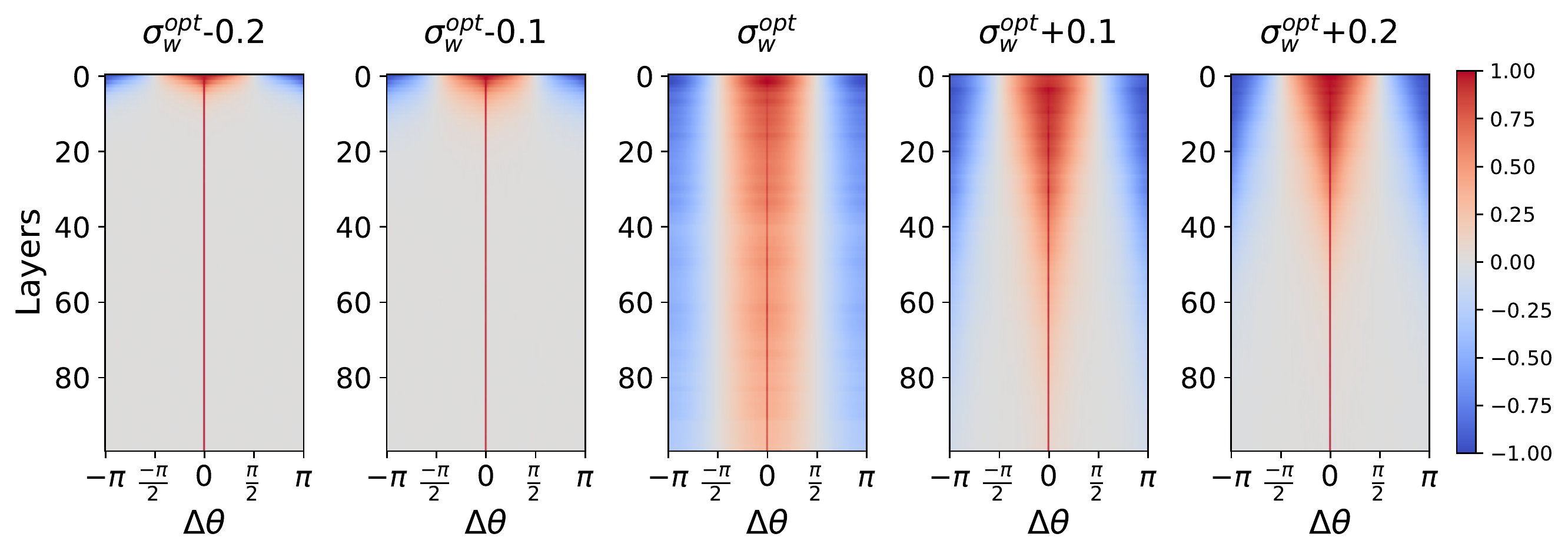}
    \end{subfigure}%
    \caption{Propagation of empirical covariance between hidden states at different layers, in quantized feed-forward networks with $N=16$, varying the standard deviation of the weights $\sigma_w$. 
    $\Delta \theta$ is the angle between two normalized inputs. Signal propagation is maximized when $\sigma_w=\sigma_w^{\text{opt}}$, and degrades as $\sigma_w$ deviates from it.}
    \label{fig:manifold_16_states}
    \vspace{-.25in}
\end{figure}

\subsection{Signal propagation may improve generalization\label{sec:NTK generalization}}
The argument that a network will be untrainable if signals cannot propagate from the inputs to the loss, corresponding to the rapid convergence of the dynamical system eq. \ref{eq:QCsys}, has empirical support across numerous architectures. A choice of initialization hyperparameters that facilitates signal propagation has also been shown to lead to slight improvements in generalization error, yet understanding of this was beyond the scope of the existing analysis. Indeed, there is also empirical evidence that when training very deep networks it is only the generalization error that is impacted adversely but the training error is not \cite{xiao2018dynamical}. Additionally, one may wonder whether a deep network could still be trainable despite a lack of signal propagation. On the one hand, rapid convergence of the correlation map between the pre-activations is equivalent to the distance between $f(x),f(x')$ converging to a value that is independent of the distance between $x,x'$. On the other, since deep networks can fit random inputs and labels \cite{zhang2016understanding} this convergence may not impede training. . 

To understand the effect of signal propagation on generalization, we consider the dynamics of learning for wide, deep neural networks in the setting studied in \cite{Jacot2018-dv,lee2019wide}. We note that this setting introduces an unconventional scaling of the weights. Despite this, it should be a good approximation for the early stages of learning in networks with standard initialization, as long as the weights do not change too much from their initial values. In this regime, the function implemented by the network evolves linearly in time, with the dynamics determined by the Neural Tangent Kernel (NTK). We argue that rapid convergence of eq.  \ref{eq:QCsys} in deep networks implies that the error at a typical test point should not decrease during training since the resulting form of the NTK will be independent of the label of the test point. Conversely, this effect will be mitigated with a choice of hyperparameters that maximizes signal propagation, which could explain the beneficial effect on generalization error that is observed empirically. We provide details and empirical evidence in support of this claim for networks with both quantized and continuous activation functions in Appendix \ref{app:NTK_STE}.

\section{Mean field theory of signal propagation with quantized activations}\label{main:genq}

In this section, we will explore the effects of using a quantized activation function on signal propagation in feed-forward networks. We will start by developing the mean field equations for a sign activations and then consider more general activation function, and establish a theory that predicts the relationship between the number of quantization states, the initialization parameters, and the feed-forward network depth.

\subsection{Warm-up: sign activations}
We begin by considering signal propagation in the network in eq. \ref{eq:net} with $\phi(x)=\text{sign}(x)$. Substituting  $\phi(x)=\text{sign}(x)$, $\phi'(u)=2\delta(u)$  in eqs. \ref{eq:QCsys} and \ref{eq:chi_basic} gives
\begin{equation}\label{eq:warmup_start}
\qu^{\ast}=\sigma_{w}^{2}+\sigma_{b}^{2},\,\,\,\chi=4\sigma_{w}^{2}\underset{(u_{1},u_{2})\sim\mathcal{N}(0,{\Sigma(\qu^{\ast},\cu^{\ast})})}{\mathbb{E}}\delta(u_{1})\delta(u_{2}).
\end{equation}
As shown in Appendix \ref{sup:sign_mean_field_theory}, we obtain
\begin{equation} \label{eq:chi_sign}
\chi=\frac{2\sigma_{w}^{2}}{\pi\left(\sigma_{w}^{2}+\sigma_{b}^{2}\right)\sqrt{1-(\cu^{\ast})^{2}}}
\end{equation}
\begin{equation} \label{eq:Mc_sign1}
\mathcal{M}(\cu)=\frac{\frac{2\sigma_{w}^2}{\pi}\arcsin\left(\cu\right)+\sigma_{b}^2}{\sigma_{w}^2+\sigma_{b}^2}.
\end{equation}
The closed form expressions \ref{eq:chi_sign} and \ref{eq:Mc_sign1}, which are not available for more complex architectures, expose the main challenge to signal propagation. It is clear from these expressions that the derivative of $\mathcal{M}(\cu)$ diverges at $1$, and since $\mathcal{M}(\cu)$ is differentiable and convex, it can have no stable fixed point in $[0,1]$ that satisfies the signal propagation condition $\chi=1$. In fact, as we show in Appendix \ref{pf:sign_sigmaB} that the maximal value of $\chi$ for this architecture is achievable when $\sigma_b=0$, and is bounded from above by $\chi_{\max}=\frac{2}{\pi}$ for all choices of the initialization hyperparameters. The corresponding depth scale is bounded by $\xi_{\max} < 3$.

Additionally, one may wonder if using stochastic binary activations \cite{hubara2017quantized} might improve signal propagation. In Appendix \ref{sup:base_stochastic_rounding} we show this is not the case: we consider a stochastic rounding quantization scheme and show that stochastic rounding can only further decrease the signal propagation depth scale.

\subsection{General quantized activations \label{sec:General Quatnization}}
We consider a general activation function $\phi_N:\mathbb{R}\to S$, where $S$ is a finite set of real numbers of size $|S|=N$. To obtain a flexible class of non-decreasing functions of this form, we define
\begin{equation}\label{eq:genq_def}
\phi_{N}(x)=A+\sum_{i=1}^{N-1}H\left(x-g_{i}\right)h_{i}\, ,
\end{equation}
where $A\in \mathbb{R}, \forall i\in \{1,2,...,N-1\}, g_i\in \mathbb{R},h_i \in \mathbb{R}_{>0}$, and $H:\mathbb{R} \to \mathbb{R}$ is the Heaviside function. This activation function can be thought of as a "stairs" function, going from the minimum state of $A$ to the maximum state $A+\sum_{i=1}^{N-1}h_{i}$, over $N-1$ stairs, with stair $i$ located at an offset $g_i$ with a height $h_i$. We will assume that the offsets $g_i$ are ordered, for simplicity.
The development of the mean field equations for this activation function is located in appendix \ref{sup:genq_mean_field}, where we find that:
\begin{equation}\label{eq:genq_q}
\Q^{(l)}=\sum_{i=1}^{N-1}\sum_{j=1}^{N-1}h_{i}h_{j}\Phi\left(-\frac{\max(g_{i},g_{j})}{\sqrt{\qu^{(l)}}}\right)\Phi\left(\frac{\min(g_{i},g_{j})}{\sqrt{\qu^{(l)}}}\right),\,\,\qu^{(l+1)}=\sigma_{w}^{2}\Q^{(l)}+\sigma_{b}^{2}
\end{equation}
\begin{equation}\label{eq:genq_chi}
\chi=\frac{\sigma_{w}^{2}}{2\pi \qu^{\ast}\sqrt{1-\left(\cu^{\ast}\right)^{2}}}\sum_{i=1}^{N-1}\sum_{j=1}^{N-1}h_{i}h_{j}\exp\left[-\frac{g_{i}^{2}-2\cu^{\ast}g_{i}g_{j}+g_{j}^{2}}{2\qu^{\ast}\left(1-\left(\cu^{\ast}\right)^{2}\right)}\right],
\end{equation}
where $\Phi$ is the gaussian CDF and $\Q^{(l)}$ is the hidden state covariance, as explained in appendix \ref{app:hidden_state}. This expression diverges as $\cu^{\ast} \rightarrow 1$ since all the summands are non-negative and the diagonal ones simplify to $\frac{\sigma_{w}^{2}h_{i}^{2}}{2\pi \qu^{\ast}\sqrt{1-\left(\cu^{\ast}\right)^{2}}}\exp\left[-\frac{g_{i}^{2}}{2\qu^{\ast}\left(1+\cu^{\ast}\right)}\right]$. Since $\mathcal{M}(\cu)$ is convex (see Lemma \ref{lem:fixedpoints}), we find that as in the case of sign activation, $\chi=1$ is not achievable for any choice of a quantized activation function.


  

To optimize the signal propagation for any given number of states, we would like to find the parameters that will bring the fixed point slope $\chi$ as close as possible to 1. For simplicity, we will henceforth use the initialization $\sigma_b=0$, which is quite common \cite{glorot2010understanding}. Empirical evidence in appendix \ref{sup:genq_mapping_approx} suggest that using $\sigma_b>0$ is sub-optimal, which is not very surprising, given our similar (exact) results for sign activation. For $\sigma_b=0$, $\cu=0$ becomes a fixed point. We eliminate eq. \ref{eq:genq_chi} direct dependency on $\qu^{\ast}$, by defining \textit{normalized offsets} $\tilde{g}\equiv\frac{g}{\sqrt{\qu^{\ast}}}$. By moving to normalized offsets, substituting $\cu^{\ast}=0$ and the remaining $\qu^{\ast}$ by eq. \ref{eq:genq_q}, our expression for the fixed point slope becomes:
\begin{equation}\label{eq:genq_chi_b0}
\chi=\frac{\sum_{i=1}^{N-1}\sum_{j=1}^{N-1}\frac{1}{2\pi}\exp\left[-\frac{1}{2}\left(\g{i}^{2}+\g{j}^{2}\right)\right]h_{i}h_{j}}{\sum_{i=1}^{N-1}\sum_{j=1}^{N-1}\Phi\left(-\max(\g{i},\g{j})\right)\Phi\left(\min(\g{i},\g{j})\right)h_{i}h_{j}}
\end{equation}
Eq. \ref{eq:genq_chi_b0} provides us with way to determine the quality of any quantized activation function in regard to signal propagation, without concerning ourselves with the initialization parameters, that will only have a linear effect on the offsets. Since the normalized offsets are sufficient to determine $\Q,\qu$, using eq. \ref{eq:genq_chi}, moving from normalized offsets to actual offsets becomes trivial. \\
To measure the relation between the number of states and depth scale, we will use eq.  \ref{eq:genq_chi_b0} over a limited set of constant-spaced activations, where we choose $A<0, \forall i\in\{1,..,N-1\},h_{i}=\mathrm{const.}$ and the offsets are evenly spaced and centered around zero, with $D$ defined as the distance between two sequential offsets so that $g_i = D \left(i-\frac{N}{2}\right)$, and $\tilde{D}$ defined as $\tilde{D}=\frac{D}{\sqrt{\qu^{\ast}}}$. We view this configuration as the most obvious selection of activation function, where the 'stairs' are evenly spaced between the minimal and the maximal state. Using eq. \ref{eq:genq_chi_b0} on an activation in this set, we get:
\begin{equation}\label{eq:genq_chi_b0_constant_spacing}
\chi=\frac{\sum_{i\in K}\sum_{j\in K}\frac{1}{2\pi}\exp\left[-\frac{1}{2}\left(i^{2}+j^{2}\right)\tilde{D}^{2}\right]}{\sum_{i\in K}\sum_{j\in K}\Phi\left(-\max\left(i,j\right)\tilde{D}\right)\Phi\left(\min\left(i,j\right)\tilde{D}\right)} 
\end{equation}
when $K=\left\{ k-\frac{N}{2}|\forall k\in\mathcal{\mathbb{N}},k<N\right\}$. A numeric analysis using of eq. \ref{eq:genq_chi_b0_constant_spacing} is presented in figure \ref{fig:linear_spacing},  and reveals a clear logarithmic relation between the level of quantization to the optimal fixed point slope, and the normalized spacing required to reach this optimal configuration. By extrapolating the numerical results, as seen in the right panels of Fig. \ref{fig:linear_spacing}, we find a good approximations for the the maximal achievable slope for any quantization level $\chi_{\max}(N)$ and the corresponding normalized spacing $D_{\text{opt}}(N)$. Using those extrapolated values, we predict the depth-scale of a quantized, feed-forward network to be:   
\begin{equation} \label{eq:xi_N}
\xi_{N}=-\frac{1}{\log(\chi_{\max}(N))}\simeq-\frac{1}{\log(1-e^{0.71}\left(N+1\right)^{-1.82})}\simeq\frac{1}{2}\left(N+1\right)^{1.82}.
\end{equation}
where the latter approximation is valid for large $N$. While the depth scale in eq. \ref{eq:xi_N} is applicable to uniformly spaced quantized activations, numerical results presented in Appendix \ref{sup:beyond_linspace} suggest that using more complex activations with the same quantization level will not produce better results. 

In their work regarding mean field theory of convolutional neural networks, \cite{xiao2018dynamical} shows that the dynamics of hidden-layer's correlations in CNNs decouple into independently evolving Fourier modes that evolves near the fixed point, each with a corresponding fixed-point-slope of $\chi_c \lambda_i$, with $\chi_c$ depending the initialization hyperparameters and equivalent to the fixed point slope as calculated for fully connected networks, and $\lambda_i \le 1$ being a frequency dependant modifier corresponding to mod $i$. While the exact dynamics in this case may depend on the decomposition of the input to Fourier mods, it is apparent that the maximal depth-scale of each mod can not exceed the depth-scale calculated for the fully-connected case, and thus our upper limit on the number of layers holds for the case of CNNs. Similarly, following \cite{chen2018dynamical} and \cite{gilboa2019dynamical}, our results can be easily extended to single layer RNNs, LSTMs and GRUS, in which case the limitation applies to the timescale of the network memory.

\begin{figure}[h]
	\centering
    \begin{subfigure}[]{0.4\textwidth}
       \centering
       \includegraphics[height=2.8in]{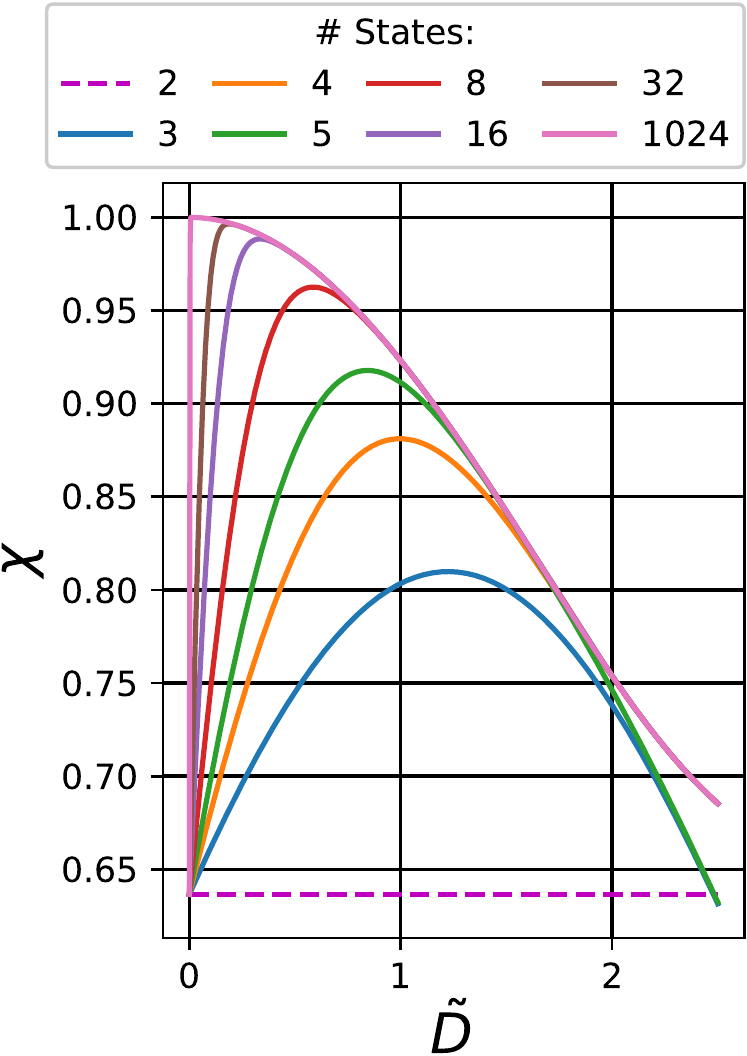}
    \end{subfigure}
    \begin{subfigure}[]{0.55\textwidth}
        \centering
        \includegraphics[height=2.6in,width=2.8in]{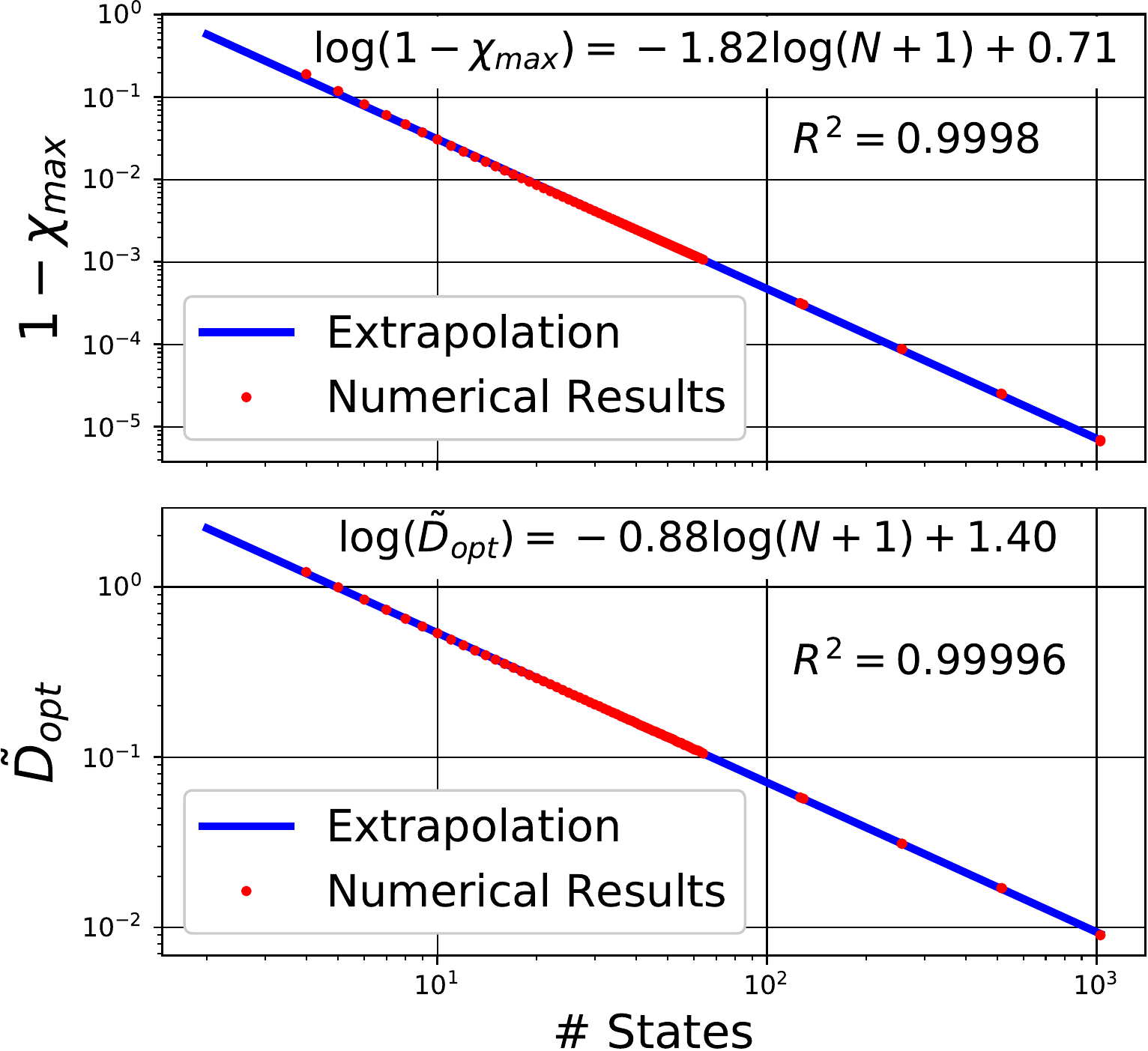}
  \end{subfigure} 
    \caption{Numerical analysis of the covariance propagation fixed point slope for quantized activation functions. \textit{Left:} The convergence rate in eq. \ref{eq:genq_chi_b0_constant_spacing} of the covariances of the hidden states as a function of the normalized spacing between offsets $\tilde{D}$ for activations with different levels of quantization $N$. \textit{Top Right:} The difference between $1$ and maximal achievable convergence rate $\chi_{\text{max}}$ as a function of $N$. \textit{Bottom Right:} The normalized spacing between states $\tilde{D}$ corresponding to $\chi_{\text{max}}$ as a function of $N$. We find that the dependence of $1-\chi_{\text{max}}$ on $N$ is approximated well by a power law.
    }
    \label{fig:linear_spacing} 
    \vspace{-.2in}
\end{figure}

\section{Experimental results}\label{sec:experiment}

To visualize the covariance propagation in eq.~\ref{eq:cov_udpate} we reconstruct an experiment presented in \cite{poole2016exponential}, and apply it to untrained quantized neural networks. We consider a neural network with $L=100$ fully-connected layers, all of width $n=1000$. We draw two orthonormal vectors $u^0,u^1\in\mathbb{R}^{1000}$ and generate the 1 dimensional manifold $U=\left\{ u_i=\sqrt{Q_{s}^{\ast}}\left(u^{0}\cos(\theta)+u^{1}\sin(\theta)\right)|i\in\{0,\frac{1}{r},\frac{2}{r},..,\frac{r-1}{r}\},\theta=2\pi i\right\} $, where $r=500$ is the number of samples, and $Q^{\ast}_s$ is the fixed point, calculated numerically. After initializing the neural network, we use the manifold values as inputs to the neural network and measure the covariance in all hidden layers. We then plot in Figure \ref{fig:manifold_16_states} the empirical covariance of the hidden states as a function of the difference in the angle $\theta$ of their corresponding inputs. The reason for multiplying the initial values by $\sqrt{Q^{\ast}_s}$ is so we can isolate the convergence of the off-diagonal correlations from that of the diagonal. 

To test the predictions of the theory, we have constructed a similar experiment to the one described in \cite{schoenholz2016deep}, training neural networks of varying depths over the MNIST dataset. We study how the maximal trainable depth of a quantized activation fully-connected network depends on the weight variance $\sigma^2_w$ and the number of states in the activation function $N$. For our quantized activations, we used the constant-spaced activations we have analyzed in section \ref{sec:General Quatnization}:
\[
\phi_{N}(x)=-1+\sum_{i=1}^{N-1}\frac{2}{N-1}H\left(x-\frac{2}{N-1}\left( i - \frac{N}{2} \right)\right),
\]
which describes an activation function with a distance of $D=\frac{2}{N-1}$ between offsets, and with states ranging between -1 and 1. 

To find the best initialization parameters for each activation function, we first used eq. \ref{eq:genq_q} to compute $\Q^{\ast}$ assuming our normalized spacing $\frac{D}{\sqrt{\qu^{\ast}}}$ is optimized ($\tilde{D}_{\text{opt}}$, computed using the linear regression parameters of Figure \ref{fig:linear_spacing} bottom right panel). Then, we picked $\sigma_b=0$, $\sigma_w= \frac{1}{\sqrt{\Q^{\ast}}}\frac{D}{\tilde{D}_{\mathrm{opt}}}$, and thus ensured that the normalized offsets are indeed optimal. Gradients are computed using the Straight-Through Estimator (STE) \cite{hubara2017quantized}:
\begin{equation} \label{eq:ste_rho}
\Delta_{\mathrm{input}}=\begin{cases}
\Delta_{\mathrm{output}} & \left|\mathrm{input}\right|<1\\
0 & \text{else}
\end{cases}\,,
\end{equation}
where $\Delta_{\text{output}}$ is gradient we get from the next layer and $\Delta_{\text{input}}$ is the gradient we pass to the preceding layer. The conditions required for allowing the gradients information to propagate backward are discussed in appendix \ref{sup:backward}. Those conditions are not enforced in this experiment, as they have no significant effect on the results, as shown in appendix \ref{sup:MNIST}, where we add more results that isolate the forward-pass from the backward pass. Also included in appendix \ref{sup:MNIST} are results that show the evolution of the training and test accuracy in training time. A simplified initialization scheme for the use of practitioners is included in appendix \ref{sup:practition}.

We set the hidden layer width to 2048. We use SGD for training, a learning rate of $10^{-3}$ for networks with 10-90 layers, and a learning rate of $5 \times 10^{-4}$  when training 100-220 layers. Those parameters were selected to match those reported in \cite{schoenholz2016deep}, with the second learning rate adjusted to fit our area-of-search. We also use a batch size of 32, and use a standard preprocessing of the MNIST input\footnote{The code for running this experiment and more is provided in \git.}.

\begin{figure}
    \centering
    \begin{subfigure}[]{0.5\textwidth}
        \centering
        \includegraphics[height=2.0in,width=2.5in]{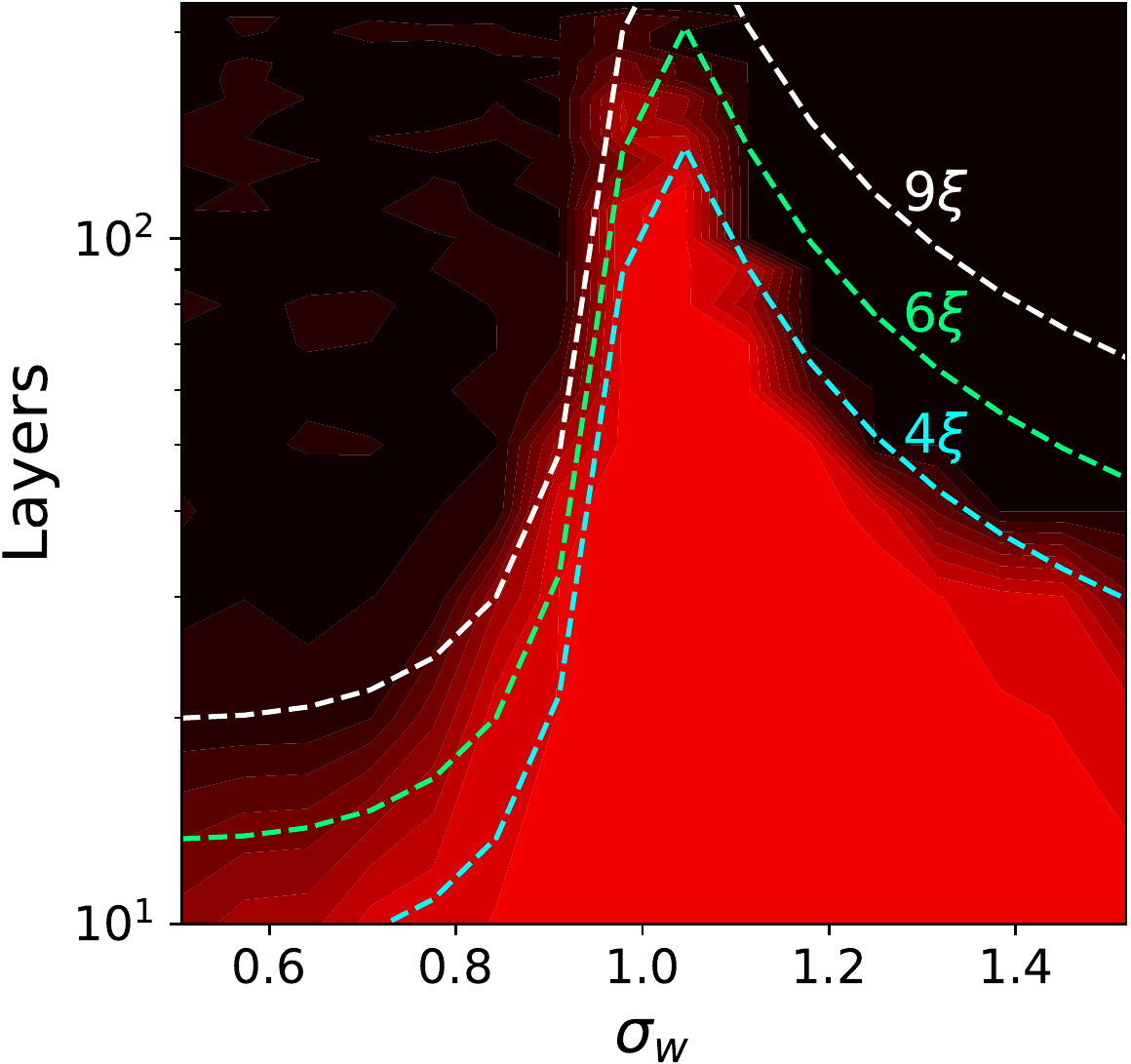}
    \end{subfigure}%
    ~
    \begin{subfigure}[]{0.5\textwidth}
        \centering
        \includegraphics[height=2.0in,width=2.7in]{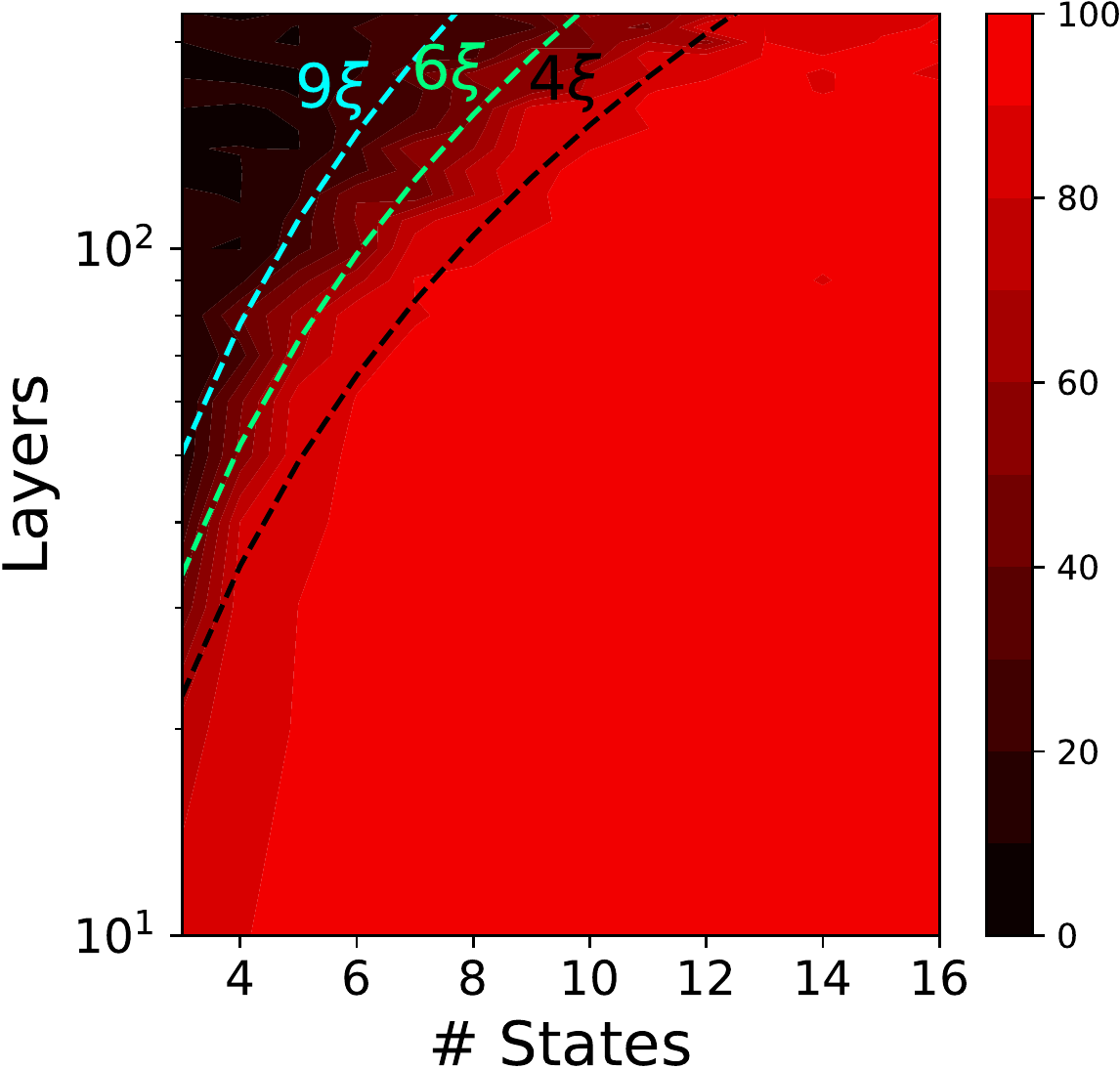}
    \end{subfigure}%
    \caption{Test accuracy of feed-forward networks of different depth with quantized activation functions trained on MNIST classification after 1600 training steps, compared with the theoretical depth scale predictions (eq. \ref{eq:xi}). Up to a constant factor, the theoretical depth scale predicts the phase transition between regimes where a network is trainable and one where training fails. \textit{Left:} Networks with a 10 states activation function and different values of the weight variance. \textit{Right:} Networks with different quantization levels (number of states), with variances adjusted to allow optimal signal propagation. 
    }
    \vspace{-.2in}
    \label{fig:MNIST_experiment}
\end{figure}
    
Figure \ref{fig:MNIST_experiment} shows that the initialization of the network using the parameters suggested by our theory achieves the optimal trainability when the number of layers is high. When measuring test accuracy at the early stage of the network, we can see that the accuracy is high when the network has $\sim4\xi$ layers or less. As demonstrated by the advanced training stage results shown in appendix \ref{sup:MNIST}, and by the results of \cite{schoenholz2016deep}, networks of depth exceeding $\sim6 \xi$ appear to be untrainable.
   \vspace{-.08in}
\section{Discussion}
    \vspace{-.08in}
In this paper, we study the effect of using quantized activations on the propagation of signals in deep neural networks, from the inputs to the outputs. We focus on quantized activations, which maps its input to a finite set of $N$ possible outputs. Our analysis suggests an initialization scheme that improves network trainability, and that fully-connected/convolutional networks to become untrainable when the number of layers exceeds $L_{\max}\sim3\left(N+1\right)^{1.82}$.

Additionally, we propose a possible explanation for the improved generalization observed when training networks that are initialized to enable stable signal propagation. While the motivation for the critical initialization has been improved trainability \cite{schoenholz2016deep}, empirically these initialization schemes were shown to improve generalization as well, an observation that was beyond the scope of the analysis which motivated them. By considering the dynamics of learning in wide networks that exhibit poor signal propagation, we find that generalization error in the early stages of training will typically not improve. This effect will be minimized when using a critical initialization. 

The limitations of poor signal propagation can perhaps be overcome with certain modifications to the architecture or training procedure. Residual connections, for example, can be initialized \cite{zhang2019fixup} to maintain the signal propagation conditions even when the full-network depth exceeds our theoretical limit \cite{yang2017mean}. Another possible modification is batch normalization, which we did not consider in the analysis. While batch normalization by itself was shown to have negative side effects on the signal propagation \cite{yang2019mean}, other studies  \cite{banner2018scalable,courbariaux2016binarized,hubara2017quantized} have already suggested that applying proper batch normalization is key when training quantized feed-forward networks. There are, however, cases where batch normalization does not work well, like in the case of recurrent neural networks. We expect our findings to have as increased significance if generalized to such architectures, as was done previously for continuous activations \cite{chen2018dynamical, gilboa2019dynamical}.

\ifthenelse{\boolean{blind}}
{

}
{
\subsection*{Acknowledgements}
The work of DS was supported by the Israel Science foundation (grant No. 31/1031), the Taub Foundation and used a Titan Xp donated by the NVIDIA Corporation. The work of DG was supported by the NSF NeuroNex Award DBI-1707398 and the Gatsby Charitable Foundation. The work of DG and DS was done in part while the authors were visiting the Simons Institute for the Theory of Computing.
}

\bibliographystyle{plain}
\bibliography{bibliography}

\begin{thebibliography}{10}

\bibitem{Anderson2017}
Alexander~G. Anderson and Cory~P. Berg.
\newblock {The High-Dimensional Geometry of Binary Neural Networks}.
\newblock {\em ICLR}, (2014):1--13, 2018.

\bibitem{arora2019exact}
Sanjeev Arora, Simon~S Du, Wei Hu, Zhiyuan Li, Ruslan Salakhutdinov, and
  Ruosong Wang.
\newblock On exact computation with an infinitely wide neural net.
\newblock {\em arXiv preprint arXiv:1904.11955}, 2019.

\bibitem{banner2018scalable}
Ron Banner, Itay Hubara, Elad Hoffer, and Daniel Soudry.
\newblock Scalable methods for 8-bit training of neural networks.
\newblock In {\em Advances in Neural Information Processing Systems}, pages
  5145--5153, 2018.

\bibitem{chen2018dynamical}
Minmin Chen, Jeffrey Pennington, and Samuel~S Schoenholz.
\newblock Dynamical isometry and a mean field theory of rnns: Gating enables
  signal propagation in recurrent neural networks.
\newblock {\em arXiv preprint arXiv:1806.05394}, 2018.

\bibitem{chen2015compressing}
Wenlin Chen, James Wilson, Stephen Tyree, Kilian Weinberger, and Yixin Chen.
\newblock Compressing neural networks with the hashing trick.
\newblock In {\em International Conference on Machine Learning}, pages
  2285--2294, 2015.

\bibitem{courbariaux2016binarized}
Matthieu Courbariaux, Itay Hubara, Daniel Soudry, Ran El-Yaniv, and Yoshua
  Bengio.
\newblock Binarized neural networks.
\newblock {\em Advances in Neural Information Processing Systems}, 2016.

\bibitem{das2018mixed}
Dipankar Das, Naveen Mellempudi, Dheevatsa Mudigere, Dhiraj Kalamkar, Sasikanth
  Avancha, Kunal Banerjee, Srinivas Sridharan, Karthik Vaidyanathan, Bharat
  Kaul, Evangelos Georganas, et~al.
\newblock Mixed precision training of convolutional neural networks using
  integer operations.
\newblock {\em arXiv preprint arXiv:1802.00930}, 2018.

\bibitem{gilboa2019dynamical}
Dar Gilboa, Bo~Chang, Minmin Chen, Greg Yang, Samuel~S Schoenholz, Ed~H Chi,
  and Jeffrey Pennington.
\newblock Dynamical isometry and a mean field theory of lstms and grus.
\newblock {\em arXiv preprint arXiv:1901.08987}, 2019.

\bibitem{glorot2010understanding}
Xavier Glorot and Yoshua Bengio.
\newblock Understanding the difficulty of training deep feedforward neural
  networks.
\newblock In {\em Proceedings of the thirteenth international conference on
  artificial intelligence and statistics}, pages 249--256, 2010.

\bibitem{gupta2015deep}
Suyog Gupta, Ankur Agrawal, Kailash Gopalakrishnan, and Pritish Narayanan.
\newblock Deep learning with limited numerical precision.
\newblock In {\em International Conference on Machine Learning}, pages
  1737--1746, 2015.

\bibitem{hinton2012neural}
G~Hinton.
\newblock Neural networks for machine learning. coursera,[video lectures],
  2012.

\bibitem{howard2017mobilenets}
Andrew~G Howard, Menglong Zhu, Bo~Chen, Dmitry Kalenichenko, Weijun Wang,
  Tobias Weyand, Marco Andreetto, and Hartwig Adam.
\newblock Mobilenets: Efficient convolutional neural networks for mobile vision
  applications.
\newblock {\em arXiv preprint arXiv:1704.04861}, 2017.

\bibitem{hubara2017quantized}
Itay Hubara, Matthieu Courbariaux, Daniel Soudry, Ran El-Yaniv, and Yoshua
  Bengio.
\newblock Quantized neural networks: Training neural networks with low
  precision weights and activations.
\newblock {\em The Journal of Machine Learning Research}, 18(1):6869--6898,
  2017.

\bibitem{Jacot2018-dv}
Arthur Jacot, Franck Gabriel, and Cl{\'e}ment Hongler.
\newblock Neural tangent kernel: Convergence and generalization in neural
  networks.
\newblock June 2018.

\bibitem{lee2017deep}
Jaehoon Lee, Yasaman Bahri, Roman Novak, Samuel~S Schoenholz, Jeffrey
  Pennington, and Jascha Sohl-Dickstein.
\newblock Deep neural networks as gaussian processes.
\newblock {\em arXiv preprint arXiv:1711.00165}, 2017.

\bibitem{lee2019wide}
Jaehoon Lee, Lechao Xiao, Samuel~S Schoenholz, Yasaman Bahri, Jascha
  Sohl-Dickstein, and Jeffrey Pennington.
\newblock Wide neural networks of any depth evolve as linear models under
  gradient descent.
\newblock {\em arXiv preprint arXiv:1902.06720}, 2019.

\bibitem{Li2017b}
Hao Li, Soham De, Zheng Xu, Christoph Studer, Hanan Samet, and Tom Goldstein.
\newblock {Training Quantized Nets: A Deeper Understanding}.
\newblock {\em NIPS}, jun 2017.

\bibitem{lin2017towards}
Xiaofan Lin, Cong Zhao, and Wei Pan.
\newblock Towards accurate binary convolutional neural network.
\newblock In {\em Advances in Neural Information Processing Systems}, pages
  345--353, 2017.

\bibitem{Maass1998}
Wolfgang Maass and Pekka Orponen.
\newblock {On the Effect of Analog Noise in Discrete-Time Analog Computations}.
\newblock {\em Neural Computation}, 10(5):1071--1095, jul 1998.

\bibitem{matthews2018gaussian}
Alexander G de~G Matthews, Mark Rowland, Jiri Hron, Richard~E Turner, and
  Zoubin Ghahramani.
\newblock Gaussian process behaviour in wide deep neural networks.
\newblock {\em arXiv preprint arXiv:1804.11271}, 2018.

\bibitem{mishra2017wrpn}
Asit Mishra, Eriko Nurvitadhi, Jeffrey~J Cook, and Debbie Marr.
\newblock Wrpn: wide reduced-precision networks.
\newblock {\em arXiv preprint arXiv:1709.01134}, 2017.

\bibitem{miyashita2016convolutional}
Daisuke Miyashita, Edward~H Lee, and Boris Murmann.
\newblock Convolutional neural networks using logarithmic data representation.
\newblock {\em arXiv preprint arXiv:1603.01025}, 2016.

\bibitem{pennington2017resurrecting}
Jeffrey Pennington, Samuel Schoenholz, and Surya Ganguli.
\newblock Resurrecting the sigmoid in deep learning through dynamical isometry:
  theory and practice.
\newblock In {\em Advances in neural information processing systems}, pages
  4785--4795, 2017.

\bibitem{poole2016exponential}
Ben Poole, Subhaneil Lahiri, Maithra Raghu, Jascha Sohl-Dickstein, and Surya
  Ganguli.
\newblock Exponential expressivity in deep neural networks through transient
  chaos.
\newblock In {\em Advances in neural information processing systems}, pages
  3360--3368, 2016.

\bibitem{rastegari2016xnor}
Mohammad Rastegari, Vicente Ordonez, Joseph Redmon, and Ali Farhadi.
\newblock Xnor-net: Imagenet classification using binary convolutional neural
  networks.
\newblock In {\em European Conference on Computer Vision}, pages 525--542.
  Springer, 2016.

\bibitem{schoenholz2016deep}
Samuel~S Schoenholz, Justin Gilmer, Surya Ganguli, and Jascha Sohl-Dickstein.
\newblock Deep information propagation.
\newblock {\em arXiv preprint arXiv:1611.01232}, 2016.

\bibitem{Siegelmann1991}
Hava~T. Siegelmann and Eduardo~D. Sontag.
\newblock {Turing computability with neural nets}.
\newblock {\em Applied Mathematics Letters}, 4(6):77--80, jan 1991.

\bibitem{wang2018training}
Naigang Wang, Jungwook Choi, Daniel Brand, Chia-Yu Chen, and Kailash
  Gopalakrishnan.
\newblock Training deep neural networks with 8-bit floating point numbers.
\newblock In {\em Advances in neural information processing systems}, pages
  7675--7684, 2018.

\bibitem{wu2018deterministic}
Anqi Wu, Sebastian Nowozin, Edward Meeds, Richard~E. Turner, Jose~Miguel
  Hernandez-Lobato, and Alexander~L. Gaunt.
\newblock Deterministic variational inference for robust bayesian neural
  networks.
\newblock In {\em International Conference on Learning Representations}, 2019.

\bibitem{xiao2018dynamical}
Lechao Xiao, Yasaman Bahri, Jascha Sohl-Dickstein, Samuel~S Schoenholz, and
  Jeffrey Pennington.
\newblock Dynamical isometry and a mean field theory of cnns: How to train
  10,000-layer vanilla convolutional neural networks.
\newblock {\em arXiv preprint arXiv:1806.05393}, 2018.

\bibitem{yang2019mean}
Greg Yang, Jeffrey Pennington, Vinay Rao, Jascha Sohl-Dickstein, and Samuel~S
  Schoenholz.
\newblock A mean field theory of batch normalization.
\newblock {\em arXiv preprint arXiv:1902.08129}, 2019.

\bibitem{yang2017mean}
Greg Yang and Samuel Schoenholz.
\newblock Mean field residual networks: On the edge of chaos.
\newblock In {\em Advances in neural information processing systems}, pages
  7103--7114, 2017.

\bibitem{Yin2019}
Penghang Yin, Jiancheng Lyu, Shuai Zhang, Stanley Osher, Yingyong Qi, and Jack
  Xin.
\newblock {Understanding straight-through estimator in training activation
  quantized neural nets}.
\newblock {\em ICLR}, pages 1--30, 2019.

\bibitem{zhang2016understanding}
Chiyuan Zhang, Samy Bengio, Moritz Hardt, Benjamin Recht, and Oriol Vinyals.
\newblock Understanding deep learning requires rethinking generalization.
\newblock {\em arXiv preprint arXiv:1611.03530}, 2016.

\bibitem{zhang2019fixup}
Hongyi Zhang, Yann~N Dauphin, and Tengyu Ma.
\newblock Fixup initialization: Residual learning without normalization.
\newblock {\em arXiv preprint arXiv:1901.09321}, 2019.

\bibitem{zhou2018adaptive}
Yiren Zhou, Seyed-Mohsen Moosavi-Dezfooli, Ngai-Man Cheung, and Pascal
  Frossard.
\newblock Adaptive quantization for deep neural network.
\newblock In {\em Thirty-Second AAAI Conference on Artificial Intelligence},
  2018.

\end{thebibliography}

\newpage
\appendix
\part*{Appendix}
\section{Proof of Lemma 1}\label{pf:fixedpoints}
\begin{proof}[Proof of Lemma \ref{lem:fixedpoints}] 
The dynamical system is given by 
\begin{equation} \label{eq:QCsys2}
\left(\begin{array}{c}
Q^{(l)}\\
C^{(l)}
\end{array}\right)=\left(\begin{array}{c}
\sigma_{w}^{2}\underset{u\sim\mathcal{N}(0,Q^{(l-1)})}{\mathbb{E}}\phi^{2}(u)+\sigma_{b}^{2}\\
\frac{1}{Q^{(l-1)}}\left[\sigma_{w}^{2}\underset{(u_{1},u_{2})\sim\mathcal{N}(\mean,\Sigma(Q^{(l-1)},C^{(l-1)}))}{\mathbb{E}}\phi(u_{1})\phi(u_{2})+\sigma_{b}^{2}\right]
\end{array}\right)\equiv\mathcal{M}\left[\left(\begin{array}{c}
Q^{(l-1)}\\
C^{(l-1)}
\end{array}\right)\right].
\end{equation}

Since $Q^{(l)}=\sigma_{w}^{2}\widehat{Q}^{(l-1)}+\sigma_{b}^{2}$, convergence of $Q^{(l)}$ to a fixed point is equivalent to convergence of $\widehat{Q}^{(l)}$. If we assume $Q^{(l)}$ has converged to $Q^\ast$, the system in eq. \ref{eq:QCsys2} reduces to
\begin{equation}
\mathcal{M}_{Q^{\ast}}(C)=\frac{1}{Q^{\ast}}\left[\sigma_{w}^{2}\underset{(u_{1},u_{2})\sim\mathcal{N}(\mean,\Sigma(Q^{\ast},C))}{\mathbb{E}}\phi(u_{1})\phi(u_{2})+\sigma_{b}^{2}\right]
\end{equation}
Linearizing the above equation gives

\[
\mathcal{M}_{Q^{\ast}}(C)=\mathcal{M}_{Q^{\ast}}(C^{\ast})+\underbrace{\frac{\partial\mathcal{M}_{Q^{\ast}}(C^{\ast})}{\partial C}}_{\equiv\chi}(C-C^{\ast})+O\left((C-C^{\ast})^{2}\right)
\]

and using a Cholesky decomposition and denoting by $\mathcal{D}x$ a standard Gaussian measure, we have 
\[
\chi_{C^\ast}=\frac{1}{Q^{\ast}}\frac{\partial}{\partial C}\left[\sigma_{w}^{2}\underset{(u_{1},u_{2})\sim\mathcal{N}(\mean,\Sigma(Q^{\ast},C))}{\mathbb{E}}\phi(u_{1})\phi(u_{2})+\sigma_{b}^{2}\right]_{C=C^{\ast}}
\]
\[
=\frac{\sigma_{w}^{2}}{Q^{\ast}}\int\mathcal{D}z_{1}\mathcal{D}z_{2}\phi(\sqrt{Q_{\ast}}z_{1}+\mu_{b})\frac{\partial}{\partial C}\phi(\sqrt{Q_{\ast}}\left(Cz_{1}+\sqrt{1-C^{2}}z_{2}\right)+\mu_{b})_{C=C^{\ast}}
\]
\[
=\frac{\sigma_{w}^{2}}{Q^{\ast}}\int\mathcal{D}z_{1}\mathcal{D}z_{2}\phi(\sqrt{Q_{\ast}}z_{1}+\mu_{b})\phi'(\sqrt{Q_{\ast}}\left(Cz_{1}+\sqrt{1-C^{2}}z_{2}\right)+\mu_{b})\sqrt{Q_{\ast}}\left(z_{1}-\frac{z_{2}C}{\sqrt{1-C^{2}}}\right)
\]
and using $\int\mathcal{D}zg(z)z=\int\mathcal{D}zg'(z)$ which holds for any $g(z)$

\[
=\sigma_{w}^{2}\underset{(u_{1},u_{2})\sim\mathcal{N}(\mean,\Sigma(Q^{\ast},C))}{\mathbb{E}}\phi'(u_{1})\phi'(u_{2}).
\]

The time scale of convergence dictated by the rate $\chi$ is obtained by solving the linear equation for $\varepsilon^{(l)}=C^{(l)}-C^{\ast}$, which gives $\varepsilon^{(l)}=\varepsilon_{0}e^{-l/\xi}$ and thus in the linear regime we have 
\[
e^{-1/\xi}=\frac{\varepsilon^{(l+1)}}{\varepsilon^{(l)}}=\frac{\mathcal{M}_{Q^{\ast}}(C^{(l)})-C^{\ast}}{C^{(l)}-C^{\ast}}\approx\frac{C^{\ast}+\chi\left(C^{(l)}-C^{\ast}\right)-C^{\ast}}{C^{(l)}-C^{\ast}}=\chi
\]

\[
\xi=-\frac{1}{\log\chi}.
\]

Since a smooth convex function can intersect a linear function at no more than two points unless the two are equal (since otherwise the gradient must change sign twice implying negative curvature at some point), in order to show that $\mathcal{M}_{Q^{\ast}}(C)$ can have at most two fixed points in $[0,1]$ it suffices to show that it is convex in this range. A calculation similar to the one above gives:
\[
\frac{\partial^{2}\mathcal{M}_{Q_{\ast}}(C)}{\partial C^{2}}=\sigma_{w}^{2}Q_{\ast}\underset{(u_{1},u_{2})\sim\mathcal{N}(\mean,\Sigma(Q^{\ast},C))}{\mathbb{E}}\phi''(u_{1})\phi''(u_{2}).
\]

If $\phi$ is odd, so is $\phi''$ and then the expression above is non-negative for $C \in [0,1] $ according to Lemma 2 in \cite{gilboa2019dynamical}. It is obviously also non-negative simply if $\phi''$ is uniformly non-negative. The result applies to quantized activation as well since we can replace the Heaviside function with a smooth approximation that is identical to within machine precision, and apply the above argument.

Since a fixed point is only stable if the slope $\chi$ is smaller than $1$ and there are at most two fixed points in $[0,1]$, there can be at most one stable fixed point. It follows that the fixed point of the dynamics does not depend on initialization as long as $C^{(0)} \geq 0$. While there may be another stable fixed point in $[-1,0)$, the network will still be unable to distinguish between any two inputs that are either completely uncorrelated or positively correlated, which will generally prevent learning aside from trivial tasks where data points in different classes are always negatively correlated, and thus the data is linearly separable.  
\end{proof}
\section{Covariances of post-activations} \label{app:hidden_state}

In the main text we review results on asymptotic normality of pre-activations $\alpha^{(l)}(x)$ of deep feed-forward networks at the infinite width limit. The analysis of signal propagation in such networks is based on studying convergence of the covariances of these pre-activations to their fixed points. The convergence rate in eq. \ref{eq:chi_basic} and the corresponding time scale in eq. \ref{eq:xi} that gives the typical maximal trainable depth are thus the main objects of interest. 

It will be convenient at times to consider instead the evolution of the covariances of the post-activations $\widehat{\alpha}^{(l)}(x)=\phi(\alpha^{(l)}(x))$. We do this by defining, analogously to eq. \ref{eq:QC_def},
\[
\left(\begin{array}{cc}
\mathbb{E}\widehat{\alpha}_{i}^{(l)}(x)\widehat{\alpha}_{i}^{(l)}(x) & \mathbb{E}\widehat{\alpha}_{i}^{(l)}(x)\widehat{\alpha}_{i}^{(l)}(x')\\
\mathbb{E}\widehat{\alpha}_{i}^{(l)}(x)\widehat{\alpha}_{i}^{(l)}(x') & \mathbb{E}\widehat{\alpha}_{i}^{(l)}(x')\widehat{\alpha}_{i}^{(l)}(x')
\end{array}\right)=\left(\begin{array}{cc}
\widehat{\Sigma}^{(l)}(x,x) & \widehat{\Sigma}^{(l)}(x,x')\\
\widehat{\Sigma}^{(l)}(x,x') & \widehat{\Sigma}^{(l)}(x',x')
\end{array}\right)+\left(\widehat{\mu}^{(l)}\right)^{2}\left(\begin{array}{cc}
1 & 1\\
1 & 1
\end{array}\right)
\]
\begin{equation}\label{eq:realQC}
=\widehat{Q}^{(l)}\left(\begin{array}{cc}
1 & \widehat{C}^{(l)}\\
\widehat{C}^{(l)} & 1
\end{array}\right)+\left(\widehat{\mu}^{(l)}\right)^{2}\left(\begin{array}{cc}
1 & 1\\
1 & 1
\end{array}\right)
\end{equation}
For a given $x,x'$ the quantities $Q^{(l)},C^{(l)}$ are trivially related to $\widehat{\mu}^{(l-1)},\widehat{Q}^{(l-1)},\widehat{C}^{(l-1)}$ via eq. \ref{eq:cov_udpate}, which gives 
\[
Q^{(l)}=\sigma_{w}^{2}\left(\widehat{Q}^{(l-1)}+\left(\widehat{\mu}^{(l-1)}\right)^{2}\right)+\sigma_{b}^{2}
\]
\[
C^{(l)}=\frac{\sigma_{w}^{2}\left(\widehat{Q}^{(l-1)}\widehat{C}^{(l-1)}+\left(\widehat{\mu}^{(l-1)}\right)^{2}\right)+\sigma_{b}^{2}}{Q^{(l)}}.
\]
The covariance map for the hidden states analogous to eq. \ref{eq:M} is simply
\begin{equation}\label{eq:realM}
\mathcal{\widehat{M}}_{\widehat{\mu}^{\ast},\widehat{Q}^{\ast}}(\widehat{C})=\frac{1}{\widehat{Q}^{\ast}}\underset{(u_{1},u_{2})\sim\mathcal{N}(\widehat{\mu}^{\ast},\widehat{\Sigma}(\widehat{Q}^{\ast},\widehat{C}))}{\mathbb{E}}\phi(u_{1})\phi(u_{2})
\end{equation}
where $\widehat{\Sigma}(\widehat{Q}^{\ast},\widehat{C})=\left(\begin{array}{cc}
\sigma_{w}^{2}\widehat{Q}^{\ast}+\sigma_{b}^{2} & \sigma_{w}^{2}\widehat{Q}^{\ast}\widehat{C}+\sigma_{b}^{2}\\
\sigma_{w}^{2}\widehat{Q}^{\ast}\widehat{C}+\sigma_{b}^{2} & \sigma_{w}^{2}\widehat{Q}^{\ast}+\sigma_{b}^{2}
\end{array}\right)$. 
The convergence rates for \ref{eq:M} are identical since
\[
\frac{\partial\mathcal{M}_{Q^{\ast}}(C^{(l)})}{\partial C^{(l)}}=\frac{\partial C^{(l+1)}}{\partial C^{(l)}}=\frac{1}{Q^{\ast}}\frac{\partial\sigma_{w}^{2}\left(\widehat{Q}^{\ast}\widehat{C}^{(l)}\right)+\sigma_{b}^{2}}{\partial C^{(l)}}
\]
\[
=\frac{\sigma_{w}^{2}\widehat{Q}^{\ast}}{Q^{\ast}}\frac{\partial\widehat{C}^{(l)}}{\partial\widehat{C}^{(l-1)}}\frac{\partial\widehat{C}^{(l-1)}}{\partial C^{(l)}}=\frac{\partial\widehat{C}^{(l)}}{\partial\widehat{C}^{(l-1)}}=\frac{\partial\widehat{\mathcal{M}}_{\widehat{\mu}^{\ast},\widehat{Q}^{\ast}}(\widehat{C}^{(l-1)})}{\partial\widehat{C}^{(l-1)}}
\]
giving
\[
\chi=\underset{l\rightarrow\infty}{\lim}\frac{\partial\mathcal{M}_{Q^{\ast}}(C^{(l)})}{\partial C^{(l)}}=\underset{l\rightarrow\infty}{\lim}\frac{\partial\widehat{\mathcal{M}}_{\widehat{\mu}^{\ast},\widehat{Q}^{\ast}}(\widehat{C}^{(l-1)})}{\partial\widehat{C}^{(l-1)}}=\widehat{\chi}.
\]

\section{Calculation of the fixed point slope for sign-activation}\label{sup:sign_mean_field_theory}
For convinience, we use the hidden states covariances and mapping $\C,\Q,\M$ as defined in appendix \ref{app:hidden_state}, as they have a linear relationship to the pre-activation at the fixed point. Using a Cholesky decomposition on the equation \ref{eq:warmup_start}: $\chi=4\sigma_{w}^{2}\underset{(u_{a},u_{b})\sim\mathcal{N}(0,\mathbf{\Sigma}(\qu^{\ast},\cu^{\ast}))}{\mathbb{E}}\delta(u_{a})\delta(u_{b})$, we get
\[
4\sigma_{w}^{2}\underset{u_{1}}{\int}\underset{u_{2}}{\int}\frac{1}{2\pi}\exp\left(-\frac{u_{1}^{2}+u_{2}^{2}}{2}\right)\delta(\sqrt{\qu^{\ast}}u_{1})\delta\left(\sqrt{\qu^{\ast}}\left(\cu^{\ast}u_{1}+\sqrt{1-(\cu^{\ast})^{2}}u_{2}\right)\right)du_{1}du_{2}.
\]
The delta functions enforces: $u_{1}=0,\mu_{2}=0$, giving us
\[\chi=\frac{2}{\pi}\frac{\sigma_{w}^{2}}{\qu^{\ast}\sqrt{1-(\cu^{\ast})^{2}}}.\]
Then, using $\qu^{\ast}=\sigma_{w}^{2}\Q^{\ast}+\sigma_{b}^{2}$, and since  $\Q^{\ast}=1$ for sign activation:
\[\chi=\frac{2}{\pi}\frac{\sigma_{w}^{2}}{(\sigma_{w}^{2}+\sigma_{b}^{2})\sqrt{1-(\cu^{\ast})^{2}}}.\]
While this equation is written for the fixed point $\cu^{\ast}$, this equation can describe the slope of $\mathcal{M}(C)$ for every value of $\cu$. Rather than directly calculating $\mathcal{M}(\cu)$ using equation \ref{eq:QCsys}, it is surprisingly time saving to calculate it by using our expression for  $ \chi(\cu)=\frac{d\mathcal{M}(\cu)}{d\cu}$:

\[\mathcal{M}(\cu)-\text{const}=\int_{0}^{\cu}\chi(\cu')d\cu'=\frac{2}{\pi}\frac{\sigma_{w}^{2}}{(\sigma_{w}^{2}+\sigma_{b}^{2})}\arcsin(\cu).\]
We know that $\mathcal{M}(\cu=1)=1$, from which we can compute the constant
\[\text{const}=\mathcal{M}(1)-\frac{2}{\pi}\frac{\sigma_{w}^{2}}{(\sigma_{w}^{2}+\sigma_{b}^{2})}\arcsin(1)=\frac{\sigma_{b}^2}{\sigma_{w}^2+\sigma_{b}^2}.\]
In conclusion:
\[\mathcal{M}(\cu)=\frac{\frac{2\sigma_{w}^2}{\pi}\arcsin\left(\cu\right)+\sigma_{b}^2}{\sigma_{w}^2+\sigma_{b}^2}\]
It's also worth noting that for the hidden-states, the mapping for sign activation is:

\[\M(\C)=\frac{2}{\pi}\arcsin\left(\frac{\C\sigma_{w}^2+\sigma_{b}^2}{\sigma_{w}^2+\sigma_{b}^2}\right)\]

In addition to the fixed point $\M(\C=1)=1$, the covariance mapping function suggests an additional fixed point within the range $[0,1)$. In the case of $\sigma_{b}^{2}=0$, The entire network becomes anti-symmetric upon initialization and $C=-1$ becomes an infinitely unstable fixed point as well.

\section{Stochastic Rounding}\label{sup:base_stochastic_rounding}
One possible way to counter the negative effects of quantization which has proven itself in the past, is by adding noise to the rounding process. Being a commonplace method in machine learning, we would like to explore the effects of stochastic rounding on the dynamics of the neural network. When using this method the sign activation becomes probabilistic and can be modeled as:
\begin{equation}
\phi(x)=\text{sign}(x+n)
\end{equation}when $n\sim \mathrm{Uniform}[-1,1]$ is randomized for every neuron. Rather than working with a uniformly distributed noise, we replace it with a normal-distributed noise. Therefore, $\phi(u)=\text{sign}(u+n)$, for $n\sim \mathcal{N}(0,a^2) $. We justify this using a numeric simulation presented in figure \ref{fig:stochastic}, and in Appendix \ref{sup:stochastic_rounding} we find that the expression for stochastic rounding mapping (for hidden states) $\M_{sr}(\C)$ is:
\begin{equation}
\M_{sr}(\C)=\frac{2}{\pi}\arcsin\left(\frac{1}{B}\frac{\C\sigma_{w}^2+\sigma_{b}^2}{\sigma_{w}^2+\sigma_{b}^2}\right)
\end{equation}where $B=\sqrt{1+\left(\frac{a}{\qu^{\ast}}\right)^{2}\left(2\qu^{\ast}+a^{2}\right)}$.
While the new mapping function for $\C$ does not reach infinite slope at any point (since $\cu\le1,B>1$), the noise also eliminates $\C=1$ as a fixed point. This result is consistent with the findings of \cite{schoenholz2016deep} who have shown a similar phenomena when using dropout. Due to the $\mathrm{arcsin}$ function being a convex,  monotonically increasing function in the area $0<\cu<1$, We can also conclude that adding noise (and therefore, increasing $B$) can only decrease the fixed point slope. See \ref{pf:stochastic} for proof, and figure \ref{fig:stochastic} for illustration.

\begin{figure}[h!]
  \includegraphics[width=1.0\textwidth]{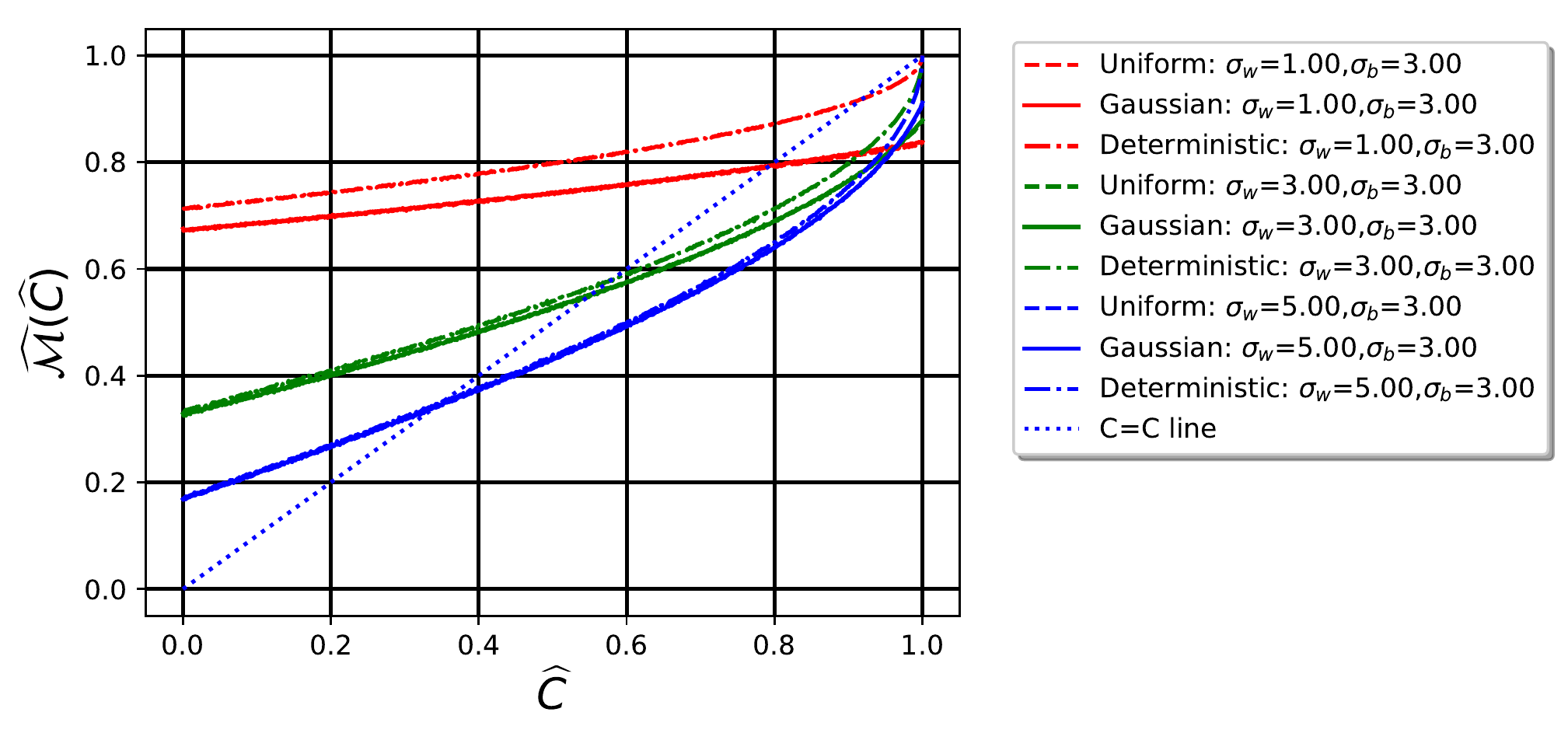}
  \caption{A simulation comparing $\M(\C)$ for deterministic and stochastic sign activations. For the Gaussian noise, we used the distribution $\mathcal{N}(0,\frac{1}{3})$, so both Gauss and Uniform stochastic activations  have the same first and second moments. In all cases, the stochastic activation with the Gauss noise was indistinguishable from the one with the uniform noise.}
\label{fig:stochastic}
\end{figure}

\subsection{Development of the mean field equations for stochastic rounding}\label{sup:stochastic_rounding}

We now want to use the stochastic sign activation function to evaluate how it effects the $\M(C)$.
Using  equation \ref{eq:chi_basic}, and we get:

\[
\chi=4\sigma_{w}^{2}\int_{-\infty}^{\infty}dn_{1}\int_{-\infty}^{\infty}dn_{2}\int_{-\infty}^{\infty}du_{1}\int_{-\infty}^{\infty}du_{2}\frac{1}{2\pi}\frac{1}{2\pi a^{2}}\exp\left(-\frac{n_{1}^{2}}{2a^{2}}\right)\exp\left(-\frac{n_{2}^{2}}{2a^{2}}\right)
\]
\[
\exp\left(-\frac{u_{1}^{2}}{2}\right)\exp\left(-\frac{u_{2}^{2}}{2}\right)\delta\left(\sqrt{\qu^{\ast}}u_{1}+n_{1}\right)\delta\left(\sqrt{\qu^{\ast}}\left(\cu u_{1}+\sqrt{1-(\cu)^{2}}u_2\right)+n_{2}\right)
\]

We use the delta functions to enforce: $u_{1}=-\frac{n_{1}}{\sqrt{\qu^{\ast}}}$, $u_{2}=-\frac{n_{2}-C(n_{1})}{\sqrt{\qu^{\ast}}\sqrt{1-(\cu)^{2}}}$ and get:
\begin{equation}
\chi=\frac{2\sigma_{w}^{2}}{\pi \qu^{\ast}a^{2}\sqrt{1-(\cu)^{2}}}\frac{1}{2\pi}\int_{-\infty}^{\infty}\int_{-\infty}^{\infty}\exp\left(-\frac{n_{1}^{2}}{2a^{2}}\right)\exp\left(-\frac{n_{2}^{2}}{2a^{2}}\right)
\end{equation}
\[
\exp\left(-\frac{(n_{1})^{2}}{2\qu^{\ast}}\right)\exp\left(-\frac{\left(n_{2}-\cu(n_{1})\right)^{2}}{2\qu^{\ast}(1-(\cu)^{2})}\right)dn_{1}dn_{2}
\]
Which can otherwise be written as:

\[
\frac{2\sigma_{w}^{2}}{\pi \qu^{\ast}a^{2}\sqrt{1-(\cu)^{2}}}\frac{1}{2\pi}\int_{-\infty}^{\infty}\int_{-\infty}^{\infty}\exp\left[-\frac{1}{2}D\right]dn_{1}dn_{2}
\]
\[
D=\frac{n_{1}^{2}\left(1-(\cu)^{2}\right)\left(a^{2}+\left(\qu^{\ast}\right)^{2}\right)+n_{2}^{2}\qu^{\ast}\left(1-(\cu)^{2}\right)+a^{2}n_{2}^{2}-2a^{2}n_{1}n_{2}\cu+a^{2}n_{1}^{2}(\cu)^{2}}{\qu^{\ast}\left(1-(\cu)^{2}\right)a^{2}}
\]So:
\begin{equation}
\chi=\frac{2\sigma_{w}^{2}}{\pi \qu^{\ast}a^{2}\sqrt{1-(\cu)^{2}}}\frac{1}{2\pi}\int_{-\infty}^{\infty}\int_{-\infty}^{\infty}\exp\left[-\frac{1}{2}\frac{1}{\qu^{\ast}\left(1-(\cu)^{2}\right)a^{2}}\left(\begin{array}{cc}
n_{1} & n_{2}\end{array}\right)\Sigma^{-1}\left(\begin{array}{c}
n_{1}\\
n_{2}
\end{array}\right)\right]dn_{1}dn_{2}
\end{equation}
\[
\Sigma^{-1}=\left(\begin{array}{cc}
\left(1-{(\cu)^{2}}\right)a^{2}+\left(1-(\cu)^{2}\right)\left(\qu^{\ast}\right)^{2}+{a^{2}(\cu)^{2}} & -a^{2}\cu\\
-a^{2}\cu & \qu^{\ast}\left(1-(\cu)^{2}\right)+a^{2}
\end{array}\right)
\]

Solving the Gaussian we get:

\begin{equation}
\left|\Sigma\right|^{-1}=\left|\Sigma^{-1}\right|=\frac{\left(\qu^{\ast}\left(1-(\cu)^{2}\right)+a^{2}\right)^{2}-\left(a^{2}\cu\right)^{2}}{\left(\qu^{\ast}\left(1-(\cu)^{2}\right)a^{2}\right)^{2}}=
\end{equation}
\[
\frac{(\qu^{\ast})^{2}\left(1-(\cu)^{2}\right)^{2}+2a^{2}\qu^{\ast}\left(1-(\cu)^{2}\right)+a^{4}-a^{4}(\cu)^{2}}{\left(\qu^{\ast}\left(1-(\cu)^{2}\right)a^{2}\right)^{2}}=
\]
\[
\frac{(\qu^{\ast})^{2}\left(1-(\cu)^{2}\right)^{2}+a^{2}\left(2\qu^{\ast}+a^{2}\right)\left(1-(\cu)^{2}\right)}{(\qu^{\ast})^{2}\left(1-(\cu)^{2}\right)^{2}a^{4}}
\]

Resulting:

\begin{equation}
\chi=\frac{2\sigma_{w}^{2}}{\pi \qu^{\ast}a^{2}\sqrt{1-(\cu)^{2}}}\frac{1}{2\pi}\left(2\pi\sqrt{\left|\Sigma\right|}\right)=\frac{2\sigma_{w}^{2}}{\pi{\qu^{\ast}a^{2}}{\sqrt{1-(\cu)^{2}}}}\sqrt{\frac{{(\qu^{\ast})^{2}\left(1-(\cu)^{2}\right)a^{4}}}{(\qu^{\ast})^{2}\left(1-(\cu)^{2}\right)+a^{2}\left(2\qu^{\ast}+a^{2}\right)}}
\end{equation}And we finally get:
\[
\chi=\frac{2\sigma_{w}^{2}}{\pi \qu^{\ast}\sqrt{\left(1-(\cu)^{2}\right)+\left(\frac{a}{\qu^{\ast}}\right)^{2}\left(2\qu^{\ast}+a^{2}\right)}}
\]
For the rest of this section, We will use the shortcut $B\equiv\sqrt{1+\left(\frac{a}{\qu^{\ast}}\right)^{2}\left(2\qu^{\ast}+a^{2}\right)}$
We can now write the equation as:

\begin{equation}\label{eq:stochastic_rounding_chi}
\chi=\frac{2\sigma_{w}^{2}}{\pi \qu^{\ast}\sqrt{\left(1-\left(\frac{\cu}{B}\right)^{2}\right)}}
\end{equation}

for $x=\frac{\cu^{\ast}}{B},\frac{d\cu}{dx}=B$

\[
\M(\C)=\int\frac{d\M(\C)}{d\C}d\C=\int\frac{d\M(C)}{dC}\frac{d\C}{d\cu}d\cu=\int\chi\frac{d\C}{d\cu}\frac{d\cu}{dx}dx
\]
When we again drop the constant so $\M(\C=1)=1$, and get:

\begin{equation}
\M(\C)=\frac{2}{\pi}\arcsin\left(\frac{\cu}{B}\right)
\end{equation}
Based on this equation, we can also use a Taylor expansion, to estimate $\C^{\ast}$, and we get the solution:
\begin{equation}
\C^{\ast}\simeq 1-\frac{4}{\pi^{2}}\frac{\sigma_{w}^{2}}{\qu^{\ast}B}\left(1+\sqrt{1+\left(\frac{\pi}{2}\right)^{4}\left(B^{2}-B\right)\left(\frac{\qu^{\ast}}{\sigma_{w}^{2}}\right)^{2}}\right)
\end{equation}

\section{Calculations of $\qu^{(l)}$ and $\chi$ for general quantized activations}\label{sup:genq_mean_field}
We start by evaluating $\Q$, the hidden-state covariance (see appendix \ref{app:hidden_state}) for the general quantization activation function defined in \ref{eq:genq_def}, using equation \ref{eq:realQC}

\[
\Q^{(l)}=\underset{u\sim\mathcal{N}(0,\qu^{(l)})}{\mathbb{E}}\left(A+\sum_{i=1}^{N-1}H\left(u-g_{i}\right)h_{i}\right)^{2}-\left(\mu^{(l)}\right)^{2},
\]
where:
\begin{equation}\label{eq:genq_mu}
\mu^{(l)}=\underset{u\sim\mathcal{N}(0,\qu^{(l)})}{\mathbb{E}}\left(A+\sum_{i=1}^{N-1}H\left(u-g_i\right)h_i\right)=A+\sum_{i=1}^{N-1}h_{i}\Phi\left(-\frac{g_{i}}{\sqrt{\qu^{(l)}}}\right).
\end{equation}
Here, we use $\Phi$ as the normal cumulative distribution function. The constant $A$ cancels out, and we can expand the multiplication:
\[
\Q^{(l)}=\sum_{i=1}^{N-1}\sum_{j=1}^{N-1}h_{i}h_{j}\left(\mathbb{E}\left[H\left(u-g_{i}\right)H\left(u-g_{j}\right)\right]-\Phi\left(-\frac{g_{i}}{\sqrt{\qu^{(l)}}}\right)\Phi\left(-\frac{g_{j}}{\sqrt{\qu^{(l)}}}\right)\right),
\]
And since $H\left(u-g_{i}\right)H\left(u-g_{j}\right) = H\left(u-g_{\max(i,j)}\right)$
\begin{equation}
\Q^{(l)}=\sum_{i=1}^{N-1}\sum_{j=1}^{N-1}h_{i}h_{j}\left(\Phi\left(-\frac{\max(g_{i},g_{j})}{\sqrt{\qu^{(l)}}}\right)-\Phi\left(-\frac{g_{i}}{\sqrt{\qu^{(l)}}}\right)\Phi\left(-\frac{g_{j}}{\sqrt{\qu^{(l)}}}\right)\right).
\end{equation}
$\Phi\left(-x\right)\Phi\left(-y\right)=\Phi\left(-\max\left(x,y\right)\right)\Phi\left(-\min(x,y)\right)$, so we can see that:
\[
\Phi\left(-\max\left(x,y\right)\right)-\Phi\left(-x\right)\Phi\left(-y\right)=\Phi\left(-\max\left(x,y\right)\right)\left(1-\Phi\left(-\min(x,y)\right)\right)
\]
And by using the CDF property $\Phi(-x) = 1-\Phi(x)$, we get
\begin{equation}\label{eq:genq_q_sup}
\Q^{(l)}=\sum_{i=1}^{N-1}\sum_{j=1}^{N-1}h_{i}h_{j}\Phi\left(-\frac{\max(g_{i},g_{j})}{\sqrt{\qu^{(l)}}}\right)\Phi\left(\frac{\min(g_{i},g_{j})}{\sqrt{\qu^{(l)}}}\right),
\end{equation}
from which we can easily compute $\qu^{(l+1)}$.
In Appendix \ref{sup:genq_mapping_approx}, we develop an approximation for $\mathcal{M}(\C)$. However, for our more immediate concerns, we will go straight to evaluating the equation for the fixed point slope, from eq. \ref{eq:chi_basic}:


\[
\begin{array}[]{c}
\chi=\sigma_{w}^{2}\sum_{i=1}^{N-1}\sum_{j=1}^{N-1}\iint_{u_1,u_2\sim\mathcal{N}(0,\Sigma(\qu^{\ast},\cu^{\ast}))}h_{i}h_{j}  \\ 
\delta\left(\sqrt{\qu^{\ast}}u_{1}-g_{i}\right)\delta\left(\sqrt{\qu^{\ast}}\left(\cu^{\ast}u_{1}+\sqrt{1-\left(\cu^{\ast}\right)^{2}}u_{2}\right)-g_{j}\right)=
\end{array}
\]
\[
\begin{array}[]{c}
\frac{\sigma_{w}^{2}}{2\pi\sqrt{\qu^{\ast}}}\sum_{i=1}^{N-1}\sum_{j=1}^{N-1}\int\exp\left[-\frac{1}{2}\frac{g_{i}^{2}}{\qu^{\ast}}\right]\exp\left[-\frac{1}{2}u_{2}^{2}\right]h_{i}h_{j} \\
\delta\left(\sqrt{\qu^{\ast}}\left(\cu^{\ast}\frac{g_{i}}{\sqrt{\qu^{\ast}}}+\sqrt{1-\left(\cu^{\ast}\right)^{2}}u_{2}\right)-g_{j}\right)=
\end{array}
\]
\[
\frac{\sigma_{w}^{2}}{2\pi \qu^{\ast}\sqrt{1-\left(\cu^{\ast}\right)^{2}}}\sum_{i=1}^{N-1}\sum_{j=1}^{N-1}\exp\left[-\frac{1}{2}\frac{g_{i}^{2}}{\qu^{\ast}}\right]\exp\left[-\frac{1}{2}\frac{\left(g_{j}-\cu^{\ast}g_{i}\right)^{2}}{\qu^{\ast}\left(1-\left(\cu^{\ast}\right)^{2}\right)}\right]h_{i}h_{j}
\]
which can be simplified to:
\begin{equation}\label{eq:genq_chi_sup}
\chi=\frac{\sigma_{w}^{2}}{2\pi \qu^{\ast}\sqrt{1-\left(\cu^{\ast}\right)^{2}}}\sum_{i=1}^{N-1}\sum_{j=1}^{N-1}h_{i}h_{j}\exp\left[-\frac{g_{i}^{2}-2\cu^{\ast}g_{i}g_{j}+g_{j}^{2}}{2\qu^{\ast}\left(1-\left(\cu^{\ast}\right)^{2}\right)}\right].
\end{equation}

\section{The general quantized activations mapping- Approximation and numeric evaluation}\label{sup:genq_mapping_approx}
\subsection{The covariance mapping of a general quantized activation}
We once again use the hidden states covariances $\Q$,$\C$
Using eq. \ref{eq:M} for general quantized activation, we get the expression:
\[
\C^{(l)}\Q^{(l)}=\underset{u_{1},u_{2}\sim\mathcal{N}(0,\Sigma(\qu^{(l)},\cu^{(l)}))}{\mathbb{E}}\left(\sum_{i=1}^{N-1}h_{i}H\left(u_{1}-g_{i}\right)-A\right)\left(\sum_{i=1}^{N-1}h_{i}H\left(u_{2}-g_{j}\right)-A\right)-\left(\mu^{(l)}\right)^{2},
\]
where we can use eq. \ref{eq:genq_mu} and expand it to:
\[
\C^{(l)}\Q^{(l)}=\sum_{i=1}^{N-1}\sum_{j=1}^{N-1}\left(\underset{u_{1},u_{2}\sim\mathcal{N}(0,\Sigma(Q^{l},C^{l}))}{\mathbb{E}}\left[H\left(u_{1}-g_{i}\right)H\left(u_{2}-g_{j}\right)\right]-\Phi\left(\frac{-g_{i}}{\sqrt{\qu^{(l)}}}\right)\Phi\left(\frac{-g_{j}}{\sqrt{\qu^{(l)}}}\right)\right).
\]
When the offsets are different than zero, there is no exact solution for the expectancy when $u_1,u_2$ are correlated. Article \cite{wu2018deterministic} suggests an approximation for finding $\M(\C)$, when $C(\C)=\frac{Q^{\ast}C\sigma_w^2+\sigma_b^2}{Q^{\ast}\sigma_w^2+\sigma_b^2}$:
\begin{equation}\label{eq:approx_genq_M}
\begin{array}{c}
\M(\C)\simeq\frac{\text{arcsin}(\cu^{*})}{2\pi Q^{*}}\sum_{i=1}^{N-1}\sum_{j=1}^{N-1}h_{i}h_{j} \cdot \\
\exp\left(-\frac{1}{2}\frac{\cu}{\text{arcsin}(\cu)\qu^{\ast}\sqrt{1-\cu^{2}}}\left(g_{i}^{2}+g_{j}^{2}-g_{i}g_{j}\frac{2\cu}{1+\sqrt{1-\cu^{2}}}\right)\right)
\end{array}
\end{equation}
We found the approximation to hold well in the area $\cu \sim 0$ , and $\forall i,g_i < \qu^{\ast}$. Therefore, when $\qu^{\ast}$ is known, this equation can be used to evaluate $\cu^{\ast}$  with reduced complexity.

\subsection{Quick numeric method to approximate the fixed point slope, for  $\sigma_b>0$}\label{alg:approx_chi}
Using eq. \ref{eq:approx_genq_M}, we suggest a numeric algorithm to evaluate the fixed point slope for $\sigma_b>0$, for any quantized activation function:

\begin{enumerate}
\item Evaluate  $\qu$ by iterative usage of eq. \ref{eq:genq_q}. Start with arbitrary value $\Q= 1.0$ and repeat $T$ times. 
\item Use eq. \ref{eq:approx_genq_M} to evaluate $\M(\C=0)$ (Reminder: $\cu(\C=0)=\frac{\sigma_{b}^2}{\Q \sigma_{w}^2+\sigma_{b}^2}$)
\item Use eq. \ref{eq:genq_chi} to evaluate $\chi(\C=0)$ 
\item Estimate $\C^{\ast}$ by $\frac{\cu(\C=0)}{1-\chi(\C=0)}$(First order approximation), and use equation  \ref{eq:genq_chi} to find the fixed point slope.
\end{enumerate}

We found this algorithm to be very efficient and accurate when studying the dynamics in the area of  $\sigma_b>0$. Results of using this estimation are displayed in figure \ref{fig:init_grid}.

\begin{figure}[t!]
    \centering
    \begin{subfigure}[]{1.0\textwidth}
        \centering
        \includegraphics[height=2.0in]{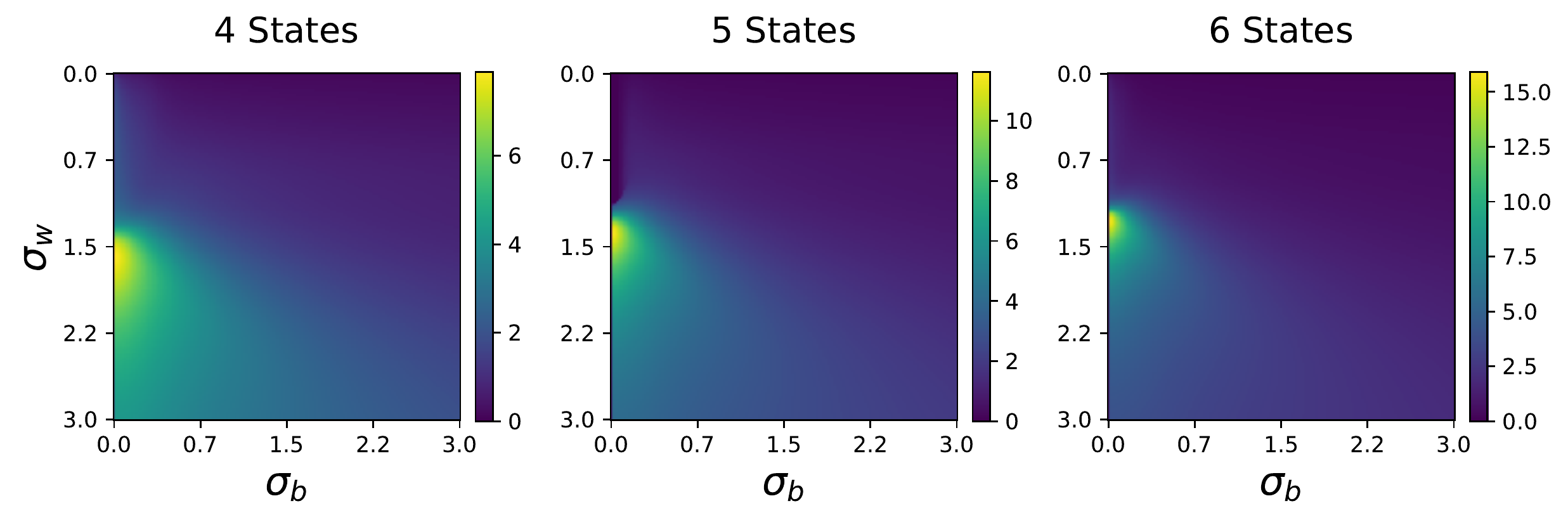}
    \end{subfigure}%
    \\
    \begin{subfigure}[]{1.0\textwidth}
        \centering
        \includegraphics[height=2.0in]{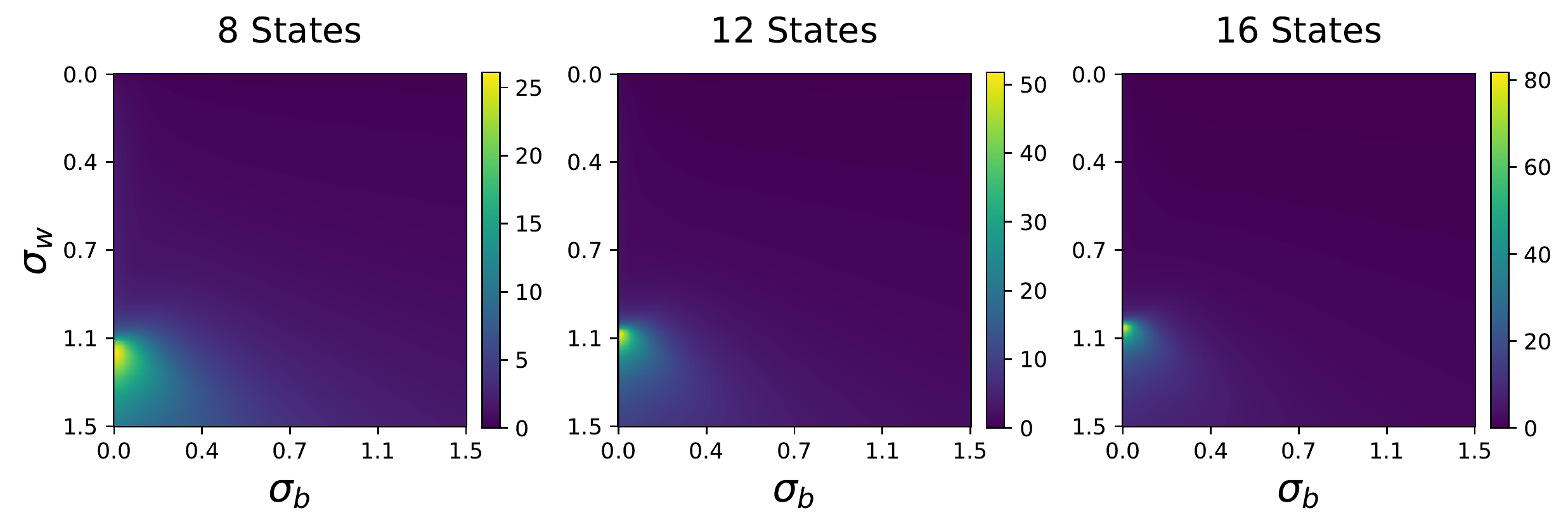}
    \end{subfigure}%
    \caption{Grid-Approximation of the depth scale $\xi$  for constant-spaced activations of different quantization levels, as a function of the initialization parameters. $D=1$ was used as the constant space between offsets. For this approximation, we used the algorithm described in \ref{alg:approx_chi}. It is apparent that the maximal depth scale for all quantization levels is achievable for $\sigma_b \simeq 0$.}
    \label{fig:init_grid}
\end{figure}


\section{Beyond constant-spaced quantized activations}\label{sup:beyond_linspace}
Our main focus in this article, have been the quantized activations with constant spacing. We now want to study the effects of using more complex activation functions on the dynamics of the network. We will do so by defining a new family of quantized activation functions, the linear-spacing activations- For any given values of $h,c_1>0,c_2\in\mathbb{R}$, the  function parameters in accordance with equation \ref{eq:genq_def}, are: 
\begin{equation}\label{eq:def_square_spacing}
\begin{array}[]{cc}
\forall i\in\{1,..,N-1\}, m\equiv \left(k-\frac{N}{2}-2\right) ,h_{i}=h,
\\ \tilde{g}_i=\tilde{D}_{0}m\left(1+\tilde{D}_{1}|m|\right)
\end{array}
\end{equation}
This family of functions can be thought of a second order generalization of the constant-spaced functions, which correspond to the special case $\tilde{D}_1=0$. This family of functions is important, as it also includes sigmoid-like quantized activation functions (given for values of $\tilde{D}_1>0$). To evaluate the dynamics of the new family, we again use eq. \ref{eq:genq_chi_b0} and the depth scale definition eq. \ref{eq:xi}, and run a grid search over the normalized values of $\tilde{D}_1$, $\tilde{D}_0 $, calculating the depth scale for each combination of parameters. The results of the grid search for several different quantization levels are presented in Figure \ref{fig:square_spacing}. In all of the tested activations, the maximal depthscale that we found was identical, within numeric error range, to the maximal depthscale found for constant-spaced activations, indicating that the additional degree of freedom does not help improving the dynamical properties of the activation.

\begin{figure}[t!]
    \centering
    \begin{subfigure}[]{1.0\textwidth}
        \centering
        \includegraphics[height=2.0in]{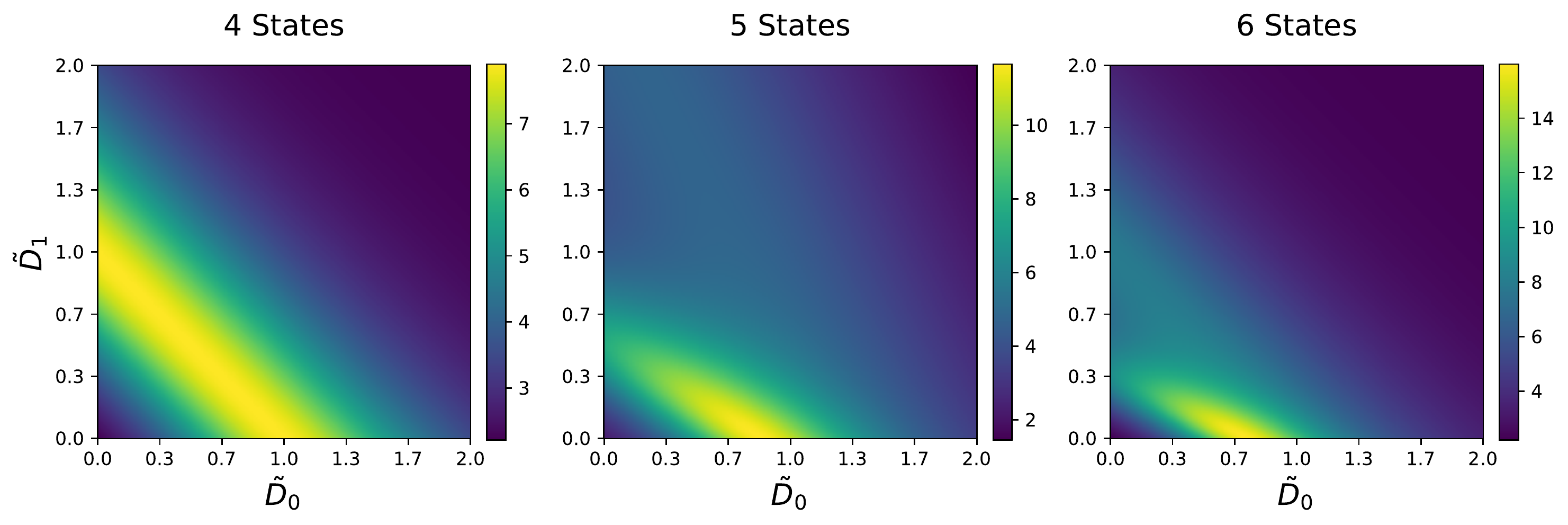}
    \end{subfigure}%
    \\
    \begin{subfigure}[]{1.0\textwidth}
        \centering
        \includegraphics[height=2.0in]{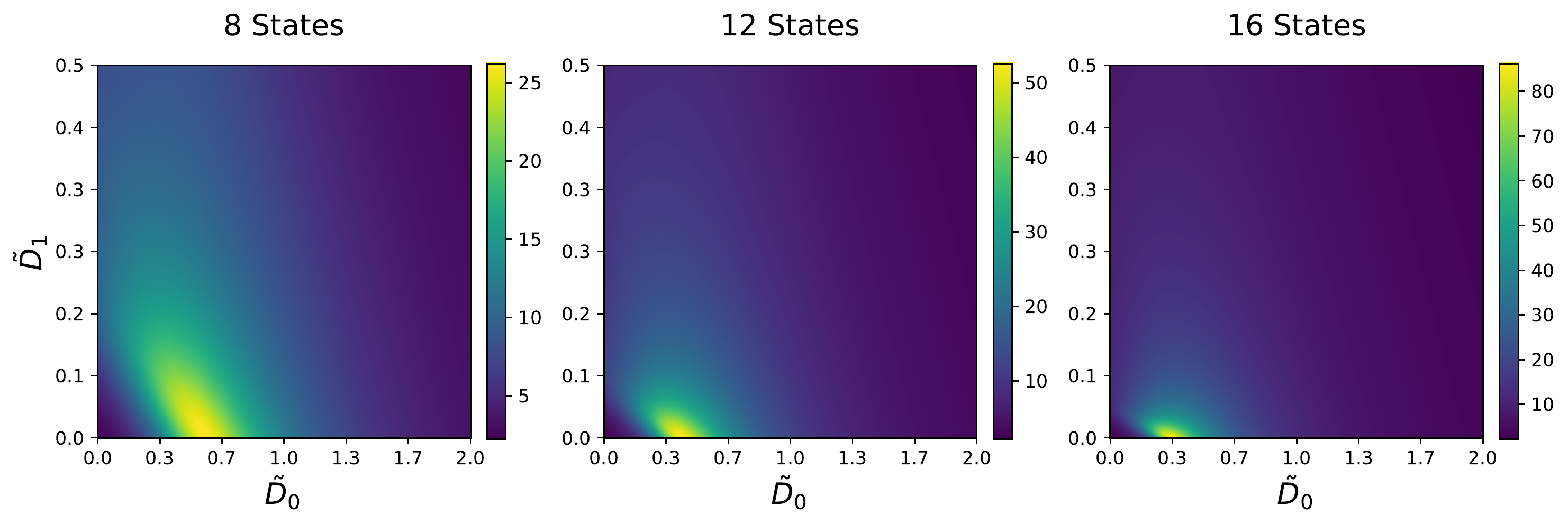}
    \end{subfigure}%
    \caption{Evaluation of the depth scale $\xi$ for linear-spaced quantized activation, with the initialization $\sigma_b=0$. The search resolution is $1000 \times 1000$ for each quantization level. The maximum depth-scale on each grid for square spacing activations is always achievable for the constant spacing as well, where $\tilde{D}_1=0$.}
    \label{fig:square_spacing}
\end{figure}

\section{Additional MNIST training-results}\label{sup:MNIST}

When studying the empirical effects of the initialization parameters on trainability when using a 10 states quantization, and seen that the longest trainable network is achieved when using the $\tilde{D}_{opt}$, the optimal normalized distance between offsets, as proposed by our theory. Additional test have been made to other quantization levels as well and gave similar results. It is unclear from the results, however, whether the degradation of deep networks is caused by the unoptimized propagation of the forward pass, or by the unoptimized backward pass. To isolate the effects of the forward pass which are of more interest to us, we measured the effects of $\sigma_w$  on a 10 states quantization once more, but optimized the STE to allow clean gradient propagation using  $\rho^{-1}=\sigma_{w}\sqrt{\text{erf}\left(\frac{1}{\sqrt{2\qu^{\ast}}}\right)}$, when using $\sigma_w$  and $\qu^{\ast}$ based on each run's initialization values. Figure \ref{fig:adapative_slope}  shows the results of this experiment, and confirms that the optimal initialization is dominated by the forward pass. 

\begin{figure}[t!]
    \centering
    \begin{subfigure}[]{1.0\textwidth}
        \centering
        \includegraphics[height=2.0in]{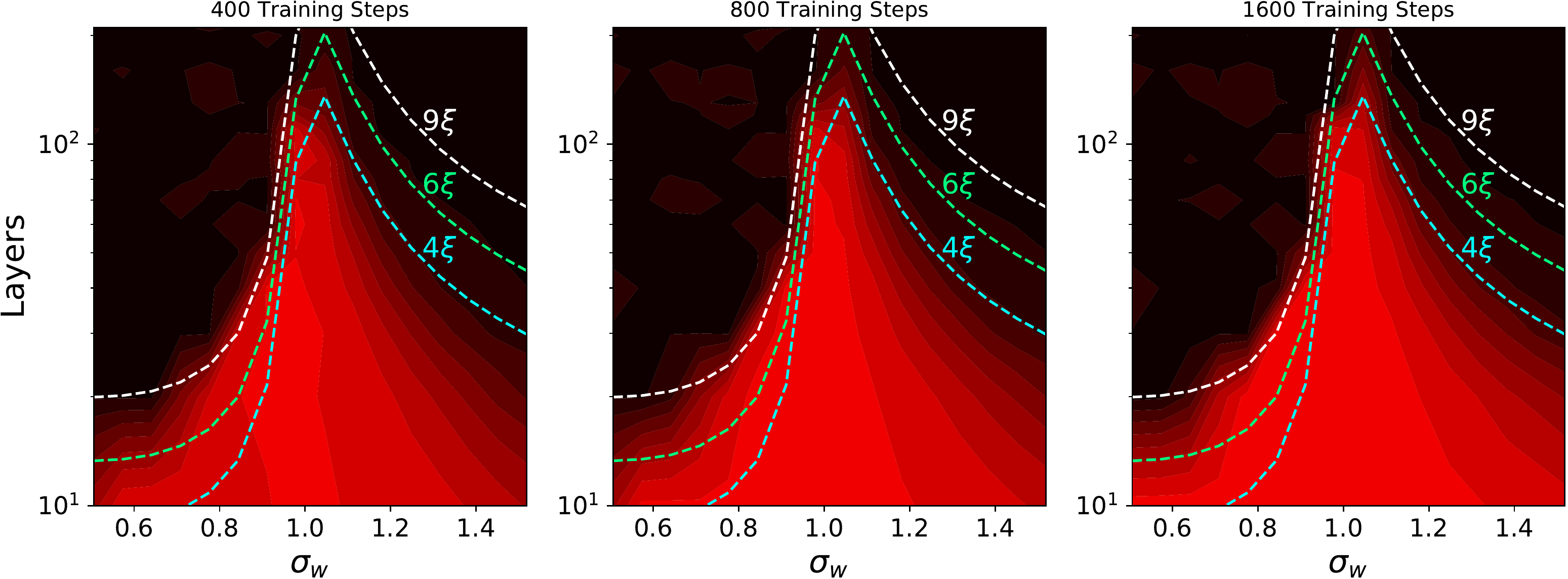}
    \end{subfigure}%
    \caption{Test accuracy of a 10-states activation in feed-forward network, over the MNIST data-set, with different initialization values and optimized STE for backward propagation of gradients. When compared with the \ref{fig:MNIST_experiment}, we can see that adjusting the networks for better backward propagation of the gradients does not have a significant effect on the trainability of deep networks.}
    \label{fig:adapative_slope}
\end{figure}

\begin{figure}[t!]
    \centering
    \begin{subfigure}[]{1.0\textwidth}
        \centering
        \includegraphics[height=2.0in]{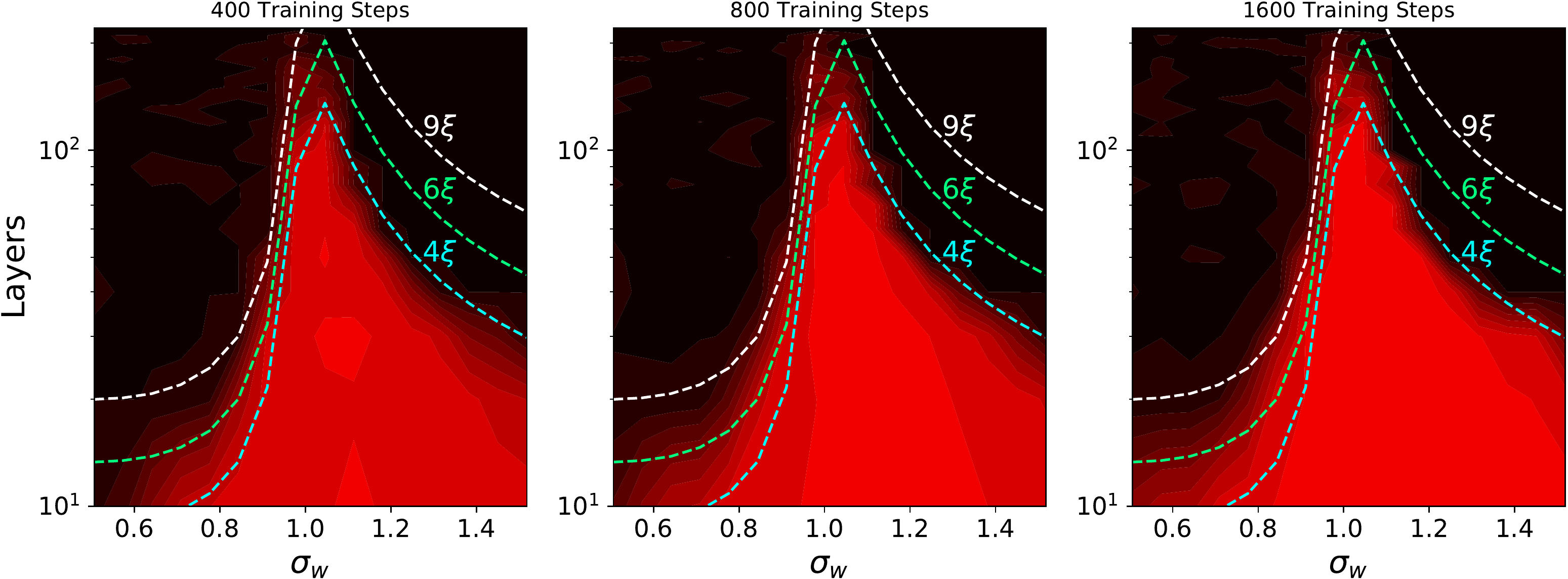}
    \end{subfigure}%
    \\
    \begin{subfigure}[]{1.0\textwidth}
        \centering
        \includegraphics[height=2.0in]{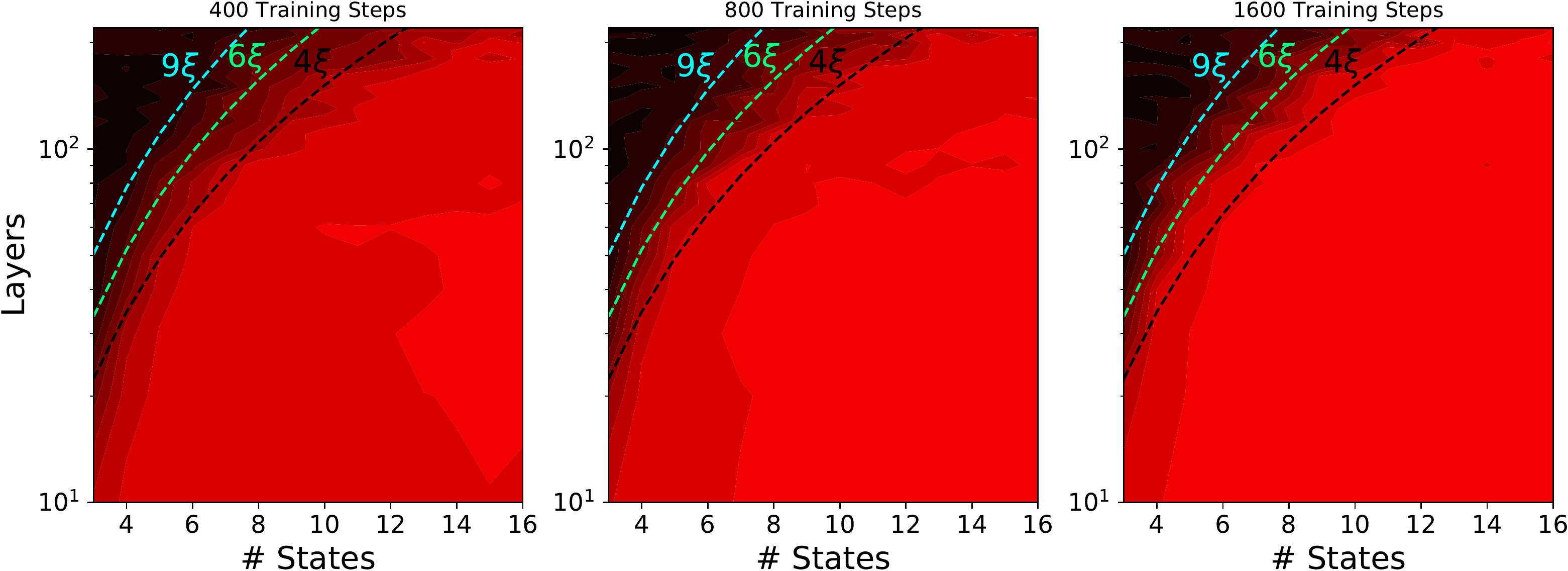}
    \end{subfigure}%
    \\
    
    \begin{subfigure}{.45\textwidth}
        \centering
        \includegraphics[height=2.0in]{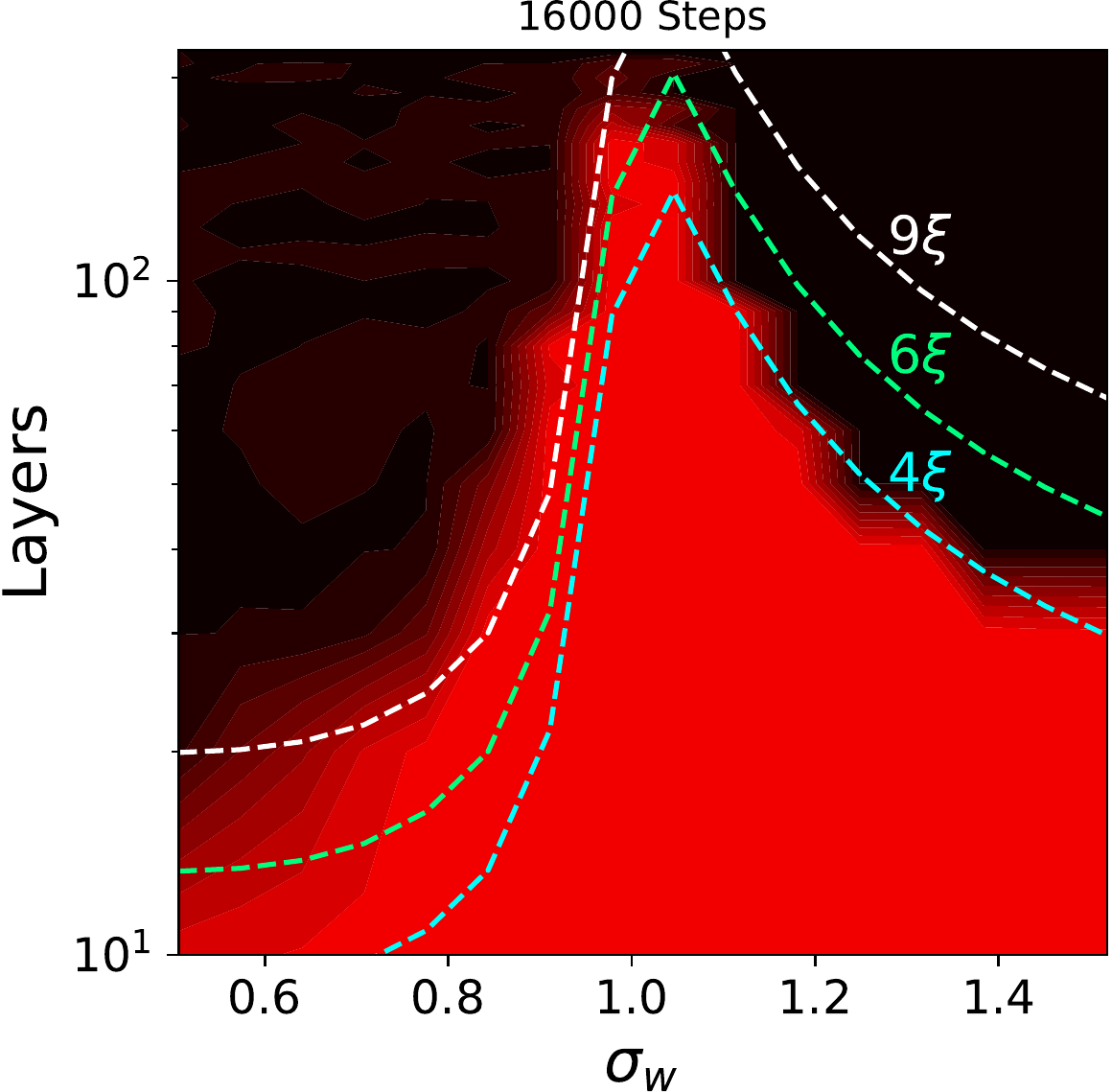}
    \end{subfigure}%
    \begin{subfigure}{.45\textwidth}
        \centering
        \includegraphics[height=2.0in]{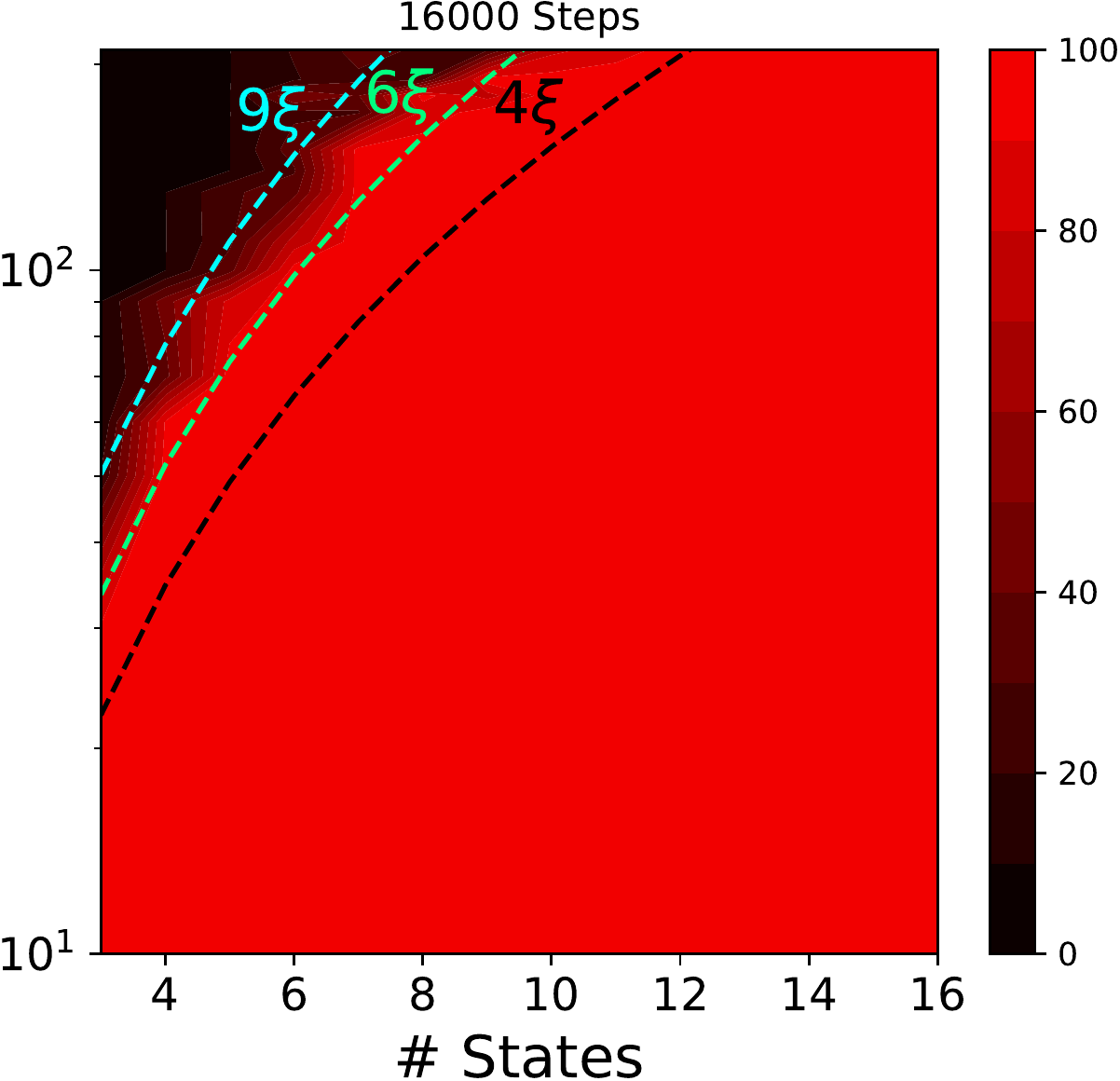}
    \end{subfigure}%

    \caption{Time evolution of the test accuracy. Line 1\&2: The evolution of the heat maps presented in figure \ref{fig:MNIST_experiment}, at an early stage of training (Training accuracy) . Bottom line: Test accuracy at an advanced stage of training (16000 steps), for the same deployment. Those results align with the results of \cite{schoenholz2016deep}, showing that even in a late stage of training, networks with layers exceeding $\sim 6\xi$ are untrainable. } 
    \label{fig:MNIST_time_evolution}
\end{figure}

\section{Simplified Optimization of the initialization parameters}\label{sup:practition}
Sections \ref{main:genq} describes an algorithm for computation of the value of the initialization parameter $\sigma_w$, that would allow the best signal propagation in the network for any quantized activation function. However, when dealing with the constant spaced activation functions of the form:
\[
\phi_{N}(x)=-1+\sum_{i=1}^{N-1}\frac{2}{N-1}H\left(x-\frac{2}{N-1}\left( i - \frac{N}{2} \right)\right),
\]
we find that our suggested method of initialization quickly converges to the \textit{Xavier initialization} \cite{glorot2010understanding}, as the quantization levels increases. For simple initialization, we suggest a small modification for the Xavier method that accounts for quantization: When $F_{in}$ and $F_{out}$ are the fan-in and fan-out of the layer, rather than simply computing the standard error for weights initialization using $\sigma_w=\sqrt{\frac{2}{F_{in}+F_{out}}}$ as in the case of normal Xavier, we suggest that using a factor of \[
\alpha_{N}=1+\frac{1.23}{\left(N+0.2\right)^2}
\] (when $N$ is the number of activation states), so that:
\[
\sigma_w=\alpha_{N} \cdot \sqrt{\frac{2}{F_{in}+F_{out}}}
\]
 We see that for the continuous case, our activation function becomes hard-tangent and our factor becomes $\lim_{N\to\infty}\alpha_{N}=1$. $\alpha_{N}$ was estimated by computing the value $\sigma_w$ that ensures $\frac{D}{\sqrt{\qu}}=\tilde{D}_{\mathrm{opt}}$ for states ranging from 1 to 128, and fitting the results $\sigma_w(N)$ to the function $1+\frac{a}{(N-b)^2}$, which behaved accordingly. For the case where the number of states is larger than 128, the factor $\alpha_N$ is small enough for the error to be irrelevant. Figure \ref{fig:Xavier} shows a comparison between the standard Xavier and our modified initialization for 3-states activation, where $\alpha_N$ is at it's peak.

\begin{figure}[h!]
    \parbox[t]{1.0\textwidth}{\null
        \begin{subfigure}{.45\textwidth}
            \centering
        	\includegraphics[width=\linewidth]{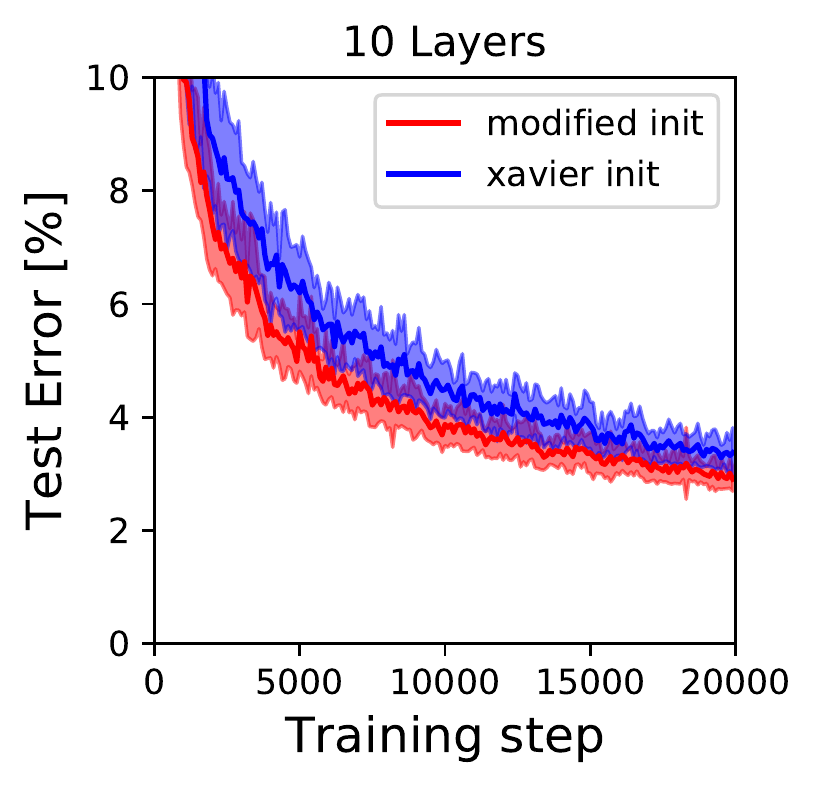}
        \end{subfigure}%
        \begin{subfigure}{.45\textwidth}
            \centering
        	\includegraphics[width=\linewidth]{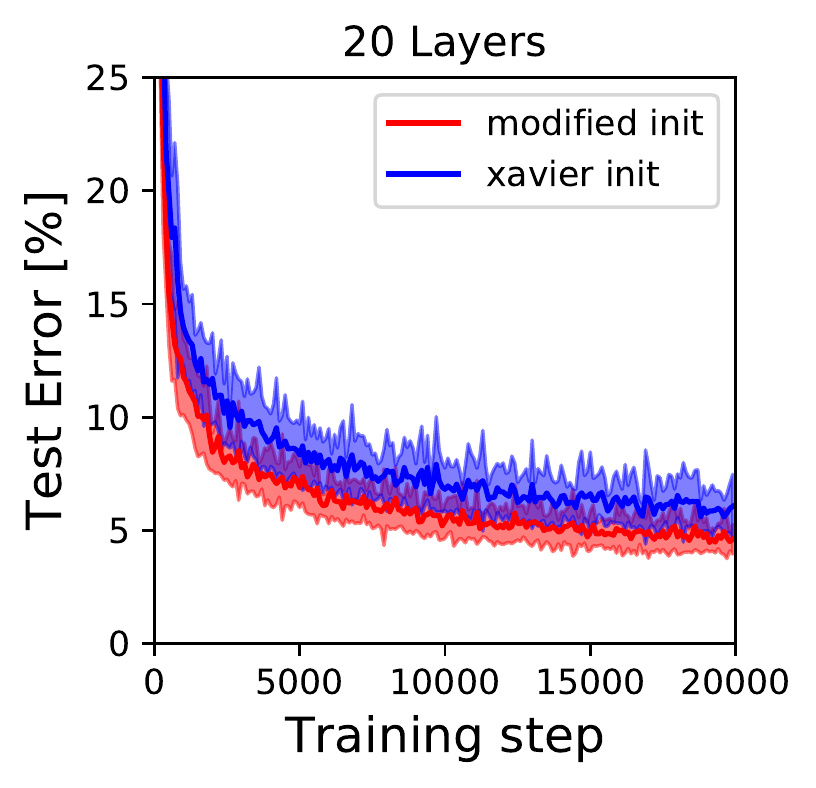}
        \end{subfigure}%
        \\
        \begin{subfigure}{.45\textwidth}
            \centering
        	\includegraphics[width=\linewidth]{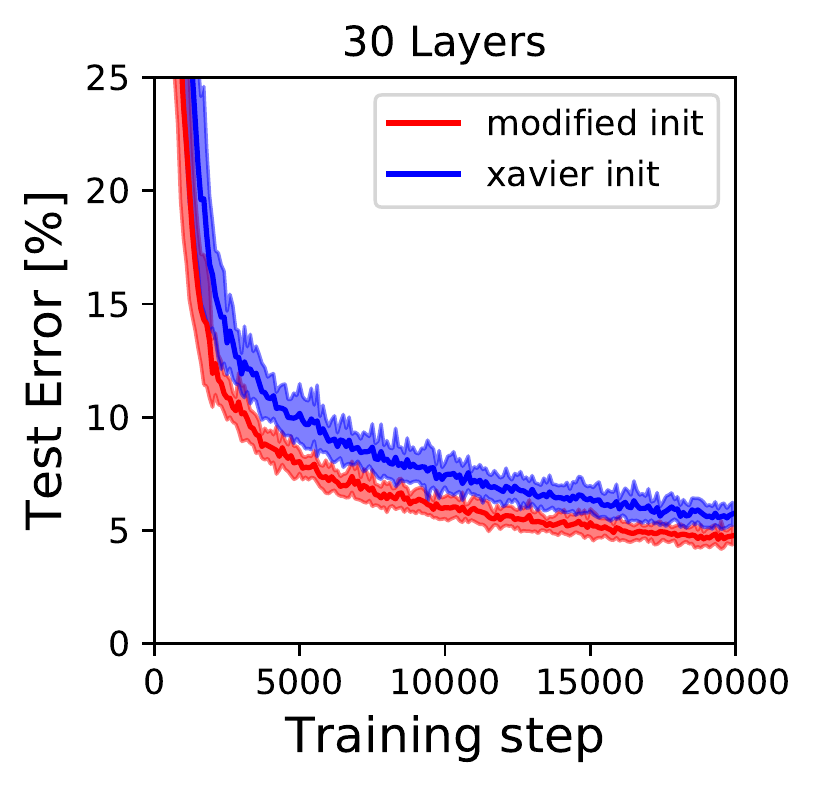}
        \end{subfigure}%
        \begin{subfigure}{.45\textwidth}
            \centering
        	\includegraphics[width=\linewidth]{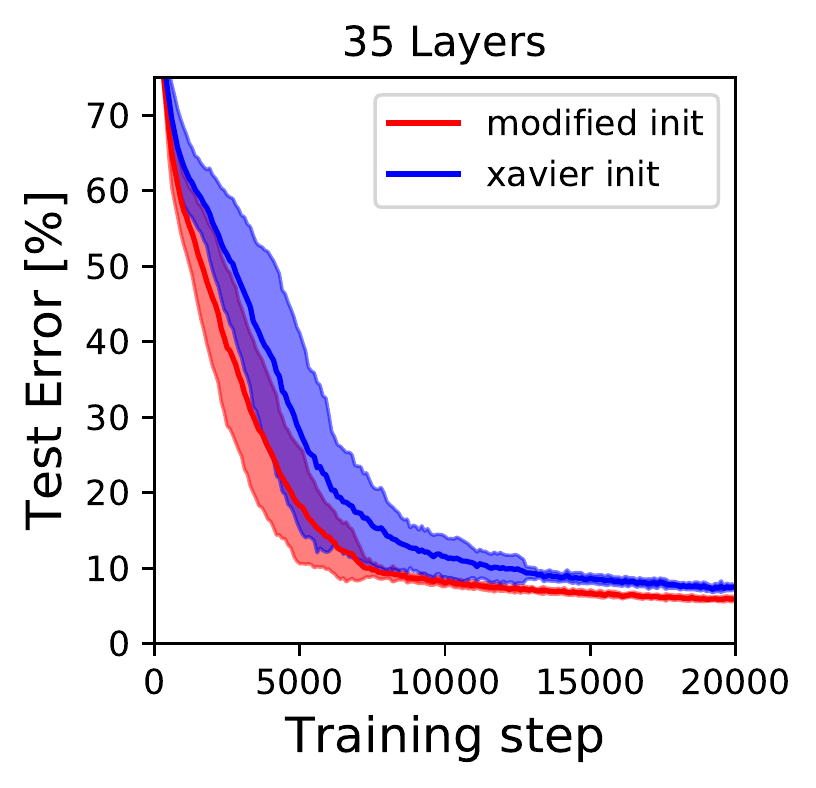}
        \end{subfigure}%
        \\
        \begin{subfigure}{.45\textwidth}
            \centering
        	\includegraphics[width=\linewidth]{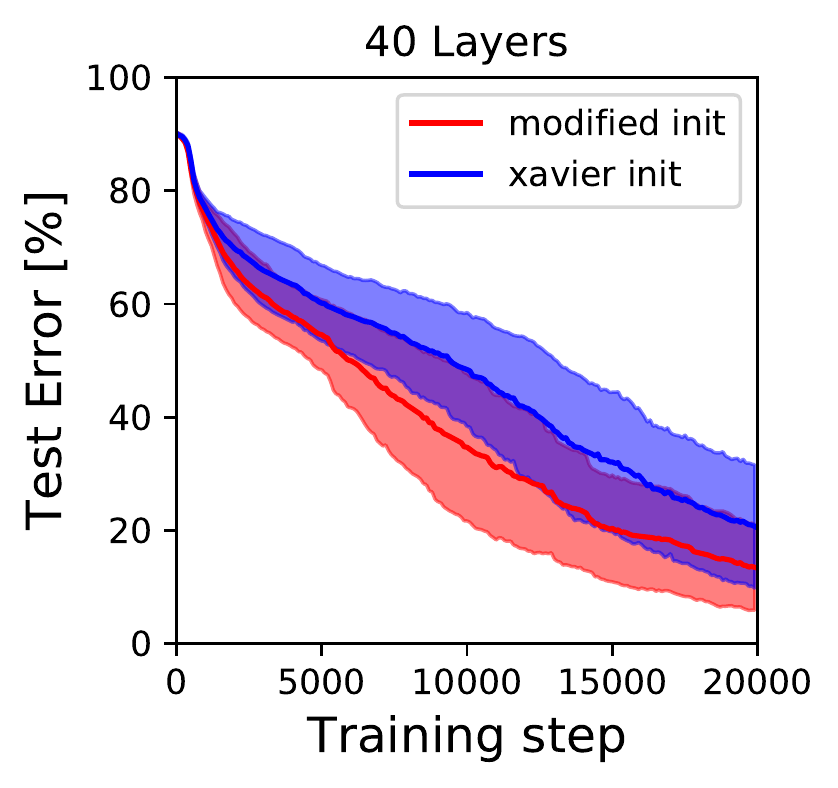}
        \end{subfigure}%
        ~
        \hspace{0.1cm}
        \parbox[t]{0.45\textwidth}{\null
        \centering
        \vspace{-2.7cm}
        \captionof{table}[t]{\textbf{Results Summary}}\label{table:sparse comparison}
        \begin{tabular}{|c|c|c|}
            \hline
            \\[-1em]
             & \multicolumn{2}{c|}{Test Error (Mean)}  \\
        	\hline
            \\[-1em]
            \textbf{Layers} & \textbf{Xavier} & \textbf{Modified}  \\
            \hline
            \\[-1em]
            10 & $3.3\pm0.4\%$ & $2.9\pm0.2\%$  \\
            \hline
            \\[-1em]
            20 & $5.6\pm1.4\%$ & $4.5\pm0.7\%$ \\
            \hline
            \\[-1em]
            30 & $5.6\pm0.5\%$ & $4.6\pm0.4\%$  \\
            \hline
            \\[-1em]
            35 & $7.2\pm0.4\%$ & $5.9\pm0.4\%$ \\
            \hline
            \\[-1em]
            40 & $21\pm11\%$ & $13.5\pm7.5\%$  \\
            \hline
        \end{tabular}
        }
        
        \captionof{figure}{Comparison of our suggested initialization with the Xavier Gaussian initialization, for MNIST training using a 3-states quantized activation for layer numbers near the depth scale $6\xi_{max}\simeq\mathbf{37}$. For each number of layers and initialization, we used a grid search to find best learning rate from the values $[0.25,0.5,1,2,4,10] \times 10^{-3}$, with all other run parameters as described in the experimental part of section \ref{sec:experiment}. We ran 25 seeds using that learning rate, and the plot describes the mean and standard error of the test accuracy, at every step. In all cases, our suggested modification outperforms Xavier initialization by a small margin. With 40 layers, the network depth exceeds the theoretical depth scale, and all trainings fail under the 20000 steps limitation.}
        \label{fig:Xavier}
    }
    \vspace{-0.3cm}
\end{figure}

\section{Backwards signal propagation for straight through estimator}\label{sup:backward}

While we use quantized activations for the forward pass, the backward propagation of quantized neural networks is, in our case, done by straight through estimators (STE). When using constant-spaced quantized activations, we choose a STE to imitate the backward pass of the hard-tanh function:

\[
\phi_{\rho}(x)=\begin{cases}
-\rho^{-1} & x<-1\\
x\rho & -1\leq x\leq1\\
\rho^{-1} & x>1
\end{cases}
\]
where $\rho>0$ is a parameter that controls the slope of the hard-tanh, so the backward equation is determined by the derivative:
\begin{equation}
\phi_{\rho}^{\prime}(x)=\begin{cases}
\rho & \left|x\right|<1\\
0 & \text{else}
\end{cases}
\end{equation}
The moments of a random $N\times N$ matrix \textbf{$\overline{A}$} are given by $m_{\mathbf{\overline{A}}}^{(i)}=\frac{1}{N}\mathbb{E}\text{tr}\left(\mathbf{\overline{A}}^{i}\right)$.
In the case of eq. \ref{eq:Mjj}, and our STE $\phi_{\rho}$, the equation is reduced to

\[
m_{\mathbf{J}\mathbf{J}^{T}}^{(1)}=\frac{1}{N}\mathbb{E}\text{tr}\left(\mathbf{\phi_{\rho}^{\prime}(\mathbf{u}^{\ast})}\mathbf{W}\left(\mathbf{\phi_{\rho}^{\prime}(\mathbf{u}^{\ast})}\mathbf{W}\right)^{}\right)
\]

where $u_{i}^{\ast}\sim\mathcal{N}(0,\qu^{\ast})$ .i.d and
$\mathbf{D}_{\phi'(\mathbf{u}^{\ast})}$ is a diagonal matrix with
$\phi'(\mathbf{u}^{\ast})$ on the diagonal. This gives

\[
m_{\mathbf{J}\mathbf{J}^{T}}^{(1)}=\sigma_{w}^{2}\int\left(\phi_{\rho}^{\prime}(\sqrt{\qu^{\ast}}z)\right)^{2}\mathcal{D}z
\]
where $\mathcal{D}z=\frac{1}{\sqrt{2\pi}}\exp\left({\frac{-z^2}{2}}\right)$.
Then obtain:
\begin{equation}
m_{\mathbf{J}\mathbf{J}^{T}}^{(1)}=\sigma_{w}^{2}\rho^{2}\underset{-1/\sqrt{\qu^{\ast}}}{\overset{1/\sqrt{\qu^{\ast}}}{\int}}\mathcal{D}u=\sigma_{w}^{2}\rho^{2}\text{erf}\left(\frac{1}{\sqrt{2\qu^{\ast}}}\right).\label{eq:m1_def}
\end{equation}

Assuming we already have the value $\sigma_w$, $\qu^{\ast}$, we can set $\rho^{-1}=\sigma_{w}\sqrt{\text{erf}\left(\frac{1}{\sqrt{2\qu^{\ast}}}\right)}$ to ensure $m_{\mathbf{J}\mathbf{J}^{T}}=1$, and thus avoid vanishing and exploding gradients. In our main results, we avoided modifying the STE parameter $\rho$ in order to keep the experiment simple, and used the trivial STE using $\rho=1$.

\section{Comparing convergence in $C$ and $Q$ directions}\label{sup:CQcomparison}

In previous papers studying signal propagation in feed-forward networks \cite{poole2016exponential, schoenholz2016deep, xiao2018dynamical}, it has been argued that the convergence in $Q$ direction is significantly faster than the convergence in the $C$ direction. Under this assumption, one can derive the approximate depth-scale by analyzing convergence in the $C$ direction only. The claim was established using empirical evidence \cite{poole2016exponential} and using an approximated Taylor expansion of the activation function \cite{schoenholz2016deep}, by showing that the slope $\chi_c$ at $C^{\ast}=1$ is always larger than the slope $\chi_q$ at $Q^{\ast}$. In our case, however, it is invalid to assume that the Taylor expansion of the quantized activation is correctly approximating the function behaviour, and either way $C^{\ast}=1$ is an infinitely unstable fixed point and the convergence there can not be used as a baseline for comparison with the convergence in the $Q$ direction. It is therefore necessary to assert that this assumption holds for quantized activations as well. We will start by comparing $\chi_c$,$\chi_q$ analytically for general quantized activation function in the limit where the $\sigma_w$ is very small or very large, show that our assumption may fail in the case of some nontrivial activation functions and provide empirical evidence that the condition $\chi_c(C=C^{\ast})>\chi_{q}(Q=Q^{\ast})$ holds for trivial activation functions. 

First, we argue that it is sufficient to show that $\chi_c(C=C^{\ast})>\chi_q(Q=Q^{\ast})$ for the depthscale in the $C$ direction $(\xi_C)$ to be indicative of the full system-convergence. This is true because the mapping function of $Q$ is independent of the value of $C$. In the case of where $\chi_c(C=C^{\ast})=\chi_q(Q=Q^{\ast})$, we can, at the worst case, consider that $C$ will only start converging once $Q$ has converged, in which case the system would converge after a $K_c\xi_C+K_q\xi_q$ where $K_c,K_q$ are some constants.

Going back to eq. \ref{eq:genq_chi}, using $\tilde{g_{i}}=\frac{g_{i}}{\sqrt{Q}}$, and picking the minimal value of $C=0$ ($\mathcal{M}(C)$ is convex) :
\begin{equation}\label{eq:sup_c_convergence}
\begin{array}{c}
\chi_c>\frac{\sigma_{w}^{2}}{Q}\sum_{i=1}^{N-1}\sum_{j=1}^{N-1}h_{i}h_{j}\frac{1}{2\pi}\exp\left[-\frac{\tilde{g_{i}}^{2}+\tilde{g_{j}}^{2}}{2}\right]=  \\

\frac{\sigma_{w}^{2}}{Q}\sum_{i=1}^{N-1}\sum_{j=1}^{N-1}h_{i}h_{j}\phi\left(\tilde{g_{i}}\right)\phi\left(\tilde{g_{j}}\right)  
\end{array}
\end{equation}

We do a similar derivation for the mapping of $Q$. From eq. \ref{eq:genq_q_sup}, using $\frac{d\Phi(\frac{a}{\sqrt{x}})}{dx}=\frac{-a}{2x^{3/2}}\phi\left(\frac{a}{\sqrt{x}}\right)$, and denoting $G_{i,j}^{+}=\max(\tilde{g_{i}},\tilde{g_{j}}),G_{i,j}^{-}=min(\tilde{g_{i}},\tilde{g_{j}})$ we get that:
\begin{equation}\label{eq:sup_q_convergence}
\chi_q=\frac{d\mathcal{M}(Q)}{dQ}=\frac{\sigma_{w}^{2}}{2Q}\sum_{i=1}^{N-1}\sum_{j=1}^{N-1}h_{i}h_{j}\left[G_{i,j}^{+}\phi\left(G_{i,j}^{+}\right)\Phi\left(G_{i,j}^{-}\right)-G_{i,j}^{-}\phi\left(G_{i,j}^{-}\right)\Phi\left(-G_{i,j}^{+}\right)\right]
\end{equation}

Combining those results, we get that:
\begin{equation}
\frac{\chi_{q}}{\chi_{c}}\leq\frac{1}{2}\frac{\sum_{i=1}^{N-1}\sum_{j=1}^{N-1}h_{i}h_{j}\left[G_{i,j}^{+}\phi\left(G_{i,j}^{+}\right)\Phi\left(G_{i,j}^{-}\right)-G_{i,j}^{-}\phi\left(G_{i,j}^{-}\right)\Phi\left(-G_{i,j}^{+}\right)\right]}{\sum_{i=1}^{N-1}\sum_{j=1}^{N-1}h_{i}h_{j}\phi\left(G_{i,j}^{+}\right)\phi\left(G_{i,j}^{-}\right)}
\end{equation}

From this result, we can immediately see that when taking $\sigma_w\to\infty$, resulting, $G_{i,j}^{-/+}\to0$, we get that $\frac{\chi_{q}}{\chi_{c}}\to0$, so $\chi_q \ll \chi_c$.

To analyze the behaviour of $\sigma_w\to0$, we will consider the \textbf{continuous} activation functions:
\begin{equation}\label{eq:sup_convergence_alpha}
\phi\left(x\right)=\begin{cases}
x\left(1-\alpha\left|x\right|\right) & \left|x\right|<A\\
1 & \text{else}
\end{cases}
\end{equation}
where for $\alpha=0$ we get an hard-tanh and for $\alpha>0$ we get a sigmoid like function. The derivative of this function is:
\begin{equation}
\phi^{\prime}\left(x\right)=\begin{cases}
1-2\alpha\left|x\right| & \left|x\right|<A\\
0 & \text{else}
\end{cases}.
\end{equation}
We also calculate the derivative  $\frac{\partial\mathcal{M}(Q)}{\partial Q}$ directly from eq. \ref{eq:QCsys} and get:

\begin{equation}
\chi_q(Q)=\frac{1}{Q}\sigma_{w}^{2}\mathbb{E}\left[\phi'\left(x\right)\phi\left(x\right)x\right]
\end{equation}

where $X~\mathbb{N}(0,Q)$. We will also use the previous result $\chi_c(C)>\chi_c(0)=\sigma_{w}^{2}\mathbb{E}\left[\phi'\left(x\right)^2\right]$. If we look at values where $\sigma_{w}$ is small, resulting small enough $Q$ so values outside the region $x<|A|$ can be ignored, and we get:
\begin{equation}
\chi_q=\frac{\sigma_{w}^{2}}{Q}\mathbb{E}\left[\phi^{\prime}\left(x\right)\phi\left(x\right)x\right] =\frac{\sigma_{w}^{2}}{Q}\mathbb{E}\left[\left(1-3\alpha\left|x\right|+2\alpha^{2}x^{2}\right)x^{2}\right]
\end{equation}
which sums up to:

\begin{equation}
\chi_q=\sigma_{w}^{2}\left(1-6\alpha\sqrt{\frac{2Q}{\pi}}+2\alpha^{2}Q\left(3!!\right)\right)=\sigma_{w}^{2}\left(1-6\alpha\sqrt{Q}\sqrt{\frac{2}{\pi}}+6\alpha^{2}Q\right)
\end{equation}

Similarly, 

\begin{equation}
\chi_c(C=0)=\sigma_{w}^{2}\mathbb{E}\left[\left(1-2\alpha E\left|x\right|\right)^{2}\right]=\sigma_{w}^{2}\left(1+4\alpha^{2}Q-4\alpha \sqrt{Q}\sqrt{\frac{2}{\pi}}\right).
\end{equation}

The condition $\chi_c\ge\chi_q$ therefore translates to:
\begin{equation}
1+4\alpha^{2}Q-4\alpha \sqrt{Q}\sqrt{\frac{2}{\pi}}>1-6\alpha\sqrt{Q}\sqrt{\frac{2}{\pi}}+6\alpha^{2}Q
\end{equation}

or simply $\alpha\sqrt{\frac{2}{\pi}}>\alpha^{2}\sqrt{Q}\left(3-\frac{4}{\pi}\right)$.
We can immediately see that in the non-trivial case of $\alpha < 0$, the activation functions will not comply with the condition ($Q$ can be infinitely small), and $Q$ may, indeed, converge slower than $C$. For $\alpha=0$, we can see that the convergence of $Q$ and $C$ is identical. For the case of $\alpha>0$,
we get the new condition $\alpha\sqrt{Q}<\frac{\sqrt{\frac{2}{\pi}}}{\left(3-\frac{4}{\pi}\right)}\simeq.462$

To see if this is true we need to estimate what is the region where our “small $Q$” assumption is valid. First, to keep the function continuous we can calculate $A=\frac{1-\sqrt{1-4\alpha}}{2\alpha}$, and we will check the condition in the case $\sqrt(Q)=\frac{A}{3}$ (so the probability of $x>A$ is small), giving us the condition $.462>\frac{1-\sqrt{1-4\alpha}}{6}>\frac{1}{6}$ which is always true.

To conclude the analytical analysis, we saw that for large values of $Q$ (when $\sigma_w$ is large) $\chi_q>\chi_c$ for quantized activation functions, and that for small $\sigma_w,Q$ we can expect the convergence rates to match on trivial continuous activation functions. To check the intermediate range and to verify those results for quantized activation, we numerically calculate the values of $\chi_q,\chi_c$ using equations   \ref{eq:sup_c_convergence},\ref{eq:sup_q_convergence}. Results of this experiment are shown in figure \ref{fig:convergence_CQ}.

\begin{figure}[t!]
    \parbox[t]{1.0\textwidth}{\null
        \centering
        \begin{subfigure}{0.9\textwidth}
            \centering
        	\includegraphics[width=\linewidth]{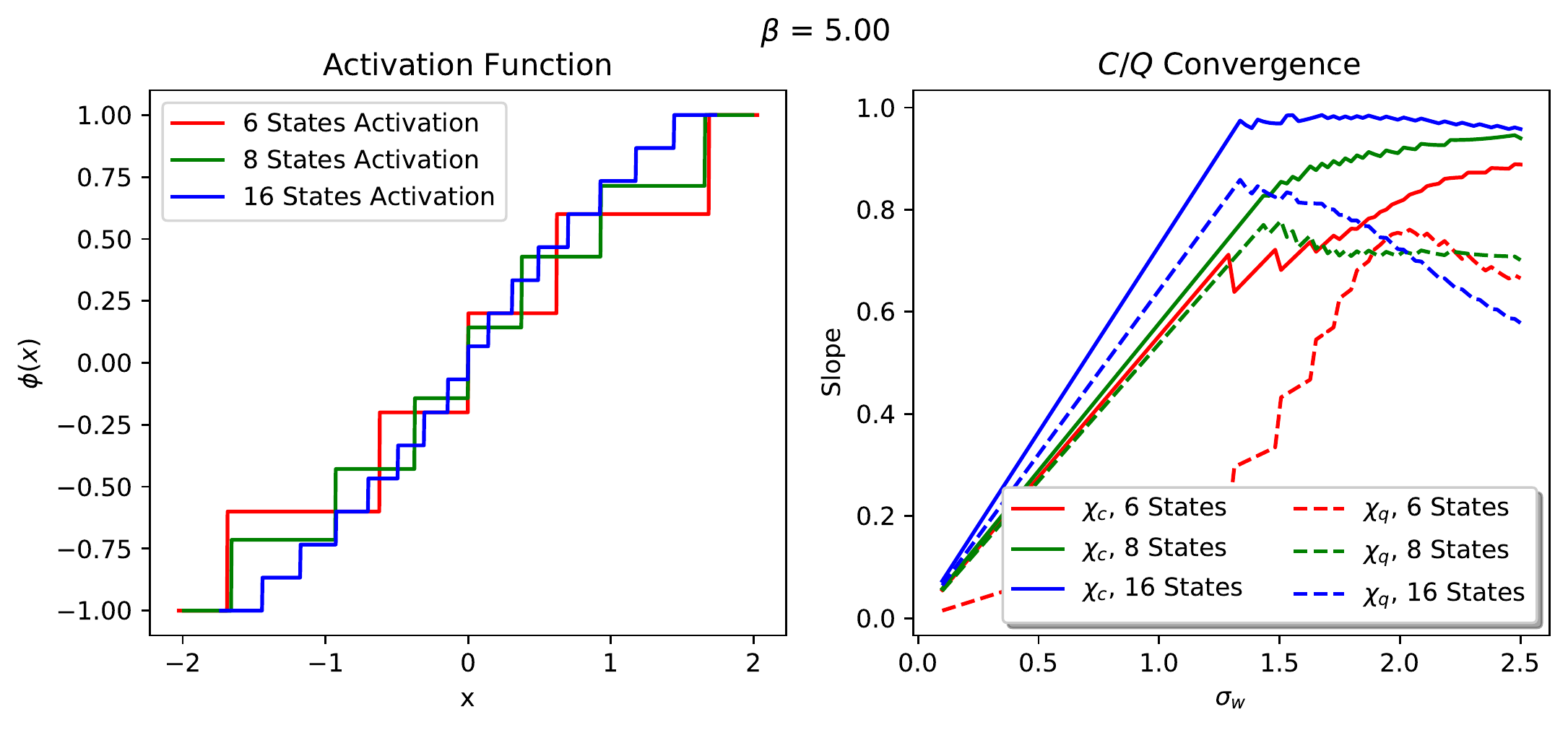}
        \end{subfigure}%

        \begin{subfigure}{0.9\textwidth}
            \centering
        	\includegraphics[width=\linewidth]{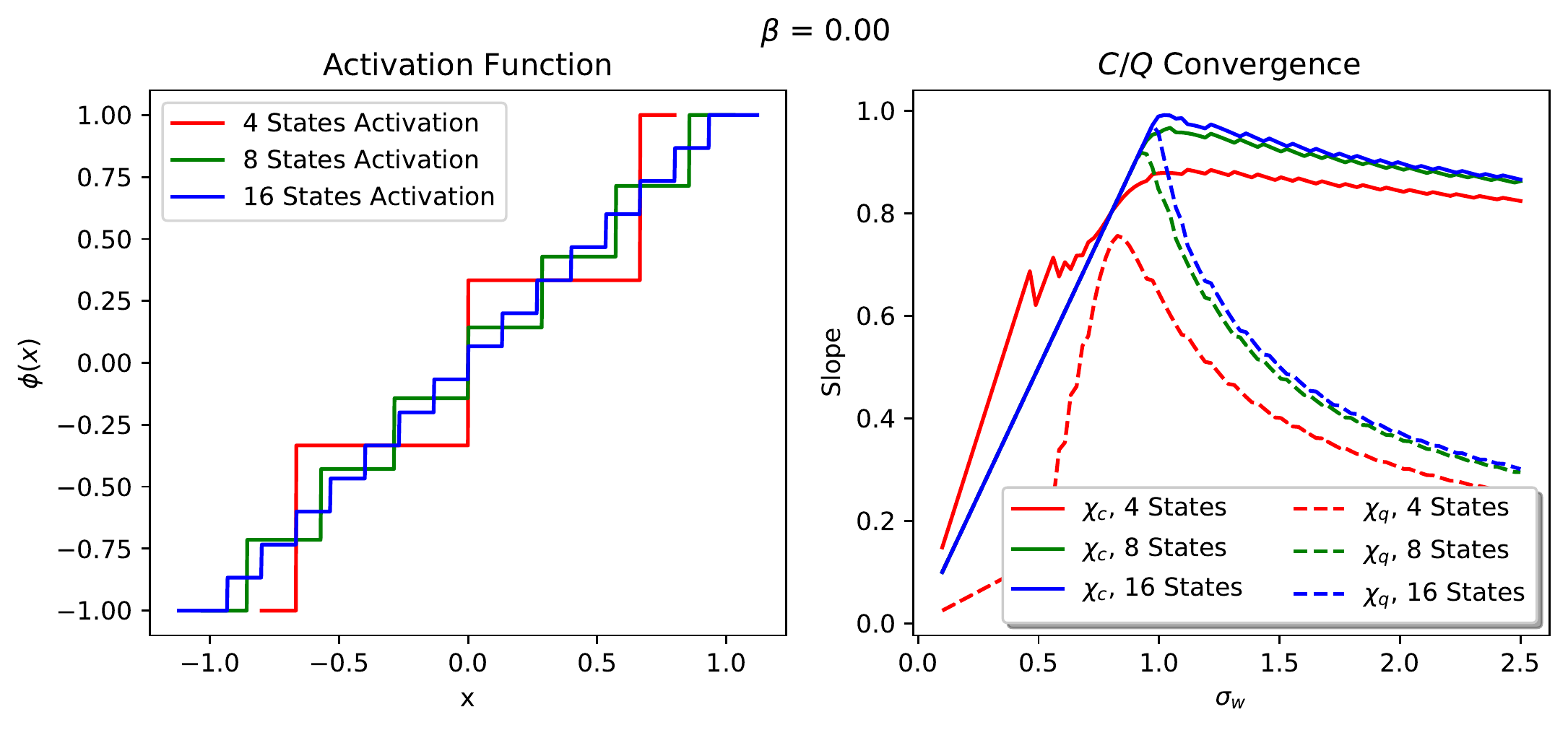}
        \end{subfigure}%

        \begin{subfigure}{0.9\textwidth}
            \centering
        	\includegraphics[width=\linewidth]{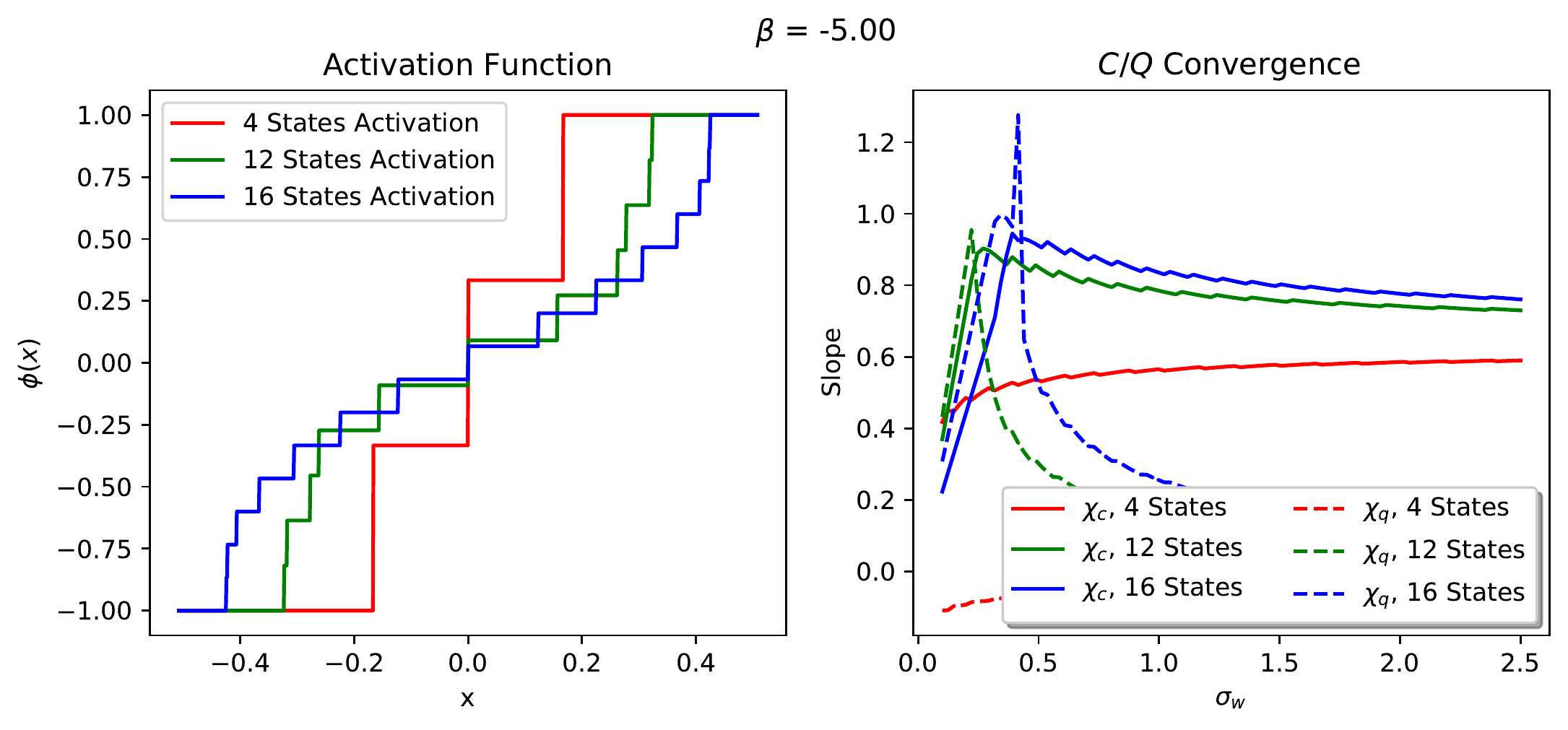}
        \end{subfigure}%
        ~
        \hspace{0.1cm}
        \parbox[t]{0.45\textwidth}{\null
        \centering

        }
        
        \captionof{figure}{Empirical comparison of the convergence rate (Fixed point slope) of $C$ and $Q$ for different quantized activation function, and varying hyperparameters ($\sigma_w$). The activation function's offsets, for each value of $\beta$ ($\beta \propto \alpha$ from eq. \ref{eq:sup_convergence_alpha}) and for given number of activation states, is calculated by $g_{i}=\frac{2}{n-1}\left(i-\frac{n}{2}\right)\left(1+\frac{\beta}{n^{2}}\left|i-n/2\right|\right)$. In accordance with our theoretical derivation, $\lim_{\sigma_{w}\to\infty}\frac{\chi_{q}}{\chi_{c}}=0$  and $\lim_{\sigma_{w}\to0}\frac{\chi_{q}}{\chi_{c}} \le 1$ if $\alpha \ge 0$. Our results also show that the gap between $\chi_c$ and $\chi_q$ is generally wider when the number of activation states is low, and that the condition $\chi_{c}>\chi_{q}$ holds for all hyperparameters for all trivial activation functions ($\beta>0$)}
        \label{fig:convergence_CQ}
    }
\end{figure}

\vspace{2.0in}

\section{Additional Proofs} 
\subsection{}\label{pf:sign_sigmaB}
\begin{proof}[Proof that fixed point slope for sign activation can only be optimal for $\sigma_b=0$]
We would like to prove the the optimal slope at the fixed point for sign activation can only be achieved when we take $\sigma_{b}$to zero. First, we will use the implicit function theorem to calculate $\frac{d\C^{\ast}}{d\sigma_{b}}$ ($\C$ is the hidden states covariance, as described in appendix \ref{app:hidden_state}), using the fixed point equation:

\[
F(\C^{\ast},\sigma_{b})=\C^{\ast}-\frac{2}{\pi}\arcsin\left(\cu^{\ast}\right)=0
\]

when ${\cu^{\ast}=\frac{\C^{\ast}\sigma_{w}^{2}+\sigma_{b}^{2}}{\sigma_{w}^{2}+\sigma_{b}^{2}},\qu^{\ast}}=\sigma_{w}^{2}+\sigma_{b}^{2}:$

\[
\frac{\partial F}{\partial\C^{\ast}}=1-\chi
\]
When we $\chi$ can be expressed  using \ref{eq:chi_sign}. Also:

\[
\frac{\partial F}{\partial \sigma_{b}}=-\frac{2}{\pi}\frac{1}{\sqrt{1-\left(\cu^{\ast}\right)^{2}}}\frac{\partial \cu^{\ast}}{\partial \sigma_{b}}=
\]

\[
-\frac{2}{\pi}\frac{1}{\sqrt{1-\left(\cu^{\ast}\right)^{2}}}\frac{2\sigma_{b}\sigma_{w}^{2}\left(1-\C^{\ast}\right)}{\left(\sigma_{w}^{2}+\sigma_{b}^{2}\right)^{2}}=-\chi\frac{2\sigma_{b}\left(1-\C^{\ast}\right)}{\qu^{\ast}}
\]

and using the implicit function theorem:

\[
\frac{d \C^{\ast}}{d \sigma_{b}}=-\frac{\frac{\partial F}{\partial \sigma_{b}}}{\frac{\partial F}{\partial \C^{\ast}}}=\frac{\chi}{1-\chi}\frac{2\sigma_{b}\left(1-\C^{\ast}\right)}{\qu^{\ast}}
\]
we can now use it to calculate:

\[
\frac{d\chi}{d\sigma_{b}}=\frac{2\sigma_{w}^{2}}{\pi\left(\sigma_{w}^{2}+\sigma_{b}^{2}\right)\sqrt{1-\left(\cu^{\ast}\right)^{2}}}\left[-\frac{2\sigma_{b}}{\qu^{\ast}}+\frac{\cu^{\ast}}{1-\left(\cu^{\ast}\right)^{2}}\frac{d\cu}{d\sigma_{b}}\right]
\]

while:

\[
\frac{d\cu}{d\sigma_{b}}=\frac{\sigma_{w}^{2}}{\sigma_{w}^{2}+\sigma_{b}^{2}}\frac{dC^{\ast}}{d\sigma_{b}}+\frac{2\sigma_{b}}{\sigma_{w}^{2}+\sigma_{b}^{2}}-\frac{C^{\ast}\sigma_{w}^{2}+\sigma_{b}^{2}}{\left(\sigma_{w}^{2}+\sigma_{b}^{2}\right)}2\sigma_{b}=
\]

\[
\frac{1}{\qu^{\ast}}\left(\sigma_{w}^{2}\frac{dC^{\ast}}{d\sigma_{b}}+2\sigma_{b}\left(1-\cu^{\ast}\right)\right)
\]

so:

\[
\frac{d\chi}{d\sigma_{b}}=\frac{2\sigma_{w}^{2}}{\pi\left(\qu^{\ast}\right)^{2}\sqrt{1-\left(\cu^{\ast}\right)^{2}}}\left[+\frac{\cu^{\ast}}{1-\left(\cu^{\ast}\right)^{2}}\left(\sigma_{w}^{2}\frac{dC^{\ast}}{d\sigma_{b}}+2\sigma_{b}\left(1-\cu^{\ast}\right)\right)-2\sigma_{b}\right]=
\]

\[
\frac{d\chi}{d\sigma_{b}}=\frac{2\sigma_{w}^{2}}{\pi\left(\qu^{\ast}\right)^{2}\sqrt{1-\left(\cu^{\ast}\right)^{2}}}\left[+\frac{\cu^{\ast}}{1-\left(\cu^{\ast}\right)^{2}}\left(\sigma_{w}^{2}\frac{\chi2\sigma_{b}}{1-\chi}\left(1-\cu^{\ast}\right)\right)+2\sigma_{b}\frac{\cu^{\ast}}{1+\left(\cu^{\ast}\right)}-2\sigma_{b}\right]=
\]

\[
\frac{2\sigma_{w}^{2}}{\pi\left(\qu^{\ast}\right)^{2}\left(1+\cu^{\ast}\right)\sqrt{1-\left(\cu^{\ast}\right)^{2}}}\left[\frac{\cu^{\ast}}{1-\cu^{\ast}}\sigma_{w}^{2}\frac{\chi}{1-\chi}\frac{2\sigma_{b}\left(1-\C^{\ast}\right)}{\qu^{\ast}}-2\sigma_{b}\right]=
\]

\[
\frac{4\sigma_{w}^{2}\sigma_{b}}{\pi\left(\qu^{\ast}\right)^{2}\left(1+\cu^{\ast}\right)\sqrt{1-\left(\cu^{\ast}\right)^{2}}}\left[\frac{\cu^{\ast}\qu^{\ast}}{\left(1-\C^{\ast}\right)\sigma_{w}^{2}}\sigma_{w}^{2}\frac{\chi}{1-\chi}\frac{\left(1-\C^{\ast}\right)}{\qu^{\ast}}-1\right]=
\]

\[
\frac{2\sigma_{b}\chi}{\left(\qu^{\ast}\right)\left(1+\cu^{\ast}\right)}\left[\frac{\chi \cu^{\ast}}{1-\chi}-1\right]
\]

we learn that  $sign\left(\frac{d\chi}{d\sigma_{b}}\right)$ depends on $\frac{\chi \cu^{\ast}}{1-\chi}-1$. if for some value of $\sigma_{b},\sigma_{w}$, $\frac{\chi \cu^{\ast}}{1-\chi}>1$, then, $\frac{\chi \cu^{\ast}}{1-\chi}-1$ will remain positive when increasing $\sigma_{b}$, since $\frac{d\chi}{d\sigma_{b}}>0$ and  $\frac{d\cu^{\ast}}{d\sigma_{b}}>0$ results $\frac{d}{d\sigma_{b}}\frac{\chi \cu^{\ast}}{1-\chi}>0$. The optimal (highest) value of $\chi$ for the given value of $\sigma_{w}$ will therefore be achieved in the limit $\sigma_{b}\to\infty$, and we can use the slope equation to calculate it:

\[
\lim_{\sigma_{b}\to\infty}\chi=\lim_{\sigma_{b}\to\infty}\frac{2\sigma_{w}^{2}}{\pi\sqrt{\left(\sigma_{w}^{2}+\sigma_{b}^{2}\right)^{2}-\left(\C^{\ast}\sigma_{w}^{2}+\sigma_{b}^{2}\right)^{2}}}=0
\]
(for this we use the fact that $\C^{\ast}>0$ for $\sigma_{b}>0$)

And this contradicts our assumption that this is the highest value of $\chi$, so $\frac{d\chi}{d\sigma_{b}}$ must be negative for all values of $\sigma_{b},\sigma_{w}$.

\end{proof}
\subsection{}\label{pf:stochastic} 
\begin{proof}[Proof that stochastic rounding results smaller slope at the fixed point]
We have shown that the the covariance mapping function with stochastic rounding is $\M(\C) = f(\frac{C_u}{B})$, when we denote $C_u(\C)= \frac{
\C\sigma_{w}^2+\sigma_{b}^2}{\sigma_{w}^2+\sigma_{b}^2} $, $C_u^{\ast}=C_u(\C^{\ast})$, and $f(x)$ is a convex function for $0\le x \le 1$ and the variable $B\ge1$ is increasing as the variance of the stochastic rounding increase, and $B=1$ gives us the mapping for deterministic function. We will show that $\frac{d\chi^{\ast}}{dB}<0$, when $\chi^{\ast}$ is the fixed point slope. Using the implicit function theorem as we did in proof \ref{pf:sign_sigmaB}, for the function: 
\[
F(\C^{\ast},B)=\C^{\ast}-\M(\C^{\ast})=0
\] for $\frac{\partial F}{\partial \C^{\ast}}$ we get:

\[
\frac{\partial F}{\partial \C^{\ast}}=1-\chi^{\ast}>0 
\]when we used the definition of $\chi^{\ast}$ as  the fixed point slop. For $\frac{\partial F}{\partial B}$, we get

\[
\frac{\partial F}{\partial B}=-f^{\prime}\left(\frac{C_u}{B}\right)\cdot \frac{-C_u}{B^2}>0 
\]
using the implicit function theorem:
\[
\frac{d\C^{\ast}}{dB}=-\frac{\frac{dF}{dB}}{\frac{dF}{d\C^{\ast}}}<0
\]
and since $\frac{d\C_{u}^{\ast}}{d\C^{\ast}}=\frac{\sigma_{w}^{2}}{\sigma_{w}^{2}+\sigma_{b}^{2}}>0 $  this also means that:

\begin{equation}\label{eq:sr1}
\frac{d\cu^{\ast}}{dB}<0
\end{equation}

from eq. \ref{eq:stochastic_rounding_chi} we know that: 
\[
\chi^{\ast}=\frac{2\sigma_{w}^{2}}{\pi \qu^{\ast}\sqrt{\left(1-\left(\frac{C_u^{\ast}}{B}\right)^{2}\right)}}
\]we can immediately see that for $ \overline{C} \equiv  \frac{C_u^{\ast}}{B}$, we get $\frac{d\chi^{\ast}}{d\overline{C}}>0$, and from eq. \ref{eq:sr1} we get that $\frac{d\overline{C}}{dB}=\frac{B\frac{d\cu^{\ast}}{dB}-C_u^{\ast}}{B^{2}}<0$ so the chain rule gives us $\frac{d\chi^{\ast}}{dB}<0$.
\end{proof}

\section{Neural tangent kernel for quantized activations} \label{app:NTK_STE}

We consider the dynamics of training for deep, wide neural networks. We argue that the error at an average test point will not improve during early stages of training if the signal propagation conditions are not satisfied, and thus ensuring signal propagation should have a beneficial effect on generalization error. 

\subsection{NTK setup}
We consider full-batch gradient descent with regression loss in a continuous time setting. Defining a fitting error $\zeta_i = f(x_i) - y_i$ \footnote{This can be generalized to other loss functions \cite{lee2019wide}.}, the loss function is given by 
\[\varphi=\frac{1}{2N_d}\sum_{i=1}^{N_d}\zeta^2_i.\] 
where $N_d$ is the number of data points. The weights evolve in time according to 
\[\frac{\partial\theta_{p}}{\partial t}=-\frac{\partial\varphi}{\partial\theta_{p}}=-\frac{1}{N_{d}}\underset{i=1}{\overset{N_d}{\sum}}\frac{\partial f(x_{i})}{\partial\theta_{p}}\zeta_{i} \]

for all weights $ \theta_p $. The evolution of the network function is then given by 
\[ \frac{\partial f(x_{i})}{\partial t}=\underset{p}{\sum}\frac{\partial f(x_{i})}{\partial\theta_{p}}\frac{\partial\theta_{p}}{\partial t}=-\frac{1}{N_d}\underset{p}{\sum}\frac{\partial f(x_{i})}{\partial\theta_{p}}\frac{\partial f(x_{j})}{\partial\theta_{p}}\zeta_i\equiv-\frac{1}{N_d}\left[\Theta\zeta\right]_{i} \]
where $ p $ indexes all the weights of the neural network and we have defined the Gram matrix $ \Theta\in\mathbb{R}^{N_d\times N_d} $ by 
\begin{equation} \label{eq:Theta}
\Theta(x_{i},x_{j})=\underset{p}{\sum}\frac{\partial f(x_{i})}{\partial\theta_{p}}\frac{\partial f(x_{j})}{\partial\theta_{p}}. 
\end{equation}

This matrix is referred to as the Neural Tangent Kernel (NTK) in \cite{Jacot2018-dv}. When considering this object at the infinite width limit, it is convenient to adopt the following parametrization for a fully connected network $f:\mathbb{R}^{n_{0}}\rightarrow\mathbb{R}^{n_{L+1}}$: 
\begin{equation} \label{eq:NTK_net}
\begin{array}{c}
\phi(\alpha^{(0)}(x))=x\\
\alpha^{(l)}(x)= \frac{\sigma_w}{\sqrt{n_{l-1}}} W^{(l)}\phi(\alpha^{(l-1)}(x))+\sigma_b b^{(l)},\text{   }l=1,...,L\\
f(x) = \alpha^{(L+1)}(x)
\end{array}
\end{equation}
for input $x \in \mathbb{R}^{n_0}$ and weight matrices $W^{(l)} \in
\mathbb{R}^{n_{l} \times n_{l-1}}$. The weights are initialized using $W_{ij}^{(l)}\sim\mathcal{N}(0,1),b_{i}^{(l)}\sim\mathcal{N}(0,1)$. The output of this NTK network is identical to that of a standard network, yet the gradients are rescaled such that $\Theta$ remains finite when taking the infinite width limit. For an appropriately chosen learning rate the dynamics of learning in the NTK network can be made identical to those of a standard network \cite{lee2019wide}. 

In \cite{Jacot2018-dv}, under some technical conditions, $\Theta$ was shown to be essentially constant during training at the sequential limit $\underset{n_{L-1}\rightarrow\infty}{\lim}...\underset{n_{2}\rightarrow\infty}{\lim}\underset{n_{1}\rightarrow\infty}{\lim}$. 
At this limit, adapting Theorem 1 of \cite{Jacot2018-dv} to allowing arbitrary variances for the weights and biases, one obtains the following asymptotic form of $\Theta$ at the sequential infinite width limit:

\begin{equation} \label{eq:asympNTK}
\overline{\Theta}(x,x')=\underset{l=1}{\overset{L+1}{\sum}}\underset{j=l+1}{\overset{L+1}{\Pi}}\Sigma'^{(j)}(x,x')\Sigma^{(l)}(x,x')
 \end{equation}

where 
\begin{equation} \label{eq:sigma_recursion}
\begin{array}{c}
\Sigma^{(1)}(x,x')=\frac{\sigma_{w}^{2}}{n_{0}}x^{T}x'+\sigma_{b}^{2}\\
\Sigma^{(l)}(x,x')=\sigma_{w}^{2}\underset{(u_{1},u_{2})\sim\mathcal{N}(\mean,\left.\Sigma^{(l-1)}\right|_{x,x'})}{\mathbb{E}}\phi(u_{1})\phi(u_{2})+\sigma_{b}^{2}\\
\left.\Sigma^{(l)}\right|_{x,x'}=\left(\begin{array}{cc}
\Sigma^{(l)}(x,x) & \Sigma^{(l)}(x,x')\\
\Sigma^{(l)}(x,x') & \Sigma^{(l)}(x',x')
\end{array}\right).
\end{array}
\end{equation}
are the covariances of the pre-activations and 
\[
\Sigma'^{(l)}(x,x')=\sigma_w^2\underset{(u_{1},u_{2})\sim\mathcal{N}(0,\Sigma^{(l)}|_{x,x'})}{\mathbb{E}}\phi'(u_{1})\phi'(u_{2}).
\]

In \cite{arora2019exact} it was also shown that for finite width ReLU networks $\mathbb{E}\Theta=\overline{\Theta}$ and concentrates about its expectation with the fluctuations scaling inversely with layer width. It follows that when taking the layer widths to infinity in arbitrary order for ReLU networks one recovers $\overline{\Theta}$, and empirically $\Theta$ concentrates well around $\overline{\Theta}$ for other choices of nonlinearities \cite{lee2019wide}. We note that even when using the standard scaling \ref{eq:net}, for very wide networks where the effect of individual weights will be negligible, even though the asymptotic for of the NTK at infinite width may be different, it will still change little in the initial phases of training. 

\subsection{Continuous activations}

We write the NTK for a feed-forward network in the NTK parametrization \ref{eq:NTK_net}, omitting the dependence on $x$ of $f,\alpha^{(l)}$ to lighten notation

\[
\frac{\partial f}{\partial W_{ij}^{(l)}}=\frac{\sigma_{w}}{\sqrt{n_{l-1}}}\frac{\partial f}{\partial\alpha_{i}^{(l)}}\phi(\alpha_{j}^{(l-1})=\frac{\sigma_{w}}{\sqrt{n_{l-1}}}\frac{\partial f}{\partial\phi(\alpha_{i}^{(l)})}\frac{\partial\phi(\alpha_{i}^{(l)})}{\partial\alpha_{i}^{(l)}}\phi(\alpha_{j}^{(l-1)})
\]

\[
=\underset{k=1}{\overset{n_{l+1}}{\sum}}\frac{\sigma_{w}}{\sqrt{n_{l-1}}}\frac{\partial f}{\partial\alpha_{k}^{(l+1)}}\frac{\partial\alpha_{k}^{(l+1)}}{\partial\phi(\alpha_{i}^{(l)})}\frac{\partial\phi(\alpha_{i}^{(l)})}{\partial\alpha_{i}^{(l)}}\phi(\alpha_{j}^{(l-1)})
\]

\[
=\frac{\sigma_{w}}{\sqrt{n_{l-1}}}\underset{k=1}{\overset{n_{l+1}}{\sum}}\frac{\partial f}{\partial\alpha_{k}^{(l+1)}}\frac{\sigma_{w}}{\sqrt{n_{l}}}W_{ki}^{(l+1)}\frac{\partial\phi(\alpha_{i}^{(l)})}{\partial\alpha_{i}^{(l)}}\phi(\alpha_{j}^{(l-1)})
\]

restoring the $x$ dependence and defining a diagonal matrix $D^{(l)}(x)=\text{diag}(\frac{\sigma_{w}}{\sqrt{n_{l}}}\frac{\partial\phi(\alpha_{i}^{(l)}(x))}{\partial\alpha_{i}^{(l)}(x)})$
we have

\[
\frac{\partial f(x)}{\partial W_{ij}^{(l)}}=\frac{\sigma_{w}}{\sqrt{n_{l-1}}}\left[\left(\frac{\partial f(x)}{\partial\alpha^{(l+1)}(x)}\right)^{T}W^{(l+1)}D^{(l)}(x)\right]_{i}\phi(\alpha_{j}^{(l-1)}(x))
\]

we can repeat the process for the elements of $\frac{\partial f(x)}{\partial\alpha^{(l+1)}}$
finally obtaining

\[
\frac{\partial f(x)}{\partial W_{ij}^{(l)}}=\frac{\sigma_{w}}{\sqrt{n_{l-1}}}\left[W^{(L+1)}D^{(L)}(x)W^{(L)}...W^{(l+1)}D^{(l)}(x)\right]_{i}\phi(\alpha_{j}^{(l-1)}(x))\equiv\frac{\sigma_{w}}{\sqrt{n_{l-1}}}\widehat{\beta}_{i}^{(l)}(x)\widehat{\alpha}_{j}^{(l)}(x)
\]

and we similarly obtain

\[
\frac{\partial f(x)}{\partial b_{i}^{(l)}}=\sigma_{b}\widehat{\beta}_{i}^{(l)}(x).
\]

The NTK thus takes the form

\[
\begin{array}{c}
\Theta(x,x')=\underset{l,i_{l}}{\sum}\frac{\partial f(x)}{\partial W_{i_{l}i_{l-1}}^{(l)}}\frac{\partial f(x')}{\partial W_{i_{l}i_{l-1}}^{(l)}}+\underset{l,i_{l}}{\sum}\frac{\partial f(x)}{\partial b_{i_{l}}^{(l)}}\frac{\partial f(x')}{\partial b_{i_{l}}^{(l)}}\\
=\sigma_{w}^{2}\underset{l=1}{\overset{L+1}{\sum}}\frac{1}{n_{l-1}}\left\langle \widehat{\beta}^{(l)}(x),\widehat{\beta}^{(l)}(x')\right\rangle \left\langle \widehat{\alpha}^{(l)}(x),\widehat{\alpha}^{(l)}(x')\right\rangle +\sigma_{b}^{2}\left\langle \widehat{\beta}^{(l)}(x),\widehat{\beta}^{(l)}(x')\right\rangle 
\end{array}
\]

According to \cite{Jacot2018-dv, arora2019exact}, this tends to \ref{eq:asympNTK} at the infinite width limit.

\subsection{Quantized activations}

We now consider dynamics in function space with quantized activations.
Analyzing a single network in this fashion is hopeless since the network
function is not a continuous function of the weights and so the dynamics
will not be continuous. We can instead consider a stochastic rounding
scheme where the post-activations are defined according to

\[
\widehat{\alpha}_{i}^{(l)}=\text{sign}(\alpha_{i}^{(l)}-z_{i}^{(l)})
\]

and $z_{i}^{(l)}\sim\text{Unif}([-1,1])$. The connection between this setup and the straight-through estimator (STE) was first observed in \cite{hubara2017quantized}. We denote the set of all
$z_{i}^{(l)}$ by $\{z\}$. Considering the dynamics of an ensemble
average such that the loss function is given by 

\[
\varphi=\frac{1}{2N_{d}}\underset{i=1}{\overset{N_{d}}{\sum}}(\mathbb{E}_{\{z\}}f(x_{i})-y_{i})^{2}=\frac{1}{2N_{d}}\underset{i=1}{\overset{N_{d}}{\sum}}\zeta_{i}^{2}
\]

We have 

\[
\frac{\partial\mathbb{E}_{z_{i}^{(l)}}f}{\partial\alpha_{i}^{(l)}}=\frac{\partial}{\partial\alpha_{i}^{(l)}}\left(p(\alpha_{i}^{(l)}-z_{i}^{(l)}>0|\alpha_{i}^{(l)})\left.f\right|_{\widehat{\alpha}_{i}^{(l)}=1}+(1-p(\alpha_{i}^{(l)}-z_{i}^{(l)}>0|\alpha_{i}^{(l)}))\left.f\right|_{\widehat{\alpha}_{i}^{(l)}=-1}\right)
\]

\[
=\frac{\partial p(\alpha_{i}^{(l)}-z_{i}^{(l)}>0|\alpha_{i}^{(l)})}{\partial\alpha_{i}^{(l)}}\left(\left.f\right|_{\widehat{\alpha}_{i}^{(l)}=1}-\left.f\right|_{\widehat{\alpha}_{i}^{(l)}=-1}\right)=\frac{1}{2}\mathbbm{1}_{\left|\alpha_{i}^{(l)}\right|\leq1}\left(\left.f\right|_{\widehat{\alpha}_{i}^{(l)}=1}-\left.f\right|_{\widehat{\alpha}_{i}^{(l)}=-1}\right).
\]

If we now consider any smooth extension of $\gamma$ of $\widehat{\alpha}_{i}^{(l)}$
such that $[-1,1]\subseteq\text{Im}(\gamma)$ and denote by $\tilde{f}$
a copy of $f$ where we replace $\widehat{\alpha}_{i}^{(l)}$ by $\gamma$.
We then have

\[
\left.f\right|_{\widehat{\alpha}_{i}^{(l)}=1}-\left.f\right|_{\widehat{\alpha}_{i}^{(l)}=-1}=\left.\tilde{f}\right|_{\gamma=1}-\left.\tilde{f}\right|_{\gamma=-1}=\left.\frac{\partial\tilde{f}}{\partial\gamma}\right|_{\gamma=0}+\mathcal{O}\left(\left.\frac{\partial^{3}\tilde{f}}{\partial\gamma^{3}}\right|_{\gamma=0}\right)=\left.\frac{\partial\tilde{f}}{\partial\gamma}\right|_{\gamma=0}+\mathcal{O}\left(\left.\frac{\partial^{3}\tilde{f}}{\partial\gamma^{3}}\right|_{\gamma=0}\right)
\]

\[
=\left.\frac{\partial\tilde{f}}{\partial\gamma}\right|_{\gamma=\pm1}+\mathcal{O}\left(\left.\frac{\partial^{2}\tilde{f}}{\partial\gamma^{2}}\right|_{\gamma=0}\right)=\left.\frac{\partial f}{\partial\widehat{\alpha}_{i}^{(l)}}\right|_{\widehat{\alpha}_{i}^{(l)}=\pm1}+\mathcal{O}\left(\left.\frac{\partial^{2}\tilde{f}}{\partial\gamma^{2}}\right|_{\gamma=0}\right)=\frac{\partial\mathbb{E}_{z_{i}^{(l)}}f}{\partial\widehat{\alpha}_{i}^{(l)}}+\mathcal{O}\left(\left.\frac{\partial^{2}\tilde{f}}{\partial\gamma^{2}}\right|_{\gamma=0}\right).
\]

If we neglect these higher order terms (which should be small since
the influence of a single neuron on the output is generally small,
and should vanish at the infinite width limit) and note that the above
approximation holds if we condition on $\{z\}\backslash\{z_{i}^{(l)}\}$,
we obtain

\begin{equation}
\frac{\partial\mathbb{E}_{\{z\}}f}{\partial\alpha_{i}^{(l)}}\approx\mathbbm{1}_{\left|\alpha_{i}^{(l)}\right|\leq1}\frac{\partial\mathbb{E}_{\{z\}}f}{\partial\widehat{\alpha}_{i}^{(l)}}.\label{eq:STEf2}
\end{equation}




We can now repeat the calculation of the NTK using eq. \ref{eq:STEf2},
obtaining

\[
\frac{\partial\mathbb{E}_{\{z\}}f}{\partial W_{ij}^{(l)}}=\frac{\sigma_{w}}{\sqrt{n_{l-1}}}\frac{\partial\mathbb{E}_{\{z\}}f}{\partial\alpha_{i}^{(l)}}\phi(\alpha_{j}^{(l-1})\approx\frac{\sigma_{w}}{\sqrt{n_{l-1}}}\frac{\partial f}{\partial\phi(\alpha_{i}^{(l)})}\mathbbm{1}_{\left\vert \alpha_{i}^{(l)}\right\vert \leq1}\phi(\alpha_{j}^{(l-1)})
\]

\[
=\frac{\sigma_{w}}{\sqrt{n_{l-1}}}\underset{k=1}{\overset{n_{l+1}}{\sum}}\frac{\partial f}{\partial\alpha_{k}^{(l+1)}}\frac{\sigma_{w}}{\sqrt{n_{l}}}W_{ki}^{(l+1)}\mathbbm{1}_{\left\vert \alpha_{i}^{(l)}\right\vert \leq1}\phi(\alpha_{j}^{(l-1)}).
\]

Defining $D_{\text{STE}}^{(l)}(x)=\text{diag}(\frac{\sigma_{w}}{\sqrt{n_{l}}}\mathbbm{1}_{\left\vert \alpha_{i}^{(l)}\right\vert \leq1})$
and applying eq. \ref{eq:STEf2} repeatedly at each
layer up until $L+1$ gives

\[
\frac{\partial f(x)}{\partial W_{ij}^{(l)}}\approx\frac{\sigma_{w}}{\sqrt{n_{l-1}}}\left[W^{(L+1)}D_{\text{STE}}^{(L)}(x)W^{(L)}...W^{(l+1)}D_{\text{STE}}^{(l)}(x)\right]_{i}\phi(\alpha_{j}^{(l-1)}(x))\equiv\frac{\sigma_{w}}{\sqrt{n_{l-1}}}\widehat{\beta}_{\text{STE},i}^{(l)}(x)\widehat{\alpha}_{j}^{(l)}(x)
\]

\[
\frac{\partial f(x)}{\partial b_{i}^{(l)}}\approx\sigma_{b}\widehat{\beta}_{\text{STE},i}^{(l)}(x).
\]

and thus applying \ref{eq:Theta} gives
\[
\frac{\partial\mathbb{E}_{\{z\}}f(x)}{\partial t}\approx-\frac{1}{N_{d}}\underset{i}{\sum}\Theta_{\text{STE}}(x,x_{i})\zeta_{i}
\] 
where 

\[
\Theta_{\text{STE}}(x,x')=\sigma_{w}^{2}\underset{l=1}{\overset{L+1}{\sum}}\frac{1}{n_{l-1}}\left\langle \widehat{\beta}_{\text{STE}}^{(l)}(x),\widehat{\beta}_{\text{STE}}^{(l)}(x')\right\rangle \left\langle \widehat{\alpha}^{(l)}(x),\widehat{\alpha}^{(l)}(x')\right\rangle +\sigma_{b}^{2}\left\langle \widehat{\beta}_{\text{STE}}^{(l)}(x),\widehat{\beta}_{\text{STE}}^{(l)}(x')\right\rangle .
\]

A trivial generalization of the calculation of the asymptotic form
of $\Theta(x,x')$ at the infinite width limit in \cite{Jacot2018-dv}
shows that at this limit $\Theta_{\text{STE}}(x,x')$ tends to 

\begin{equation} \label{eq:NTK_STE2}
\overline{\Theta}_{\text{STE}}(x,x')=\underset{l=1}{\overset{L+1}{\sum}}\underset{j=l+1}{\overset{L+1}{\Pi}}\Sigma_{\text{STE}}^{\prime(j)}(x,x')\Sigma^{(l)}(x,x')
\end{equation}

where $\Sigma^{(l)}(x,x')$ is defined in eq. \ref{eq:sigma_recursion}, 

\[
\Sigma_{\text{STE}}^{\prime(l)}(x,x')=\sigma_{w}^{2}\underset{(u_{1},u_{2})\sim\mathcal{N}(0,\Sigma^{(l)}|_{x,x'})}{\mathbb{E}}\phi'_{\text{STE}}(u_{1})\phi'_{\text{STE}}(u_{2}).
\]

and we define the hard-tanh function, 
\begin{equation} \label{eq:phi_ste}
\phi_{\text{STE}}(x)=\begin{cases}
1 & 1\leq x\\
x & -1<x<1\\
-1 & x\leq-1
\end{cases}.
\end{equation} 

for which $\phi'_{\text{STE}}(y)=\mathbbm{1}_{\left\vert y\right\vert \leq1}$. The form of $\overline{\Theta}_{\text{STE}}(x,x')$ is thus obtained by replacing the $\text{sign}$ activation with eq. \ref{eq:phi_ste} but only during the backwards pass (and not during the forward pass), in line with the motivation of the STE in \cite{hubara2017quantized}. We note that the dynamics of this ensemble average correspond to those of the update scheme in eq. \ref{eq:ste_rho} with $\rho=1$. Other choices will introduce a dependence on $\rho$ in $\overline{\Theta}_{\text{STE}}(x,x')$ but will not change the fact that it can be expressed as a function of the covariances of the inputs in eq. \ref{eq:sigma_recursion}.

\subsection{Asymptotic NTK and generalization}

We now consider a very deep network such that the covariance map
approaches its fixed point

\[
\Sigma^{\ast}|_{x,x'}=Q^{\ast}\left(\begin{array}{cc}
1 & C^{\ast}\\
C^{\ast} & 1
\end{array}\right).
\]

$\Theta^{(l)}(x,x')$ for very deep networks will approach a matrix
of the form

\begin{equation} \label{eq:thetaast}
    \underset{L\rightarrow\infty}{\lim}\frac{1}{L+1}\Theta^{(L+1)}(x,x')=\Theta^{\ast}(x,x')=\alpha\delta(x,x')+\beta(1-\delta(x,x'))
\end{equation}

for some constants $\alpha,\beta$ and $\delta(x,x')$ is a Kronecker delta. 

To understand the generalization properties of such a network, we can consider the evolution of the error at some test point $z$ that is not part of the training set. It will be given by

\[
\frac{\partial\zeta(z)}{\partial t}=-\frac{L}{N_d}\underset{i=1}{\overset{N_d}{\sum}}\Theta^{\ast}(z,x_{i})\zeta(x_{i})=-\frac{\beta L}{N_d}\underset{i=1}{\overset{N_d}{\sum}}\zeta(x_{i})
\]

which at initialization is independent of our choice of $z$. Since it is also independent of the true label of $z$ this will mean that the generalization error will typically not decrease \footnote{Aside from some trivial cases such as learning a constant function.}.

We conclude that for networks deep enough that the covariance map converges, in the initial phase of training before $\Theta$ changes considerably there
will be no improvement in the generalization error at a typical test point. Conversely, this suggests that satisfying the signal propagation condition $\chi=1$ will facilitate generalization. Presumably, if convergence to the fixed point is slow, instead of the form in eq. \ref{eq:thetaast}, $\Theta$ will exhibit some finite scale of decay from its value on the diagonal as a function of the distance between the inputs. This will enable points in the training set near $z$ that share the same label, and where the error has the same sign as $\zeta(z)$, to influence $\frac{\partial\zeta(z)}{\partial t}$ thus reducing the error at $z$. This argument is independent of the value of $\beta$, and provides further motivation for the study of critical initialization schemes that exhibit slow convergence to the fixed point \cite{schoenholz2016deep}. Such initialization schemes have also been motivated in the past by concerns of trainability (i.e. ensuring stable signal propagation from the inputs to the hidden states of a deep network, and preventing vanishing/exploding gradients). This phenomenon could perhaps be the basis for the improvements in generalization observed when using critical initialization schemes, which have hitherto been unexplained. 

To explore whether rapid convergence of the covariance map is correlated with a lack of structure in the NTK, we define a coarse metric for non-trivial structure in the off-diagonal terms of the NTK that should facilitate generalization. Given a row of the NTK $\Theta_{i}=\Theta(x_{i},\cdot)\in\mathbb{R}^{N_{d}}$, we define our signal to be the sum of off-diagonal terms in this row that share a label with $x_i$:
\[
S_{i}=\underset{\begin{array}{c}
j\neq i\\
y_{j}=y_{i}
\end{array}}{\sum}\Theta(x_{i},x_{j})
\]
while the corresponding noise measure is simply 
\[
N_{i}=\left\Vert \Theta_{i}\right\Vert _{1}-S_{i}.
\]
The idea behind this metric is that the fitting error at some $\zeta(x_j)$ with $y_i=y_j$ will be closer on average to $\zeta(x_i)$ than $\zeta(x_j)$ such that $y_j \neq y_i$. If $x_i$ is not part of the training set, $\frac{\partial\zeta(x_{i})}{\partial t}=-\frac{1}{N_{d}}\underset{j=1,j\neq i}{\overset{N_{d}}{\sum}}\Theta(x_{i},x_{j})\zeta(x_{j})$. Thus if the elements of $\Theta$ with the same label as $x_i$ are large and positive there will be a large magnitude contribution to $\frac{\partial\zeta(x_{i})}{\partial t}$ that has the opposite sign as $\zeta(x_{i})$ and thus $\zeta(x_{i})$ will decrease quickly over time. The noise in this case is the size of the other entries. Generalization error should thus improve if the signal-to-noise ratio
\begin{equation} \label{eq:SNR}
\text{SNR}=\frac{1}{N_{d}}\underset{i}{\sum}\frac{S_{i}}{N_{i}}
\end{equation}
is large and 
\begin{equation} \label{eq:signal}
S=\frac{1}{N_{d}}\underset{i}{\sum}S_{i}
\end{equation}
is large as well. The latter condition is important since in the case of networks with small weight variance $\text{SNR}$ may be large but $S$ itself vanishes and so will any change in the generalization error. For both networks with $\tanh$ and quantized activatsion we observe that the regime where $\text{SNR}$ and $S$ are both large corresponds to the one where the signal propagation time scale in eq. \ref{eq:xi} is large as well, as shown in Figure \ref{fig:NTK}. 

In this experiment, the network architecture is given by \ref{eq:net} with $L=30$ and all hidden layers of width $300$. Note that for a finite width network with constant layer widths the difference between the NTK and that of a network given by \ref{eq:NTK_net} will be a constant factor. The quantities in the plot are averaged over $450$ MNIST data points for the $\tanh$ network and $200$ images for the quantized network, and $5$ different initializations. The NTK for the network with quantized activations is calculated by replacing the terms in the backwards pass with the STE equivalents, as in \ref{eq:NTK_STE2}. We note that a similar degradation in the generalization ability when the signal propagation conditions are not satisfied has been described previously in the case of wide networks where only the last layer is trained \cite{lee2017deep}. 

\begin{figure}
    \centering
    \begin{subfigure}[]{0.5\textwidth}
        \centering
        \includegraphics[height=2.0in,width=2.7in]{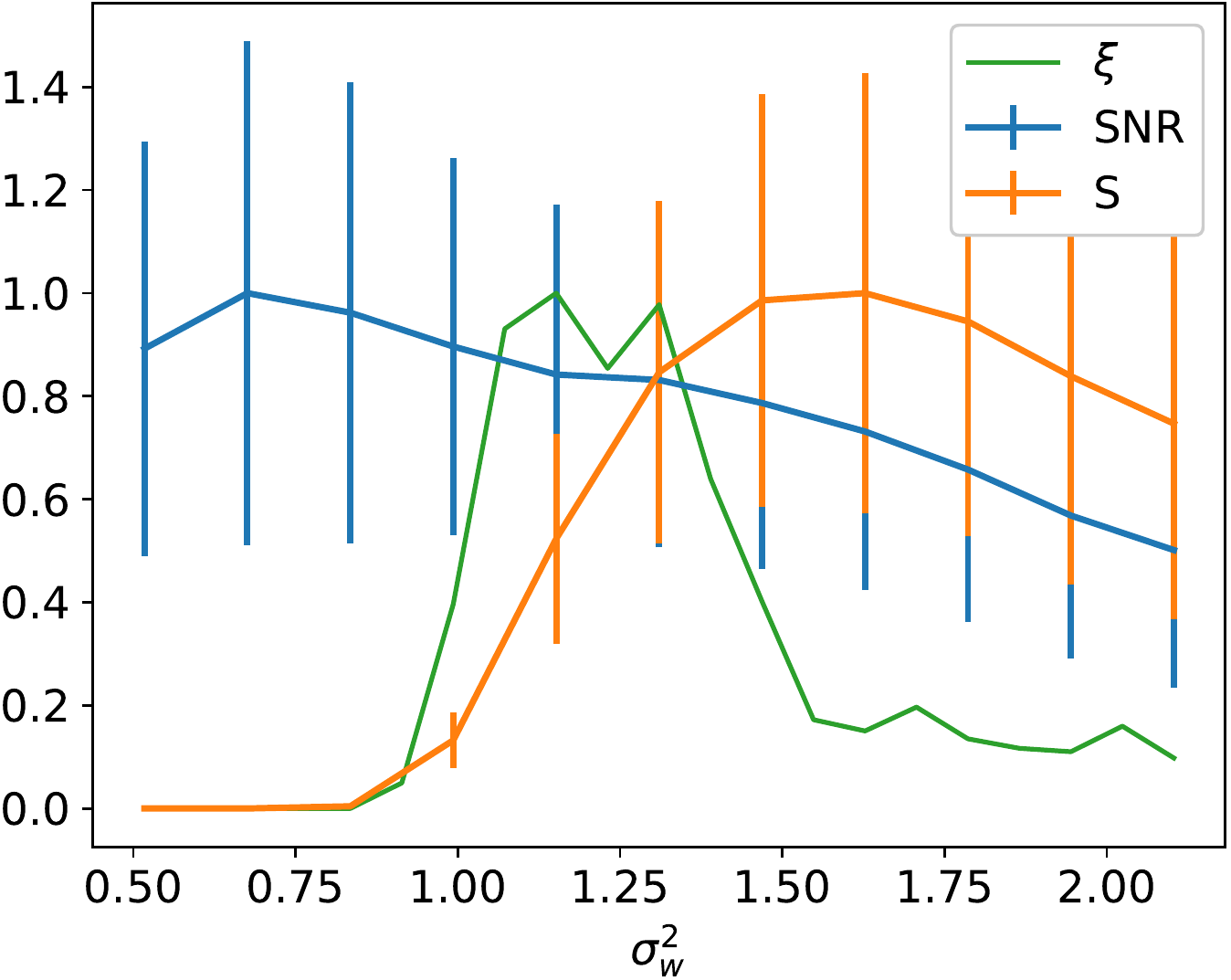}
    \end{subfigure}%
    ~
    \begin{subfigure}[]{0.5\textwidth}
        \centering
        \includegraphics[height=2.0in,width=2.7in]{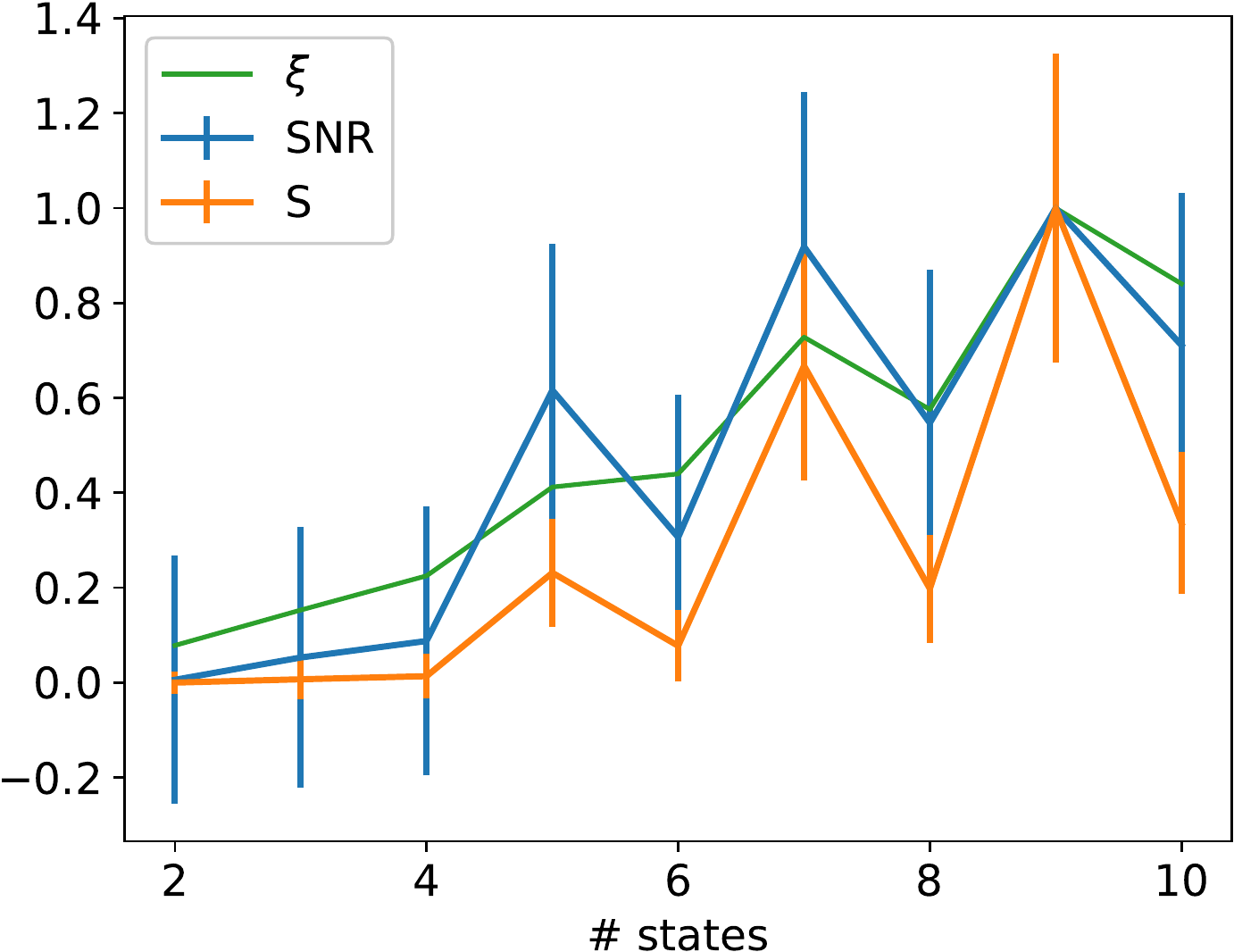}
    \end{subfigure}%
    \caption{Off-diagonal structure in the NTK is correlated with signal propagation. The signal (eq. \ref{eq:signal}) that is expected to improve generalization, the signal-to-noise ratio (eq. \ref{eq:SNR}) and the signal propagation time scale (eq. \ref{eq:xi}) are plotted for different architectures. All quantities are normalized by the maximal value in the range of parameters shown. \textit{Left:} For networks with $\tanh$ activations with different weight variance $\sigma_w^2$, the time scale $\xi$ behaves non-monotonically. The SNR decreases monotonically, while the signal $S$ spikes around the same value of $\sigma_w^2$ where signal propagation is best achieved. Thus the point that maximizes both SNR and $S$ is close to the one where signal propagation is also maximal. \textit{Right:} For networks with quantized activations, as the quantization level increases so does the SNR and the signal itself. We also observe the same non-monotonic behaviour based on the parity of the number of states in all three.}
    \label{fig:NTK}
\end{figure}

\subsection{Change of asymptotic NTK during training}

We have argued above that based on the structure of the NTK at initialization for networks where the covariance map has converged, we expect no initial improvement in the generalization error. At later times, if we assume that the Taylor expansion of $\Theta^\ast_t$ exists  
\[\Theta_{t}^{\ast}(z,x')=\underset{i=0}{\overset{\infty}{\sum}}\frac{t^{k}}{k!}\frac{\partial^{k}\Theta_{0}^{\ast}(z,x')}{\partial t^{k}}\]
we can see directly that $\Theta_{t}^{\ast}(z,x')$ will be independent of $z$ as well, since the summands in the RHS are. This argument thus extends to later times asymptotically at the infinite width limit, or for finite width until such time as deviations from the asymptotic form of the NTK influence the dynamics.

\end{document}